\newif
\ifdraftmode
\draftmodefalse

\ifdraftmode
\else
\fi	
\RequirePackage{fix-cm}
\documentclass[reqno]{amsart}
\usepackage[margin=1.5in,bottom=1.25in]{geometry}	

\usepackage{amssymb}	
\usepackage{amsfonts}	
\usepackage{amsthm}	
\usepackage[foot]{amsaddr}	

\usepackage{mathtools}	
\mathtoolsset{%
centercolon=true,
}

\usepackage[
cal=cm,
]
{mathalfa}

\usepackage{dsfont}	

\usepackage[proportional,tabular,lining,sf,mono=false]{libertine}

\usepackage[scaled=.95]{AlegreyaSans}

\usepackage[T1]{fontenc}	

\usepackage{acronym}	
\newcommand{\acdef}[1]{\define{\acl{#1}} \textup{(\acs{#1})}\acused{#1}}	

\usepackage[labelfont={bf,small},labelsep=colon,font=small]{caption}	

\usepackage{subcaption}	
\usepackage[svgnames]{xcolor}	

\definecolor{Bonfire}{HTML}{9E162E}
\definecolor{CardinalRed}{HTML}{C41E3A}
\definecolor{TuckOrange}{HTML}{D94415}

\definecolor{CadmiumGreen}{HTML}{097969}
\definecolor{Dartmouth}{HTML}{00693E}
\definecolor{ForestGreen}{HTML}{12312B}
\definecolor{RichForestGreen}{HTML}{0D1E1C}
\definecolor{SeaGreen}{HTML}{2E8B57}
\definecolor{SpringGreen}{HTML}{EAFAF1}
\definecolor{Jade}{HTML}{009900}

\definecolor{CobaltBlue}{HTML}{0047AB}
\definecolor{NavyBlue}{HTML}{000080}
\definecolor{RiverBlue}{HTML}{267ABA}
\definecolor{RiverNavy}{HTML}{003C73}
\definecolor{KleinBlue}{HTML}{002FA7}
\definecolor{OxfordBlue}{HTML}{002147}
\definecolor{SapphireBlue}{HTML}{0F52BA}
\definecolor{Zaffre}{HTML}{0819A8}

\colorlet{MyRed}{CardinalRed}
\colorlet{MyGreen}{Dartmouth}
\colorlet{MyBlue}{DodgerBlue}
\colorlet{MyViolet}{DarkMagenta}

\colorlet{MyDarkRed}{red!70!black}

\colorlet{MyLightRed}{MyRed!25}
\colorlet{MyLightGreen}{MyGreen!25}
\colorlet{MyLightBlue}{MyBlue!25}

\colorlet{PrimalColor}{MyBlue}
\colorlet{PrimalFill}{PrimalColor!25}
\colorlet{DualColor}{MyRed}

\colorlet{AlertColor}{MyRed}	
\colorlet{BadColor}{MyRed}	
\colorlet{GoodColor}{MyGreen}	
\colorlet{LinkColor}{MediumBlue}	
\colorlet{RevColor}{MediumBlue}	

\ifdraftmode
	\colorlet{DraftColor}{MyRed}	
\else
	\colorlet{DraftColor}{black}	
\fi

\newcommand{\afterhead}{.}	
\newcommand{\para}[1]{\smallskip\paragraph{\textbf{#1\afterhead}}}	

\usepackage{latexsym}	
\usepackage{fontawesome}	
\usepackage{pifont}	

\newcommand{\asterism}{\ding{70}}

\usepackage{tikz}	
\usetikzlibrary{calc,patterns,arrows.meta,positioning}	

\usepackage{array}	
\usepackage{booktabs}	
\usepackage[inline,shortlabels]{enumitem}	
\setlist[1]{topsep=\smallskipamount,itemsep=\smallskipamount,left=\parindent}
\setlist[2]{left=0pt}

\usepackage[kerning=true]{microtype}	

\usepackage{csquotes}
\usepackage{float}

\usepackage{balance}	
\usepackage{tabto}	
\usepackage{xspace}	

\usepackage[sort&compress,numbers]{natbib}	

\bibpunct[, ]{[}{]}{,}{}{,}{,}
\setcitestyle{numbers,square,comma}

\usepackage{hyperref}
\hypersetup{
final,
colorlinks=true,
linktocpage=true,
pdfstartview=FitH,
breaklinks=true,
pdfpagemode=UseNone,
pageanchor=true,
pdfpagemode=UseOutlines,
plainpages=false,
bookmarksnumbered,
bookmarksopen=false,
bookmarksopenlevel=1,
hypertexnames=false,
pdfhighlight=/O,
urlcolor=LinkColor,linkcolor=LinkColor,citecolor=LinkColor,	
pdftitle={Bounded Regret in General-Sum Games via Higher-Order Optimism},
pdfauthor={},
pdfsubject={},
pdfkeywords={},
pdfcreator={pdfLaTeX},
pdfproducer={LaTeX with hyperref}
}

\usepackage[sort&compress,capitalize,nameinlink,noabbrev]{cleveref}	

\crefname{algo}{Algorithm}{Algorithms}
\crefname{assumption}{Assumption}{Assumptions}

\makeatletter
\AtBeginDocument
{
	\def\ltx@label#1{\cref@label{#1}}	
	\def\label@in@display@noarg#1{\cref@old@label@in@display{#1}}	
}%
\makeatother

\usepackage{algorithm}	
\usepackage{algpseudocode}	

\usepackage[most]{tcolorbox}

\theoremstyle{plain}
\newtheorem{theorem}{Theorem}	
\newtheorem{corollary}{Corollary}	
\newtheorem{lemma}{Lemma}	
\newtheorem{proposition}{Proposition}	

\newtheorem*{theorem*}{Theorem}	
\newtheorem*{corollary*}{Corollary}	

\theoremstyle{definition}

\newtheorem*{definition*}{Definition}	
\newtheorem*{assumption*}{Assumptions}	
\newtheorem*{example*}{Example}	

\theoremstyle{remark}
\newtheorem{remark}{Remark}	

\newtheorem*{remark*}{Remark}	
\newtheorem*{notation*}{Notation}	

\def\endenv{\hfill\asterism}	

\newcounter{proofstep}

\numberwithin{example}{section}	

\ifdraftmode
	\usepackage[showdeletions]{color-edits}	
\else
	\usepackage[suppress]{color-edits}	
\fi

\ifdraftmode
	\newcommand{\draft}[1]{{\color{DraftColor}#1}}	
\else
	\newcommand{\draft}[1]{#1}	
\fi

\newcommand{\define}[1]{\emph{\draft{#1}}}	

\newcommand{\newmacro}[2]{\newcommand{#1}{\draft{#2}}}	
\newcommand{\newop}[2]{\DeclareMathOperator{#1}{\draft{#2}}}	
\newcommand{\newoplims}[2]{\DeclareMathOperator*{#1}{\draft{#2}}}	

\newcommand{\eps}{\varepsilon}	

\DeclarePairedDelimiter{\bracks}{[}{]}	

\DeclarePairedDelimiterX{\setdef}[2]{\{}{\}}{#1:#2}	
\DeclarePairedDelimiterXPP{\exclude}[1]{\mathopen{}\setminus}{\{}{\}}{}{#1}	

\DeclarePairedDelimiterX{\braket}[2]{\langle}{\rangle}{#1,#2}	
\DeclarePairedDelimiterX{\inner}[2]{\langle}{\rangle}{#1,#2}	

\DeclarePairedDelimiter{\norm}{\lVert}{\rVert}	
\DeclarePairedDelimiterXPP{\dnorm}[1]{}{\lVert}{\rVert}{_{\draft{\ast}}}{#1}	
\DeclarePairedDelimiterXPP{\onenorm}[1]{}{\lVert}{\rVert}{_{\draft{1}}}{#1}	
\DeclarePairedDelimiterXPP{\twonorm}[1]{}{\lVert}{\rVert}{_{\draft{2}}}{#1}	
\DeclarePairedDelimiterXPP{\pnorm}[1]{}{\lVert}{\rVert}{_{\draft{p}}}{#1}	
\DeclarePairedDelimiterXPP{\qnorm}[1]{}{\lVert}{\rVert}{_{\draft{q}}}{#1}	
\DeclarePairedDelimiterXPP{\supnorm}[1]{}{\lVert}{\rVert}{_{\draft{\infty}}}{#1}	
\DeclarePairedDelimiterXPP{\matnorm}[1]{}{\lVert}{\rVert}{_{\draft{F}}}{#1}	

\newop{\defeq}{\coloneqq}	
\newop{\eqdef}{\eqqcolon}	

\newmacro{\from}{\colon}	
\newop{\too}{\rightrightarrows}	
\newop{\injects}{\hookrightarrow}	
\newop{\surjects}{\twoheadrightarrow}	

\newmacro{\N}{\mathbb{N}}	
\newmacro{\Z}{\mathbb{Z}}	
\newmacro{\Q}{\mathbb{Q}}	
\newmacro{\R}{\mathbb{R}}	
\newmacro{\C}{\mathbb{C}}	

\newoplims{\argmax}{arg\,max}	
\newoplims{\argmin}{arg\,min}	
\newoplims{\intersect}{\bigcap}	
\newoplims{\union}{\bigcup}	

\newop{\aff}{aff}	
\newop{\bd}{bd}	
\newop{\bigoh}{\mathcal{O}}	
\newop{\tildeoh}{\mathcal{\tilde O}}	
\newop{\baroh}{\mathcal{\bar O}}	
\newmacro{\littleoh}{o}	
\newop{\card}{\#}	
\newop{\cl}{cl}	
\newop{\conv}{conv}	
\newop{\crit}{crit}	
\newop{\curl}{curl}	
\newop{\diag}{Diag}	
\newop{\diam}{diam}	
\newop{\dist}{dist}	
\newop{\diver}{div}	
\newop{\dom}{dom}	
\newop{\eig}{eig}	
\newop{\ess}{ess}	
\newop{\grad}{grad}	
\newop{\ind}{ind}	
\newop{\im}{im}	
\newop{\intr}{int}	
\newop{\Jac}{Jac}	
\newmacro{\one}{\mathbf{1}}	
\newop{\proj}{proj}	
\newop{\prox}{prox}	
\newop{\rank}{rank}	
\newop{\relint}{ri}	
\newop{\sign}{sgn}	
\newop{\supp}{supp}	
\newop{\Sym}{Sym}	
\newop{\tr}{tr}	
\newop{\unif}{unif}	
\newop{\vol}{vol}	

\newcommand{\ie}{i.e.,\xspace}	

\newmacro{\commentsymbol}{\triangleright}	
\newcommand{\eqstop}{\,.}	

\newcommand{\alt}[1]{#1'}	

\newmacro{\argdot}{\boldsymbol{\cdot}}	
\newmacro{\dd}{\kern0pt\:d}	
\newmacro{\ddt}{\frac{d}{dt}}	
\newmacro{\del}{\partial}	

\newmacro{\const}{c}	
\newmacro{\Const}{C}	

\newmacro{\param}{\theta}	
\newmacro{\params}{\Theta}	

\newmacro{\coef}{\lambda}	

\newmacro{\fn}{f} 

\newmacro{\pexp}{p}	
\newmacro{\qexp}{q}	
\newmacro{\rexp}{r}	

\newmacro{\iCount}{i}	
\newmacro{\jCount}{j}	
\newmacro{\kCount}{k}	
\newmacro{\nCounts}{n}	
\newmacro{\counts}{\mathcal{I}}	

\newmacro{\point}{x}	
\newmacro{\pointalt}{\alt\point}	
\newmacro{\pointaux}{u}	
\newmacro{\points}{\mathcal{X}}	
\newmacro{\intpoints}{\relint\points}	

\newmacro{\base}{q}	
\newmacro{\basealt}{q}	
\newmacro{\auxpoint}{u}	

\newmacro{\elem}{a}	
\newmacro{\elemalt}{b}	
\newmacro{\iElem}{a}	
\newmacro{\jElem}{b}	
\newmacro{\kElem}{c}	
\newmacro{\set}{S}	

\newmacro{\borel}{\mathcal{B}}	
\newmacro{\closed}{\mathcal{C}}	
\newmacro{\cpt}{\mathcal{K}}	
\newmacro{\nhd}{\mathcal{U}}	
\newmacro{\nhdalt}{\mathcal{V}}	
\newmacro{\open}{\mathcal{U}}	

\newmacro{\domain}{\mathcal{D}}	
\newmacro{\region}{\mathcal{R}}	

\newmacro{\interval}{\mathcal{I}}	
\newmacro{\rectangle}{\mathcal{R}}	
\newmacro{\cone}{\mathcal{K}}	

\newmacro{\tstart}{0}	
\newmacro{\timealt}{s}	
\newmacro{\timealtalt}{\tau}	
\newmacro{\horizon}{T}	

\newmacro{\curve}{\gamma}	
\DeclarePairedDelimiterXPP{\curveof}[1]{\curve}{(}{)}{}{#1}	
\DeclarePairedDelimiterXPP{\curveofX}[2]{\curve_{#1}}{(}{)}{}{#2}	
\DeclarePairedDelimiterXPP{\velof}[1]{\dot\curve}{(}{)}{}{#1}	
\DeclarePairedDelimiterXPP{\velofX}[2]{\dot\curve_{#1}}{(}{)}{}{#2}	

\newmacro{\flowmap}{\Phi}	
\DeclarePairedDelimiterXPP{\flowof}[2]{\flowmap_{#1}}{(}{)}{}{#2}	

\newmacro{\traj}{x}	
\newmacro{\dtraj}{\dot\traj}	

\DeclarePairedDelimiterXPP{\trajof}[1]{\traj}{(}{)}{}{#1}	
\DeclarePairedDelimiterXPP{\trajofX}[2]{\traj_{#1}}{(}{)}{}{#2}	
\DeclarePairedDelimiterXPP{\dtrajof}[1]{\dtraj}{(}{)}{}{#1}	
\DeclarePairedDelimiterXPP{\dtrajofX}[2]{\dtraj_{#1}}{(}{)}{}{#2}	

\newmacro{\vdim}{n}	
\newmacro{\realspace}{\R^{\vdim}}	

\newmacro{\iCoord}{i}	
\newmacro{\jCoord}{j}	
\newmacro{\kCoord}{k}	
\newmacro{\nCoords}{\vdim}	
\newmacro{\mCoords}{m}	

\newmacro{\unitvec}{u}	
\newmacro{\bvec}{e}	
\newmacro{\bvecs}{\mathcal{E}}	

\newmacro{\vecspace}{\mathcal{V}}	
\newmacro{\subspace}{\mathcal{W}}	

\newcommand{\dual}[1][\vecspace]{#1^{\ast}}	
\newmacro{\dspace}{\dual[\vecspace]}	

\newmacro{\hilbert}{\mathcal{H}}	
\newmacro{\banach}{\mathcal{X}}	

\newmacro{\mat}{M}	
\newmacro{\hmat}{H}	

\newmacro{\ones}{\mathbf{1}}	
\newmacro{\eye}{I}	
\newmacro{\zer}{\mathbf{0}}	

\newmacro{\eigval}{\lambda}	
\newmacro{\eigvec}{u}	

\newmacro{\ball}{\mathbb{B}}	
\newmacro{\sphere}{\mathbb{S}}	
\newmacro{\radius}{r}
\newmacro{\Radius}{R}

\newmacro{\mfld}{\mathcal{M}}	
\newmacro{\tanvec}{z}	
\newmacro{\form}{\omega}	

\newmacro{\gmat}{g}	
\newmacro{\gdist}{\dist_{\gmat}}	

\newmacro{\vertex}{v}	
\newmacro{\vertexalt}{w}	
\newmacro{\iVertex}{\vertex_{\iCount}}	
\newmacro{\jVertex}{\vertex_{\jCount}}	
\newmacro{\kVertex}{\vertex_{\kCount}}	
\newmacro{\nVertices}{V}	
\newmacro{\vertices}{\mathcal{V}}	

\newmacro{\edge}{e}	
\newmacro{\edgealt}{\alt\edge}	
\newmacro{\iEdge}{\edge_{\iCount}}	
\newmacro{\jEdge}{\edge_{\jCount}}	
\newmacro{\kEdge}{\edge_{\kCount}}	
\newmacro{\nEdges}{E}	
\newmacro{\edges}{\mathcal{\nEdges}}	

\newmacro{\graph}{\mathcal{G}}	
\newmacro{\graphfull}{\graph(\vertices,\edges)}	

\newop{\minimize}{minimize}	
\newop{\opt}{Opt}	
\newop{\gap}{Gap}	

\newmacro{\cvx}{\mathcal{C}}	
\newmacro{\obj}{f}	
\newmacro{\sobj}{F}	
\newmacro{\oper}{A}	
\newmacro{\vecfield}{v}	

\newmacro{\subd}{\partial}	
\newmacro{\subsel}{\nabla}	
\newmacro{\gvec}{g}	

\newmacro{\gbound}{G}	
\newmacro{\vbound}{V}	
\newmacro{\lips}{L}	
\newmacro{\strong}{\mu}	
\newmacro{\smooth}{\beta}	

\newop{\tcone}{TC}	
\newop{\dcone}{\tcone^{\ast}}	
\newop{\ncone}{NC}	
\newop{\pcone}{PC}	
\newop{\hull}{\Delta}	

\newop{\ex}{\mathbb{E}}	
\newop{\prob}{\mathbb{P}}	
\newop{\Var}{Var}	
\newop{\cov}{cov}	
\newop{\simplex}{\Delta}	

\DeclarePairedDelimiterXPP{\exof}[1]{\ex}{[}{]}{}{
 #1}

\DeclarePairedDelimiterXPP{\exwrt}[2]{\ex_{#1}}{[}{]}{}{
 #2}

\DeclarePairedDelimiterXPP{\probof}[1]{\prob}{(}{)}{}{
 #1}

\DeclarePairedDelimiterXPP{\probwrt}[2]{\prob_{#1}}{(}{)}{}{
 #2}

\DeclarePairedDelimiterXPP{\oneof}[1]{\one}{\{}{\}}{}{#1}	

\DeclarePairedDelimiterXPP{\varof}[1]{\var}{[}{]}{}{
 #1}

\DeclarePairedDelimiterXPP{\covof}[1]{\cov}{(}{)}{}{
 #1}

\newmacro{\event}{\mathcal{E}}       
\newmacro{\eventalt}{H}       

\newmacro{\sample}{\omega}	
\newmacro{\samples}{\Omega}	
\newmacro{\filter}{\mathcal{F}}	
\newmacro{\probspace}{(\samples,\filter,\prob)}	

\newmacro{\pdist}{P}	
\newmacro{\history}{\mathcal{H}}	

\newmacro{\mean}{\mu}	
\newmacro{\sdev}{\sigma}	
\newmacro{\variance}{\sdev^{2}}	
\newmacro{\covmat}{\Sigma}	

\newmacro{\seq}{a}	
\newmacro{\seqalt}{b}	

\newmacro{\beforestart}{0}	
\newmacro{\start}{1}	
\newmacro{\afterstart}{2}	
\newmacro{\running}{\start,\afterstart,\dotsc}	

\newmacro{\run}{t}	
\newmacro{\runalt}{s}	
\newmacro{\runaltalt}{\tau}	
\newmacro{\nRuns}{T}	
\newmacro{\runs}{\mathcal{\nRuns}}	

\newop{\Nash}{Nash}	
\newop{\CE}{CE}	
\newop{\CCE}{CCE}	
\newop{\NI}{NI}	

\newmacro{\stratdiam}{\norm*{\strats}}
\newmacro{\distort}{\chi(\nPures)}

\newop{\brep}{BR}	
\newop{\val}{val}	

\newmacro{\play}{i}	
\newmacro{\playalt}{j}	
\newmacro{\iPlay}{i}	
\newmacro{\jPlay}{j}	
\newmacro{\kPlay}{k}	
\newmacro{\nPlayers}{N}	
\newmacro{\players}{\mathcal{N}}	

\newmacro{\pure}{\alpha}	
\newmacro{\purealt}{\beta}	
\newmacro{\nPures}{K}	
\newmacro{\pures}{\mathcal{A}_{\nPures}}	

\newmacro{\strat}{x}	
\newmacro{\stratalt}{\alt\strat}	
\newmacro{\strataux}{q}	
\newmacro{\strats}{\mathcal{X}}	
\newmacro{\intstrats}{\strats^{\circle}}	

\newmacro{\corr}{z}	
\newmacro{\corralt}{\alt\corr}	
\newmacro{\corrs}{\mathcal{Z}}	

\newmacro{\pay}{u}	
\newmacro{\loss}{\ell}	
\newmacro{\cost}{c}	
\newmacro{\pot}{f}	

\newmacro{\payvec}{w}	
\newmacro{\payv}{v}	
\newmacro{\payfield}{\payv}	
\newmacro{\paybound}{M}	
\newmacro{\payspace}{\mathcal{Y}}	
\newmacro{\payspacei}{\payspace_{\play}}	

\newmacro{\game}{\mathcal{G}}	
\newmacro{\gamefull}{\game(\players,\points,\pay)}	

\newmacro{\fingame}{\Gamma}	
\newmacro{\fingamefull}{\Gamma(\players,\pures,\pay)}	
\newmacro{\mixgame}{\Delta(\fingame)}	

\newmacro{\minmax}{L}	

\newmacro{\minvar}{\point_{1}}	
\newmacro{\minvaralt}{\alt\minvar}	
\newmacro{\minvars}{\points_{1}}	

\newmacro{\maxvar}{\point_{2}}	
\newmacro{\maxvaralt}{\alt\maxvar}	
\newmacro{\maxvars}{\points_{2}}	

\newmacro{\hreg}{h}	
\newmacro{\proxdom}{\points_{\hreg}}	

\newmacro{\breg}{D}	
\newmacro{\mprox}{P}	

\newmacro{\hconj}{h^{\ast}}	
\newmacro{\mirror}{Q}	
\newmacro{\fench}{F}	
\newmacro{\hker}{\theta}	

\newmacro{\hstr}{K}	
\newmacro{\hrange}{H}	

\newmacro{\learn}{\eta}	
\newmacro{\weight}{\lambda}	

\DeclarePairedDelimiterXPP{\bregof}[2]{\breg}{(}{)}{}{#1,#2}	
\DeclarePairedDelimiterXPP{\bregofX}[3]{\breg_{#1}}{(}{)}{}{#2,#3}	
\DeclarePairedDelimiterXPP{\fenchof}[2]{\fench}{(}{)}{}{#1,#2}	
\DeclarePairedDelimiterXPP{\fenchofX}[3]{\fench_{#1}}{(}{)}{}{#2,#3}	
\DeclarePairedDelimiterXPP{\proxof}[2]{\mprox_{#1}}{(}{)}{}{#2}	

\newmacro{\choice}{\mirror}	

\newmacro{\zone}{\mathbb{D}}	

\newop{\Eucl}{\Pi}	
\newop{\logit}{\Lambda}	
\newmacro{\dkl}{D_{\rm KL}}	

\newmacro{\dvec}{w}	
\newmacro{\dpoint}{y}	
\newmacro{\dpointalt}{\alt\dpoint}	
\newmacro{\dpoints}{\mathcal{Y}}	

\newmacro{\score}{y}	
\newmacro{\scorealt}{\alt\score}	
\newmacro{\scoreaux}{z}	
\newmacro{\scores}{\payspace}	

\newmacro{\momexp}{p}	

\newmacro{\state}{X}	
\newmacro{\dstate}{Y}	

\newmacro{\drift}{b}	
\newmacro{\diffmat}{\Sigma}	

\newmacro{\brown}{W}	
\newmacro{\ito}{M}	
\newmacro{\levy}{L}	

\newmacro{\model}{\hat\vecfield}	
\newmacro{\sgrad}{\gvec}	
\newmacro{\step}{\gamma}	
\newmacro{\runtime}{\tau}	

\newmacro{\stepexp}{\ell_{\step}}	
\newmacro{\learnexp}{\ell_{\learn}}	

\newmacro{\apt}{X}	
\DeclarePairedDelimiterXPP{\aptof}[1]{\apt}{(}{)}{}{#1}	
\DeclarePairedDelimiterXPP{\aptofX}[2]{\apt_{#1}}{(}{)}{}{#2}	

\newop{\orcl}{\mathsf{G}}	
\newop{\err}{\mathsf{\noise}}	
\newmacro{\seed}{\omega}	
\newmacro{\seeds}{\Omega}	

\newmacro{\noise}{U}	
\newmacro{\bias}{b}	

\newmacro{\bbound}{B}	
\newmacro{\totbound}{\paybound}	
\newmacro{\mombound}{V}	

\newmacro{\snoise}{\xi}	
\newmacro{\sbias}{\chi}	

\newmacro{\mix}{\delta}	
\newmacro{\perturb}{z}	
\newmacro{\pivot}{\point}	

\newop{\reg}{Reg}	
\newop{\preg}{\overline{Reg}}	

\newmacro{\bench}{p}	
\newmacro{\test}{p}	

\addauthor[Pan]{PM}{MediumBlue}

\newop{\skel}{skel}	

\newmacro{\prof}{\mathcal{A}}                  
\newmacro{\istrats}{\mathcal{X}_{\nPures}}    
\newmacro{\Kset}{\mathcal{K}}                 
\newmacro{\Kint}{\Kset^\circ}                 

\newmacro{\cpay}{g}                            
\newmacro{\pred}{m}                            
\newmacro{\perr}{e}                            
\newmacro{\zlift}{z}                           
\newmacro{\lmir}{Z}                            
\newmacro{\mass}{\lambda}                      
\newmacro{\disc}{\delta}                       
\newmacro{\lreg}{\psi}                         
\newmacro{\Lreg}{L}                            
\newmacro{\rrange}{\Omega}                     
\newmacro{\diff}{\Delta}                       
\newmacro{\dsum}{\sigma}                       
\newmacro{\dop}{\kappa}                        
\newmacro{\jac}{J}                             
\newmacro{\remd}{r}                            
\newmacro{\PV}{P}                              
\newmacro{\NV}{G}                              
\newmacro{\EE}{E}                              

\newmacro{\areg}{\varphi}                       
\newmacro{\Gmat}{\Gamma}                        
\newmacro{\Dfun}{\mathcal{D}}                   
\newmacro{\Jfun}{\mathcal{J}}                   
\newmacro{\Fpot}{\Phi}                          
\newmacro{\gext}{\gamma}                        
\newmacro{\ctr}{\mathcal{T}}                    

\newmacro{\cont}{\mathcal{C}}                   
\newmacro{\Fset}{\mathcal{F}}                    
\newmacro{\adv}{H}                              

\newop{\softmax}{softmax}
\newop{\polylog}{polylog}
\newcommand{\opnorm}[3]{\left\lVert #1\right\rVert_{#2\to #3}}

\usepackage{tabularx}

\theoremstyle{theorem}
\newtheorem*{informaltheorem}{Main Theorem}

\crefname{algorithm}{Algorithm}{Algorithms}
\crefname{definition}{Definition}{Definitions}
\crefname{lemma}{Lemma}{Lemmas}
\crefname{proposition}{Proposition}{Propositions}
\crefname{appendix}{Appendix}{Appendices}

\title
[Constant Regret in General Games]
{Constant Regret in General Games via Higher-Order Optimism}

\author
[O.~Abbadi]
{Omar Abbadi$^{c,\ast, \sharp}$}
\address{$^{c}$\,%
Corresponding author.}
\address{$^{\ast}$\,%
Univ. Mohammed VI Polytechnic, CMSIS, MCGT, 11103 Rabat, Morocco.}
\email{omar.abbadi@um6p.ma}
\author
[R.~Laraki]
{Rida Laraki$^{\ast}$}
\email{rida.laraki@um6p.ma}
\author
[P.~Mertikopoulos]
{Panayotis Mertikopoulos$^{\sharp}$}
\address{$^{\sharp}$\,%
Univ. Grenoble Alpes, CNRS, Inria, Grenoble INP, LIG, 38000 Grenoble, France.}
\email{panayotis.mertikopoulos@imag.fr}

\date{}


\newacro{HOOD}{higher-order optimism with discounting}
\newacro{FTRL}{follow-the-regularized-leader}
\newacro{RVU}{regret bounded by variation in utilities}
\newacro{CCE}{coarse correlated equilibrium}
\newacro{OFTRL}[OptFTRL]{optimistic follow-the-regularized-leader}

\begin{document}

\begin{abstract}
We introduce an uncoupled learning algorithm which, when employed by all players of an \emph{arbitrary} $\nPlayers$-player normal form game with up to $\nPures$ actions per player, guarantees $\bigoh(\nPlayers^3\log^2 \nPures)$ individual regret, uniformly over the horizon of play.
The proposed algorithm\textemdash which we call \acdef{HOOD}\textemdash is a variant of \acf{OFTRL} that combines a discounted $(\nPlayers+1)$-th order predictor with entropic regularization over a suitable ``lifting'' of the game's strategy space.
This combination of ingredients is purposefully designed to dampen large oscillations of the induced sequence of play in a controlled manner, removing in this way a key stumbling block of previous attempts to achieve constant regret in general games.
Our approach bears several striking similarities to the concurrent\textemdash and completely independent\textemdash work of \citet{LFO26}, who very recently derived an $\bigoh(\nPlayers^{21} \log^{4} \nPures)$ regret bound through the use of higher-order optimism and an exponential moving average estimator.
\end{abstract}

\allowdisplaybreaks
\acresetall
\acused{HOOD}\acused{FTRL}\acused{RVU}\acused{CCE}
\maketitle

\section{Introduction}\label{sec:intro}

The standard model for learning in games typically unfolds as follows:
at each instance of a repeated decision process
\begin{enumerate*}
[(\itshape i\hspace*{1pt}\upshape)]
\item
every player chooses an action;
\item
they receive any rewards and/or feedback generated from their chosen actions;
\item
they update their strategies, and the process repeats.
\end{enumerate*}
In this general setting, the overarching objective of each player is to improve their individual payoff over time, ideally reaching a state where no further improvements are possible or, at the very least\textemdash and, perhaps, more pragmatically\textemdash not regretting their past choices.
This last criterion leads to the seminal notion of (external) \emph{regret}, first introduced in statistical decision-making by \citet{Wal50} and \citet{Sav51}, and subsequently rediscovered in the context of games and approachability by \citet{Han57} and \citet{Bla54}, respectively.

At a high level, the individual regret of a player provides a natural ``worst-case'' performance benchmark, as it compares the player's cumulative payoff over time to that of the best fixed action in hindsight.
Accordingly, these early works also established the existence of \emph{no-regret} learning procedures, \ie update rules whose regret grows sublinearly with the horizon of play in \emph{any} ``game against Nature''\textemdash that is, against \emph{any} possible behavior of the learner's opponents, rational or otherwise.
Nonetheless, since the behavior of no-regret learners in a normal form game is \emph{not} arbitrary but shaped by the actions of all other players, an important question which has remained open since these early days of the field is whether all players following a no-regret learning rule can prevent individual regret from accumulating indefinitely over time.
Somewhat more formally:
\smallskip
\begin{quote}
\centering
\itshape
Are there learning rules which, if followed by all players,\\
guarantee constant individual regret to all players, in all games?
\end{quote}
\smallskip

An important caveat in the above is that any such rule should be \emph{uncoupled} in the sense of \citet{HMC03,HMC06}, \ie the strategy adjustment of any given player should not depend on the payoff functions of the other players (though it may of course depend on the payoff function of the focal player, and through it, the strategies of other players).
An additional consideration is that the feedback available to the players must be sufficiently informative to allow them to leverage the (optimistically) mild evolution of their learned strategies over time.
On that account, we work exclusively in the deterministic, full-information setting, where each players receives as feedback their mixed payoff vector at each stage of the process (but no further information);
without this assumption, standard information-theoretic lower bounds would exclude the possibility of achieving constant regret, except in very special cases \cite{CBL06,BCB12}.

\para{Our contributions in the context of related work}

Our main contribution is that, with these qualifications in mind, \emph{the driving question above can be answered in the affirmative.}
In more detail, we propose an uncoupled, full-information learning rule, which we call \acdef{HOOD}, and which enjoys the following universal guarantee:

\begin{informaltheorem}[Informal version]
If all players of a finite $\nPlayers$-player game with up to $\nPures$ actions per player follow \ac{HOOD}, the individual regret of each player $\play=1,\dotsc,\nPlayers$ is bounded for all $\horizon$ as
\begin{equation}
\reg_{\play}(\horizon)
    \leq 45 (\nPlayers+1)^{3}
        \bracks{16 + 4 \log\nPures + \log^{2} \nPures}
    = \bigoh(\nPlayers^{3} \log^2\nPures)
    \eqstop
\end{equation}
\end{informaltheorem}

An additional appealing feature of the proposed algorithm is that it is horizon-free, \ie it does not require prior knowledge of the horizon of play.
As a result, for any fixed $\nPlayers$ and $\nPures$, each player incurs only a finite amount of regret over the entire (possibly infinite) trajectory of play;
equivalently, regret does not continue to accumulate as play unfolds.
On that account, by standard results in the field, this also means that the empirical distribution of play converges to the game's set of coarse correlated equilibria at a rate of $\bigoh(1/\horizon)$.

To provide the proper context for our contribution, it is instructive to begin with the so-called ``adversarial'' setting, in which a single learner is involved in a ``game against Nature''\textemdash that is, they are facing an arbitrary sequence of payoff vectors, not necessarily coming from some underlying normal form game.
In this case, the full-information version of the exponential weights algorithm\textemdash known in this context as \textsc{Hedge} \citep{ACBFS95,ACBFS02,FS97,FS99,FV99,CBL06}\textemdash guarantees $\bigoh(\sqrt{\horizon\log \nPures})$ regret, compared to $\bigoh(\sqrt{\horizon \nPures\log\nPures})$ regret in the bandit case, where the learner only has access to their realized payoff.
In either case however, the $\bigoh(\sqrt{\horizon})$ dependence is unavoidable without imposing further assumptions on how the sequence of payoff vectors is generated \cite{CBL06,BCB12}.

This is where learning in games enters the stage.
If every player follows an iterative learning algorithm, the players' individual strategy adjustments are likely to be relatively small from one epoch to the next, so the induced sequence of payoff vectors faced by each player is likewise expected to be more predictable.
This predictability can be leveraged by algorithms with an ``extrapolation-based'' structure, that is, where players attempt to anticipate their future payoff landscape in order to take a more informed update step at each iteration\textemdash like the original extra-gradient algorithm of \citet{Kor76}, the single-query variant of \citet{Pop80}, and the like.
In the context of regret minimization, the most relevant algorithmic scheme is the \acdef{OFTRL} method proposed by \citet{RS13a,RS13} to take advantage of predictable payoff sequences, and which was shown to lead to $\bigoh(1)$ regret in two-player zero-sum games under self-play.

This constant regret result was subsequently expanded upon by \citet{SALS15}, and \citet{Tsu26a} recently obtained optimal bounds for the Optimistic Hedge algorithm in the same setting.
Constant regret has also been achieved in several broader structured classes of games, including variationally stable games \citep{HAM21}, strategically zero-sum and potential games \citep{APFS22} and, more recently, harmonic games \citep{LMPP+24}.
This array of results shows that attaining $\bigoh(1)$ regret \textit{is} possible in the context of game-theoretic learning, but they are all contingent on a certain, amenable ``variational'' structure.

The case of \emph{arbitrary} games is much more challenging due to the lack of such structure and lack of control over the iterates \citep{LM13,ALM26,BM23}.
One of the key difficulties here is that the movement of one player can unpredictably amplify the movement of another over time;
hence, for a long time, it was not known whether the presence of an underlying game was sufficient to overcome the adversarial worst-case $\bigoh(\sqrt{\horizon})$ regret bound.
The first general breakthrough here was the stability analysis of \citet{SALS15} who showed that, for arbitrary multiplayer normal form games, it is possible to achieve $\bigoh(\horizon^{1/4})$ individual regret, thus breaking the adversarial $\sqrt{\horizon}$ barrier.
\citet{CP20a} later improved the exponent to $1/6$ for two-player general-sum games.

The next breakthrough in this thread came from the work of \citet{DFG21} who showed that, somewhat surprisingly, the Optimistic Hedge algorithm actually inncurs only \emph{polylogarithmic} regret in \emph{any} games\textemdash the first ``near-constant'' result of its kind\footnote{A somewhat related result is due to \citet{PSS22}, who obtain bounded regret for an uncoupled learning dynamic along an $O(\log \horizon)$-sparse subsequence.}.
The proof of \citet{DFG21} introduced in particular a higher-order smoothness analysis showing that the successive temporal differences of the generated sequence of play can be controlled to sufficiently high order.
Subsequent work then simplified and strengthened this picture:
\citet{FALL+22} reduced the individual regret bound to $\bigoh(\log \horizon)$ in general convex games, \citet{AFKL+22} obtained the same $\bigoh(\log \horizon)$ horizon dependence for the stronger notion of swap regret in finite games, and, more recently \citet{SPF25a,SPF25b} reduced the action dependence to polylogarithmic while keeping a $\bigoh(\log \horizon)$ dependence on the horizon.

Following these developments, the question which naturally remains is whether we can make the ultimate step of descending from unbounded to constant regret in general games, or if there is a fundamental limitation that goes beyond the algorithms and approaches used so far.
Our paper provides a bridge for this gap by showing that, under \ac{HOOD}, the players' regret can be bounded uniformly over time.
In this sense, it completes the sequence of horizon improvements for uncoupled learning in arbitrary games under full information feedback:
\begin{equation}
\bigoh(\sqrt{\horizon})
    \;\longrightarrow\;
    \bigoh(\horizon^{1/4})
    \;\longrightarrow\;
    \;\bigoh(\polylog(\horizon))
    \;\longrightarrow\;
    \bigoh(\log \horizon)
    \;\longrightarrow\;
    \bigoh(1).
\end{equation}
We summarize this progression in \Cref{tab:rates}, where we also describe the specific \emph{ex ante} tuning required for each algorithm.%
\footnote{We note here in passing that concurrent work has obtained related improvements for stronger or alternative notions of regret:
\citet{Tsu26b} obtained $\bigoh(\nPlayers\nPures^2\sqrt{\log \nPures\log \horizon})$ \emph{swap regret} in general games, while \citet{TZ26} achieved an optimal $\bigoh(\log \nPures)$ \emph{alternating regret} bound.
However, because these concern a different measure of regret, we do not include them in \cref{tab:rates}.}
Note also that, in addition to removing the horizon dependence of the bound, our update uses fixed, horizon-independent parameters.

There is one concurrent paper that appears in \cref{tab:rates} but which we did not discuss above, and which achieves constant regret in any game.
Gvien its importance, we discuss it in detail in the end of this section.

\begin{table}[tbp]
\centering
\small
\renewcommand{\arraystretch}{1.18}
\begin{tabularx}{\textwidth}{
    @{}
    >{\raggedright\arraybackslash}p{0.25\textwidth}
    >{\raggedright\arraybackslash}p{0.32\textwidth}
    >{\raggedright\arraybackslash}X
    @{}
}
\toprule
\textbf{Work} & \textbf{Individual regret} & \textbf{Tuning} \\
\midrule
\citet{SALS15}
& $\bigoh(\sqrt{\nPlayers}\,\log \nPures\,\horizon^{1/4})$
& static, depends on $\nPlayers$ and $\horizon$ only \\

\citet{DFG21}
& $\bigoh(\nPlayers\log \nPures\,\log^4 \horizon)$
& static, depends on $\nPlayers$ and $\horizon$ only \\

\citet{FALL+22}
& $\bigoh(\nPlayers\nPures\log \horizon)$
& static, depends on $\nPlayers$ only \\

\citet{SPF25a}
& $\bigoh(\nPlayers\log^2 \nPures\,\log \horizon)$
& adaptive, depends on $\nPlayers$ and $\nPures$ only \\

\citet{LFO26}
& $\bigoh(\nPlayers^{21}\log^4 \nPures)$, uniformly in $\horizon$
& static, depends on $\nPlayers$ and $\nPures$ only \\

This paper
& $\bigoh(\nPlayers^3\log^2 \nPures)$, uniformly in $\horizon$
& static, depends on $\nPlayers$ and $\nPures$ only \\
\bottomrule
\end{tabularx}
\caption{Selected progression of individual regret bounds for arbitrary finite games.}
\label{tab:rates}
\end{table}

\para{What goes on under the \ac{HOOD}}
The idea behind our algorithm is to make the optimistic prediction more stable, and therefore more predictable, through a carefully tuned combination of regularization, discounting, and past payoff information.
In particular, we choose the optimistic predictor in a way that turns the induced prediction error into a geometrically discounted finite difference of order $\nPlayers+1$, compared to the classical optimistic predictor of \cite{Pop80,RS13,RS13a}, which uses only the previous payoff as its prediction, and leads to a first-order difference in time.
This higher-order structure provides enough regularity to control the prediction error via the method's regularization mechanism, which is entropic in nature but defined on a lifted space, following a technique introduced by \cite{FALL+22}.
This lifting additionally allows us to control the positive part of the regret rather than the signed regret, so large positive regret accumulated by certain players cannot be offset in the aggregate by large negative regret contributions by others.

The resulting control over the prediction error is then obtained through an expansion which continues only for terms that have not yet been controlled by either Bregman variation or prediction error terms that are proportional to the step size. The termination of this expansion follows from a somewhat surprising interaction between the geometry of the lifted regularization and the combinatorics of the higher-order extrapolation, whereby every term that remains uncontrolled must successively introduce a new player label, because terms associated with labels that have already appeared can already be controlled. Since there are only $\nPlayers$ players, no term in the expansion can remain uncontrolled beyond the $(\nPlayers+1)$st difference. Moreover, since the \ac{HOOD} algorithm guarantees constant individual regret, it can easily be upgraded to an adversarially robust algorithm by adding a switching rule that further leverages the boundedness of the regret.

\para{The work of \citet{LFO26}}

Three days before the submission of the current work, we were made aware of a concurrent paper by \citet{LFO26}, who propose an algorithm achieving $\bigoh(\nPlayers^{21}\log^4 \nPures)$ individual external regret in any $\nPlayers$-player finite game with up to $\nPures$ actions per player.
Even though our two papers were developed completely independently, they exhibit certain striking similarities, which we explain below.

First, the ECHO-OFTRL algorithm of \citet{LFO26} also employs a lifted regularization on the simplex in order to identify the players' lifted regret with the positive part of their external regret, and also replaces the usual one-step prediction error by a geometrically stabilized high-order difference.
The main differences with our approach here are in the way that these ingredients mesh together and the resulting analysis:
On the one hand, the \acs{HOOD} predictor encodes the required discounted differences directly in an $(\nPlayers+1)$-th order recurrence, and its analysis uses elementary bounds for this recurrence along with an explicit term-by-term expansion of players' payoff functions, with the lifted regularizer supplying bounds that control each term once a player label is repeated in the expansion, so the argument terminates after finitely many steps.
By contrast, the ECHO-OFTRL algorighm builds its predictor as an $\nPlayers$-stage exponential moving average (EMA) cascade, motivated by signal processing and transfer function design, and analyzed via time-domain filter estimates, EMA kernels and weighted convolutions, together with a recursive weighted variation argument.

Thus, even though both proofs ultimately exploit the number of players present in the game, our approach does so through a more direct algebraic expansion rather than through a recursive weighted variation argument, and our predictor and lifted regularizer are both designed around this choice.
This more direct construction ultimately leadso to a sharper bound in the case of \ac{HOOD}\textemdash $\bigoh(\nPlayers^3\log^2\nPures)$ regret for each player compared to $\bigoh(\nPlayers^{21}\log^4\nPures)$ in the case of \citet{LFO26}.

\section{Preliminaries}\label{sec:prelim}

\subsection{General notation}
We write $\one$ for the vector of all ones, $\R_+^d$ and $\R_{++}^d$ for the nonnegative and positive orthants, respectively, $\braket*{u}{v}$ for the usual Euclidean pairing, and, for $\nPlayers \in \N$, $[\nPlayers]= \{1, \dots, \nPlayers \}$. For $p\in[1,\infty]$, $\norm*{z}_p$ denotes the usual $\ell_p$ norm.
If $A:\R^a\to\R^b$ is linear and $p,q\in[1,\infty]$, we write the operator norm
\begin{equation}
 \opnorm{A}{p}{q}=\sup_{\norm*{u}_p\le1}\norm*{Au}_q.
\end{equation}
We also write $[r]_+=\max\{r,0\}$, and $A\lesssim B$ (respectively $A\gtrsim B$) when $A\le CB$ (respectively $A\ge CB$) for a universal constant $C>0$, where dependence on an additional parameter is indicated by a subscript, and we use the symbol $\odot$ to denote coordinatewise multiplication.  For probability vectors $p,q$ on the same finite set, we write
\begin{equation}
  \dkl(p\|q)=\sum_{\pure:p_\pure>0}p_\pure
       \log\frac{p_\pure}{q_\pure},
\end{equation}
for the Kullback--Leibler divergence, and, for a differentiable convex function $f$, we write
\begin{equation}
   \bregofX{f}{z'}{z}=f(z')-f(z)-\braket*{\nabla f(z)}{z'-z}
\end{equation}
for its Bregman divergence. Finally, for a set $S$, we denote its convex hull by $\operatorname{conv}(S)$.

\subsection{Finite games}
We consider a finite normal form game with player set $\players = [\nPlayers]$.  For notational simplicity, we assume throughout that every player has the same pure action set $\pures = [\nPures]$, $\nPures \ge 2$, and we set $\prof=\pures^\nPlayers$ for the pure profile set; this entails no loss for our regret guarantees because unequal action sets may be padded with duplicate actions.  Accordingly, player $i$'s mixed strategy space is the probability simplex
\begin{equation}
   \istrats=\{\strat\in\R_+^\nPures:\braket*{\one}{\strat}=1\},
\end{equation}
and $\strats = \istrats^\nPlayers$ is the \emph{strategy space}. We moreover write $\strat=(\strat_1,\ldots,\strat_\nPlayers)\in\strats$ for a mixed profile, $\strat_{-i}=(\strat_j)_{j\ne i}$ for the opponents' component, and $\strat_{i\pure_i}$ for the probability that player $i$ plays pure action $\pure_i \in \pures$ under $\strat$.  As usual, we identify each pure action $\pure\in\pures$ with the corresponding vertex of $\istrats$. Player $i$'s payoff function is $\pay_i:\prof\to[-1,1]$, and we extend it multilinearly to $\strats$ so that $\pay_i(\strat)$ is the expected payoff under the product distribution induced by $\strat$:
\begin{equation} \pay_i(\strat) = \ex_{\pure \sim \strat}[\pay_i(\pure)] = \sum_{\pure \in \prof} \strat_\pure \pay_i(\pure), \end{equation}
where, for a pure profile $\pure = (\pure_1, \dots, \pure_\nPlayers) \in \prof$, $\strat_\pure = \prod_{i \in \players} \strat_{i\pure_i}$ is the probability that $\pure$ gets played under the mixed profile $\strat$.

\begin{remark}[Payoff normalization]
The restriction $\pay_i(\prof)\subseteq[-1,1]$ is essentially without loss of generality. Since $\prof$ is finite, every $\pay_i:\prof\to\R$ is bounded, and replacing $\pay_i$ by $\widehat \pay_i=a_i \pay_i+b_i$ with $a_i>0$ rescales regret as
\begin{equation}
    \widehat{\reg}_i(\horizon)=a_i\reg_i(\horizon).
\end{equation}
Thus the regret guarantees extend to arbitrary payoffs up to the corresponding rescaling factor.
\endenv
\end{remark}

\subsection{Feedback and regret}
At every round $t=1,2,\ldots$, player $i$ chooses a mixed strategy $\strat_{i,t}\in\istrats$, and we write $\strat_t=(\strat_{1,t},\ldots,\strat_{\nPlayers,t})$ for the joint mixed profile.  We assume full information feedback: after play, player $i$ observes the payoff vector $\payv_{i,t}\in[-1,1]^\nPures$ with coordinates
\begin{equation}
     \payv_{i\pure,t}=\pay_i(\pure,\strat_{-i,t}),
     \qquad \pure\in \pures.
\end{equation}
By multilinearity, the payoff obtained at round $t$ is therefore
\begin{equation}\label{eq:pay}
     \pay_i(\strat_t)=\braket*{\strat_{i,t}}{\payv_{i,t}}.
\end{equation}
Moreover, for the update and the subsequent analysis, it will be handy to center the payoff vector by the payoff actually obtained by the player in expectation.  As such, we set
\begin{equation}\label{eq:g}
 \cpay_{i,t}=\payv_{i,t}-\pay_i(\strat_t)\one.
\end{equation}
Coordinatewise, this gives
\begin{equation}\label{eq:gcoord}
 \cpay_{i\pure,t}=\pay_i(\pure,\strat_{-i,t})-\pay_i(\strat_t),
\end{equation}
and therefore
\begin{equation}\label{eq:gprop}
 \braket*{\strat_{i,t}}{\cpay_{i,t}}=0,
 \qquad \norm*{\cpay_{i,t}}_\infty\le2.
\end{equation}
Finally, player $i$'s external regret is:
\begin{equation}\label{eq:reg}
  \reg_i(\horizon)=\max_{\pure\in \pures}
  \sum_{t=1}^\horizon\bigl[\pay_i(\pure,\strat_{-i,t})-\pay_i(\strat_t)\bigr]= \max_{\pure\in \pures}\sum_{t=1}^\horizon \cpay_{i\pure,t}.
\end{equation}
We recall that, if every player satisfies $\reg_i(\horizon)\le\eps(\horizon)$ after $\horizon$ rounds, the empirical distribution of play is an $\eps(\horizon)/\horizon$-approximate coarse correlated equilibrium (\acs{CCE}) \cite{CBL06,Rou16}.  As a direct consequence, a regret guarantee that is uniform in $\horizon$ implies $\bigoh(1/\horizon)$ convergence to the \acs{CCE} set.

\section{Algorithm and main result}\label{sec:alg}

We shall now describe the learning rule at play.

\subsection{Lifted regularization}
We start with the regularization, which we define in a \enquote{lifted space} that contains the origin, in the spirit of \cite{FALL+22}. Define the $\nPures$-dimensional simplex
\begin{equation}
     \Kset=\{z\in\R_+^\nPures:\braket*{\one}{z}\le1\},
     \quad
     \text{and}\quad \Kint=\{z\in\R_{++}^\nPures:0<\braket*{\one}{z}<1\},
\end{equation}
so that $\Kint$ is the relative interior of $\Kset$.  Every nonzero $z\in\Kset$ has the unique representation $z=\mass \strat$, where $\mass=\braket*{\one}{z} \in(0,1]$ and $\strat\in\istrats$.  For $\strat \in \istrats$, let $\hreg$ be the \emph{negative entropy}
\begin{equation}
 \hreg(\strat)=\sum_{\pure \in \pures} \strat_\pure\log \strat_\pure,
\end{equation}
with $0\log0=0$, and set
\begin{equation}
 \Lreg=4+\log \nPures+\frac14\log^2 \nPures.
\end{equation}
For a nonzero $z=\mass \strat\in \Kset$, define the \emph{lifted entropic regularizer}
\begin{equation}\label{eq:lreg}
 \lreg(z)=(\mass+\sqrt\mass)\hreg(\strat)
 -3\Lreg\sqrt{\mass(1-\mass)},
 \qquad \lreg(0)=0.
\end{equation}
This regularizer induces the \emph{lifted mirror map}
\begin{equation}\label{eq:lmir}
 \lmir(\score)=\argmax_{z\in \Kset}\{\braket*{\score}{z}-\lreg(z)\},
\end{equation}
and we call
\begin{equation}
 \choice(\score)=\frac{\lmir(\score)}{\braket*{\one}{\lmir(\score)}}
\end{equation}
the \emph{choice map}.  This corresponds to a lifted version of the usual regularized choice map from standard \acs{FTRL} \citep{MS16,KM17}.  Note also that, conditional on the optimal mass $\mass$, the maximization over $\strat$ is explicit. Indeed, the entropic nature of the regularizer implies that the choice map can be written as a softmax with a $\mass$-dependent temperature:
\begin{equation}\label{eq:chomap}
 \choice(\score)=\softmax\!\left(\frac{\sqrt{\mass}}{1+\sqrt{\mass}}\score\right).
\end{equation}
As such, the only additional optimization introduced by $\mass$ is an easy one-dimensional problem.

\subsection{The algorithm}
We denote the algorithm's \emph{step size} by $\learn>0$.  We shall now define the \emph{predictor}:
initialize $\cpay_{i,k}=\pred_{i,k}=0$ for $-\nPlayers\le k\le0$, then, at time $t$, form the prediction
\begin{equation}\label{eq:pred}
 \pred_{i,t}=
 \sum_{k=1}^{\nPlayers+1}(-1)^{k+1}\binom{\nPlayers+1}{k}
 \bigl[(1-\disc^k)\cpay_{i,t-k}+\disc^k \pred_{i,t-k}\bigr],
\end{equation}
where its \emph{discount factor} $\disc$ is fixed by the number of players:
\begin{equation}
       \disc=\frac{\nPlayers}{\nPlayers+1}.
\end{equation}
We moreover write
\begin{equation}\label{eq:perr}
 \perr_{i,t}=\cpay_{i,t}-\pred_{i,t}
\end{equation}
for its \emph{prediction error}. We then define the \acdef{HOOD} algorithm as \cref{alg:hood}.

\begin{algorithm}[h]
\caption{\acs{HOOD}}\label{alg:hood}
\begin{algorithmic}[1]
\State $\score_{i,0}\gets0$, $\cpay_{i,k},\pred_{i,k}\gets0$ for $-\nPlayers\le k\le0$
\For{$t=1,2,\ldots$}
 \State compute $\pred_{i,t}$ from \eqref{eq:pred}, set $\score_{i,t-1/2} \gets \score_{i,t-1}+\pred_{i,t}$
 \State update $\strat_{i,t}\gets \choice(\learn \score_{i,t-1/2})$, play $\strat_{i,t}$
 \State observe $\payv_{i,t}$, set
 $\cpay_{i,t}\gets \payv_{i,t}-\pay_i(\strat_t)\one$ and
 $\score_{i,t}\gets \score_{i,t-1}+\cpay_{i,t}$
\EndFor
\end{algorithmic}
\end{algorithm}

The resulting rule is uncoupled, in the sense that player $i$ only uses the game dimensions $\nPlayers,\nPures$, its own mixed action and its observed payoff vectors. Moreover, neither the update nor any of its parameters requires the horizon $\horizon$.

\subsection{Regret guarantee} We may now state our main result:

\begin{theorem}\label{thm:main}
Let \begin{equation}\Lreg=4+\log \nPures+\frac14\log^2 \nPures.\end{equation} Suppose every player uses \Cref{alg:hood} with a step size satisfying
\begin{equation}\label{eq:eta}
 0<\learn\le\frac{1}{60(\nPlayers+1)^3}.
\end{equation}
Then, for every player $i$,
\begin{equation}\label{eq:bd}
 \sup_{\horizon\ge1}\reg_i(\horizon)
 \le
 \frac{3\Lreg}{\learn}.
\end{equation}
In particular, for the largest choice $\learn=1/[60(\nPlayers+1)^3]$,
\begin{equation}\label{eq:bdmax}
 \sup_{\horizon\ge1}\reg_i(\horizon)
 \le 45 (\nPlayers+1)^{3}
        \bracks{16 + 4 \log\nPures + \log^{2} \nPures}.
\end{equation}
Consequently,
\begin{equation}\label{eq:rate}
 \sup_{\horizon\ge1}\reg_i(\horizon)=\bigoh(\nPlayers^3\log^2\nPures).
\end{equation}
\end{theorem}

Note that the admissible step size range in \eqref{eq:eta} is independent of the horizon $\horizon$. By the earlier regret-to-\acs{CCE} remark, we also have the following immediate consequence:

\begin{corollary}\label{cor:cce}
Under the conditions of \Cref{thm:main}, let $\mu_\horizon=\horizon^{-1}\sum_{t=1}^\horizon\bigotimes_i \strat_{i,t}$. Then $\mu_\horizon$ is a $\frac{3\Lreg}{\learn \horizon}$-approximate coarse correlated equilibrium.
\end{corollary}

\begin{remark}[Adversarial robustness]
In the previous theorem, all players are assumed to use \Cref{alg:hood}. We can nonetheless exploit the boundedness of its regret guarantee to obtain robustness against adversarial payoff sequences via a simple switching rule. The idea is for each player to follow \Cref{alg:hood} until the regret bound is exceeded, and then switch permanently to a learning algorithm with adversarial regret guarantees. The details are given in \Cref{app:switch}.
\endenv
\end{remark}

\section{Proof overview}\label{sec:proof}

In this section we overview the proof of \Cref{thm:main}, where the complete proofs are given in the appendices. The proof is based on the following four observations:

\begin{enumerate}
\item First, the usual \acdef{RVU} inequality for \acdef{OFTRL} suggests the two quantities that must be related. Writing $x_{i,t-1/2}$ for the optimistic iterate and $x_{i,t}$ for the corresponding unoptimistic iterate, it gives
\begin{align}
 \reg_i(\horizon)
 &\le
 \frac{\rrange}{\learn}
 +
 \frac{\learn}{2}
 \sum_{t=1}^{\horizon}\norm*{\perr_{i,t}}_\infty^2
 \\
 &\qquad
 -
 \frac1{\learn}\sum_{t=1}^{\horizon}
 \left[
  \bregofX{\hreg}{x_{i,t-1/2}}{x_{i,t-1}}
  +
  \bregofX{\hreg}{x_{i,t}}{x_{i,t-1/2}}
 \right],
 \label{eq:rvu-schematic}
\end{align}
where $\rrange$ is the range of the regularizer. The last two Bregman divergences measure the movement of the optimistic iterate to its two neighboring unoptimistic iterates. Thus, if the regret term on the left were nonnegative, \eqref{eq:rvu-schematic} would bound this Bregman movement in terms of $\rrange$ and the prediction error. Bounded regret would then follow from a complementary bound controlling the prediction error by this movement, with sufficiently small coefficients to close the two inequalities.
\item Second, following the lifting idea of \citet{FALL+22}, the first issue can be resolved by performing the regularized update on the lifted simplex $\Kset$ introduced in \Cref{sec:alg}. Since $\Kset$ contains the origin and the payoff vectors are centered as in \eqref{eq:g}, the regret term in the resulting \acs{RVU} inequality becomes $[\reg_i(\horizon)]_+$ and is therefore nonnegative.
\item Third, the needed prediction error bound can be obtained using higher-order time differences. This part is inspired by the higher-order smoothness argument of \citet{DFG21} where, as in their proof, successive payoff differences are related to strategy differences, which are then related back to payoff terms through the learning map. Here, instead of proving that sufficiently high finite differences of the resulting dynamics are small, we choose the predictor so that its prediction error already contains $\nPlayers+1$ time differences. Conditional on suitable local bounds for the learning map, every term that remains uncontrolled after one more difference must introduce a new player label. Since there are only $\nPlayers$ players, this expansion terminates after finitely many steps. The discounting in \eqref{eq:pred} then prevents these higher-order differences from introducing an exponential dependence on $\nPlayers$.
\item Fourth, this expansion requires local bounds that ordinary entropic \acs{OFTRL} satisfies on the simplex: the Jacobian of the choice map must be small (that is, proportional to the step size), Jacobian time differences and Taylor remainders must be controlled by Bregman movement, and a newly revealed variation of a player whose Jacobian already occurs in the expansion must also be controlled by Bregman movement. A naive lift does not preserve the last bound when the optimizing mass becomes small, so the particular regularizer \eqref{eq:lreg} is designed precisely so that all of these controls remain valid after lifting.
\end{enumerate}

We now make these four observations more explicit.

\subsection{The lifted \acs{RVU} inequality}

For each player $i$, recall from \Cref{alg:hood} the optimistic score $\score_{i,t-1/2}=\score_{i,t-1}+\pred_{i,t}$ and the cumulative score $\score_{i,t}=\score_{i,t-1}+\cpay_{i,t}$. Define the corresponding lifted iterates
\begin{equation}
 \zlift_{i,t-1/2}
 =
 \lmir(\learn\score_{i,t-1/2}),
 \qquad
 \zlift_{i,t}
 =
 \lmir(\learn\score_{i,t}),
\end{equation}
and the Bregman movements
\begin{equation}
 \PV_{i,t}
 =
 \bregofX{\lreg/\learn}{\zlift_{i,t-1/2}}{\zlift_{i,t-1}},
 \qquad
 \NV_{i,t}
 =
 \bregofX{\lreg/\learn}{\zlift_{i,t-1/2}}{\zlift_{i,t}}.
\end{equation}
For a horizon $\horizon$, set
\begin{equation}\label{eq:PNE-overview}
 \PV(\horizon)
 =
 \sum_{i=1}^{\nPlayers}\sum_{t=1}^{\horizon}\PV_{i,t},
 \qquad
 \NV(\horizon)
 =
 \sum_{i=1}^{\nPlayers}\sum_{t=1}^{\horizon}\NV_{i,t},
 \qquad
 \EE(\horizon)
 =
 \sum_{i=1}^{\nPlayers}\sum_{t=1}^{\horizon}
 \norm*{\perr_{i,t}}_\infty^2,
\end{equation}
where $\perr_{i,t}$ is the prediction error from \eqref{eq:perr}, and write
\begin{equation}
 \rrange
 =
 \max_{z\in\Kset}\lreg(z)-\min_{z\in\Kset}\lreg(z).
\end{equation}

The usual \acs{OFTRL} analysis already contains a positive term measuring prediction error and a negative Bregman movement term, but the obstacle to rearranging this inequality is that external regret itself can be negative, so the role of the lift is then to remove this obstacle. Indeed, \eqref{eq:gprop} gives $\braket*{\strat_{i,t}}{\cpay_{i,t}}=0$, while $\Kset=\conv(0,e_1,\ldots,e_{\nPures})$. Hence the comparator term in the lifted problem is
\begin{equation}\label{eq:auxreg-overview}
 \max_{z\in\Kset}\braket*{z}{\score_{i,\horizon}}
 =
 \max\left\{
 0,\max_{\pure\in\pures}
 \sum_{t=1}^{\horizon}\cpay_{i\pure,t}
 \right\}
 =
 [\reg_i(\horizon)]_+.
\end{equation}
The lifted \acs{RVU} inequality proved in \Cref{app:oftrl} therefore gives
\begin{equation}\label{eq:energy-overview}
 \sum_{i=1}^{\nPlayers}[\reg_i(\horizon)]_+
 \le
 \frac{\nPlayers\rrange}{\learn}
 +
 \frac{\learn}{2}\EE(\horizon)
 -
 \PV(\horizon),
 \qquad
 \NV(\horizon)
 \le
 \frac{\learn}{2}\EE(\horizon).
\end{equation}
which is the lifted analogue of the classical \ac{RVU} inequality \eqref{eq:rvu-schematic}. Now, since the left-hand side is nonnegative, \eqref{eq:energy-overview} immediately implies
\begin{equation}\label{eq:PN-overview}
 \PV(\horizon)+\NV(\horizon)
 \le
 \frac{\nPlayers\rrange}{\learn}
 +
 \learn\EE(\horizon).
\end{equation}

It remains to prove the converse type of control, that is, the prediction error must itself be bounded by the Bregman movement. The bound proved in \Cref{prop:err} is
\begin{equation}\label{eq:E-overview}
 \sqrt{\EE(\horizon)}
 \le
 5\sqrt{\nPlayers}(\nPlayers+1)
 +
 2000(\nPlayers+1)^5\learn^{3/2}
 \sqrt{\PV(\horizon)+\NV(\horizon)}
 +
 \frac3{10}\sqrt{\EE(\horizon)}.
\end{equation}
Combining \eqref{eq:PN-overview} and \eqref{eq:E-overview} will give a bound on $\EE(\horizon)$ which is uniform in $\horizon$. Thus the main work is to prove \eqref{eq:E-overview}.

\subsection{The prediction error and Bregman movement bounds}

We first state the bounds used in the higher-order argument, without yet explaining why the lifted regularizer satisfies them, which we postpone to \cref{sec:proof-lifted-geometry}.

For a sequence $w=(w_t)$, write
\begin{equation}
 (\diff w)_t=w_t-w_{t-1}
\end{equation}
for its first time difference. Let
\begin{equation}
 \Jfun(\score)=D\choice(\score)
\end{equation}
be the Jacobian of the choice map, and define
\begin{equation}
 \jac_{i,t}
 =
 \learn\Jfun(\learn\score_{i,t-1}).
\end{equation}
To control $\diff\strat_{i,t}$, note that both $\strat_{i,t}$ and $\strat_{i,t-1}$ are perturbations of the same unoptimistic iterate $\choice(\learn\score_{i,t-1})$. Indeed,
\begin{equation}
 \strat_{i,t}
 =
 \choice\!\left(\learn(\score_{i,t-1}+\pred_{i,t})\right),
 \qquad
 \strat_{i,t-1}
 =
 \choice\!\left(\learn(\score_{i,t-1}-\perr_{i,t-1})\right).
\end{equation}
Writing $\jac_{i,t}=\learn\Jfun(\learn\score_{i,t-1})$, Taylor expansion around $\learn\score_{i,t-1}$ gives
\begin{equation}\label{eq:mir-overview}
 \diff\strat_{i,t}
 =
 \jac_{i,t}\diff\score_{i,t-1/2}+\remd_{i,t},
\end{equation}
where the remainder term is
\begin{equation}
 \remd_{i,t} =
 \bigl(\strat_{i,t}-\choice(\learn\score_{i,t-1})
       -\jac_{i,t}\pred_{i,t}\bigr) -
 \bigl(\strat_{i,t-1}-\choice(\learn\score_{i,t-1})
       +\jac_{i,t}\perr_{i,t-1}\bigr).
\end{equation}
Since
\begin{equation}\label{eq:score-difference-overview}
 \diff\score_{i,t-1/2}
 =
 \cpay_{i,t}-\diff\perr_{i,t},
\end{equation}
we obtain
\begin{equation}\label{eq:mirror-overview}
 \diff\strat_{i,t}
 =
 \remd_{i,t}
 +
 \jac_{i,t}\cpay_{i,t}
 -
 \jac_{i,t}\diff\perr_{i,t}.
\end{equation}

There are then three bounds in \Cref{lem:mir1} that we will use repeatedly.
First,
\begin{equation}\label{eq:Jbd-overview}
 \opnorm{\jac_{i,t}}{\infty}{1}
 \le
 \learn,
\end{equation}
so that the terms containing $\jac_{i,t}\diff\perr_{i,t}$ can be controlled by the prediction error energy $\EE(\horizon)$. Second,
\begin{equation}\label{eq:local-energy-overview}
 \sum_{i=1}^{\nPlayers}\sum_{t=1}^{\horizon}
 \norm*{\remd_{i,t}}_1^2
 \lesssim
 (\nPlayers+1)^2\learn^3
 \bigl(\PV(\horizon)+\NV(\horizon)\bigr),
\end{equation}
and
\begin{equation}\label{eq:Jenergy-overview}
 \sum_{i=1}^{\nPlayers}\sum_{t=1}^{\horizon}
 \opnorm{\jac_{i,t}-\jac_{i,t-1}}{\infty}{1}^2
 \lesssim
 \learn^3
 \bigl(\PV(\horizon)+\NV(\horizon)\bigr),
\end{equation}
so the Taylor remainders and Jacobian differences are controlled by the same Bregman movements appearing in \eqref{eq:PN-overview}.

The third bound is used when the same player appears twice in the expansion. If $|t-t'|\le\nPlayers+1$, \Cref{lem:mir1} gives
\begin{equation}\label{eq:repeat-overview}
 \opnorm{\jac_{i,t'}}{\infty}{1}^2
 \norm*{\jac_{i,t}\diff\score_{i,t-1/2}}_1^2
 \lesssim
 \learn^3
 \bigl(\PV_{i,t}+\NV_{i,t-1}\bigr).
\end{equation}
We refer to \eqref{eq:repeat-overview} as the \emph{repeated player bound}. The higher-order argument below will use only \eqref{eq:Jbd-overview}--\eqref{eq:repeat-overview}; afterward we explain how \eqref{eq:lreg} is designed so that they remain valid despite the lifting.

\subsection{Higher-order optimism with discounting}

We next explain the prediction error expansion, where the starting point is multilinearity of the game, namely, that every payoff difference reveals strategy differences. Indeed, \Cref{lem:paydiff} shows that, up to a finite initialization term, $\diff\cpay_i$ is a sum of linear maps applied to strategy variations $\diff\strat_j$, verifying the bounds in \eqref{eq:Ccontr}. Substituting \eqref{eq:mirror-overview}, each revealed strategy variation has the form
\begin{equation}\label{eq:basic-expansion-overview}
 \diff\strat_j
 =
 \underbrace{\remd_j}_{\text{controlled by \eqref{eq:local-energy-overview}}}
 +
 \underbrace{\jac_j\cpay_j}_{\text{may continue}}
 -
 \underbrace{\jac_j\diff\perr_j}_{\text{controlled using \eqref{eq:Jbd-overview}}},
\end{equation}
so only the middle term may require another time difference.
Indeed, applying the discrete product rule gives, suppressing time offsets,
\begin{equation}\label{eq:product-overview}
 \diff(\jac_j\cpay_j)
 =
 (\diff\jac_j)\cpay_j
 +
 \jac_j\diff\cpay_j.
\end{equation}
The first term is controlled by \eqref{eq:Jenergy-overview} and $\norm*{\cpay_j}_\infty\le2$ from \eqref{eq:gprop}. The second again reveals strategy variations through \Cref{lem:paydiff}, and \eqref{eq:basic-expansion-overview} introduces another Jacobian. If the newly revealed variation belongs to a player whose Jacobian is already present in the term, \eqref{eq:repeat-overview} controls it. Thus a term can remain uncontrolled only if it introduces a player label that has not previously appeared in that term, and we will call such a label \emph{fresh}. This is in particular where higher-order differences enter the argument, in that we build enough time differences directly into the prediction error. Indeed, since every continuing step introduces a fresh label and there are only $\nPlayers$ players, so $\nPlayers+1$ differences are sufficient to force every expansion branch to terminate.

Nonetheless, using the ordinary higher-order difference $\diff^{\nPlayers+1}$ would introduce the bound
\begin{equation}
 \norm*{\diff^{\nPlayers+1}w}_{\ell_2}
 \le
 2^{\nPlayers+1}\norm*{w}_{\ell_2},
\end{equation}
where, for a vector-valued sequence, we recall that $\norm*{w}_{\ell_2}=(\sum_t\norm*{w_t}_r^2)^{1/2}$, with $r=1$ or $\infty$ depending on context. This bound is notably exponential in the number of players, and the discounting in \Cref{alg:hood} is chosen to avoid this loss. Indeed, for a sequence $w=(w_t)_{t\in\Z}$, define the discounted sum
\begin{equation}
 (\dsum w)_t
 =
 \sum_{k=0}^{\infty}\disc^k w_{t-k},
 \qquad
 \disc=\frac{\nPlayers}{\nPlayers+1},
\end{equation}
where the sequences arising below have their negative index terms set to $0$ as specified in \Cref{alg:hood}. The predictor identity proved in \Cref{app:diff} states that
\begin{equation}\label{eq:derr-overview}
 \perr_i
 =
 \dsum^{\nPlayers+1}\diff^{\nPlayers+1}\cpay_i.
\end{equation}
Hence the prediction error contains the required $\nPlayers+1$ ordinary differences, paired with geometric discounting.

To reveal these differences one at a time, define
\begin{equation}\label{eq:Kdef-overview}
 \dop_p
 =
 \dsum^{\nPlayers+1}\diff^{\nPlayers+1-p},
 \qquad
 0\le p\le\nPlayers+1,
\end{equation}
so that
\begin{equation}\label{eq:Krec-overview}
 \dop_p=\dop_{p+1}\diff,
 \qquad
 0\le p\le\nPlayers.
\end{equation}
The sequence bounds of \Cref{lem:seq} give
\begin{equation}\label{eq:seq-overview}
 \norm*{\dop_pw}_{\ell_2}
 \le
 4(\nPlayers+1)^{p+1}\norm*{w}_{\ell_2},
\end{equation}
where the estimate underlying this bound is
\begin{equation}\label{eq:Dpow-overview}
 \norm*{(\dsum\diff)^q w}_{\ell_2}
 \le
 2\norm*{w}_{\ell_2},
 \qquad
 0\le q\le\nPlayers+1,
\end{equation}
for scalar sequences $w$. Indeed, for vector-valued sequences equipped pointwise with the $\ell_1$ or $\ell_\infty$ norm, the scalar bound proved in \eqref{eq:Dql1} gives
\begin{equation}\label{eq:Dpow-vector-overview}
 \norm*{(\dsum\diff)^q w}_{\ell_2}
 \le
 4(\nPlayers+1)\norm*{w}_{\ell_2},
 \qquad
 0\le q\le\nPlayers+1,
\end{equation}
and combining this with the bound for $\dsum^p$ yields \eqref{eq:seq-overview}, so the discounting prevents the successive differences from producing exponential dependance on the number of players.

\subsection{One-step expansion}

We may now formalize the recursive expansion. At order $p$, call a term \emph{continuing} if it contains $p$ Jacobians with distinct player labels and has not yet been controlled by \eqref{eq:Jbd-overview}--\eqref{eq:repeat-overview}. The precise  conditions on its payoff and strategy factors are those in \Cref{lem:step}---for which give an abbreviated version below---and they are importantly preserved whenever a branch continues.

\begin{lemma}[One-step expansion; abbreviated]\label{lem:step-overview}
Let $M$ be a continuing term at order $1\le p\le\nPlayers$. Then
\begin{equation}\label{eq:step-overview}
 \diff M
 =
 R_M+\sum_{M'\in\cont(M)}M',
\end{equation}
where $R_M$ collects terms controlled by  $\PV(\horizon)+\NV(\horizon)$, or $\EE(\horizon)$, with the bound given in \eqref{eq:term}, the $M'\in\cont(M)$ terms are continuing terms at order $p+1$, and
\begin{equation}\label{eq:cont-overview}
 |\cont(M)|
 \le
 3p(\nPlayers-p).
\end{equation}
In particular, $\cont(M)=\varnothing$ when $p=\nPlayers$.
\end{lemma}

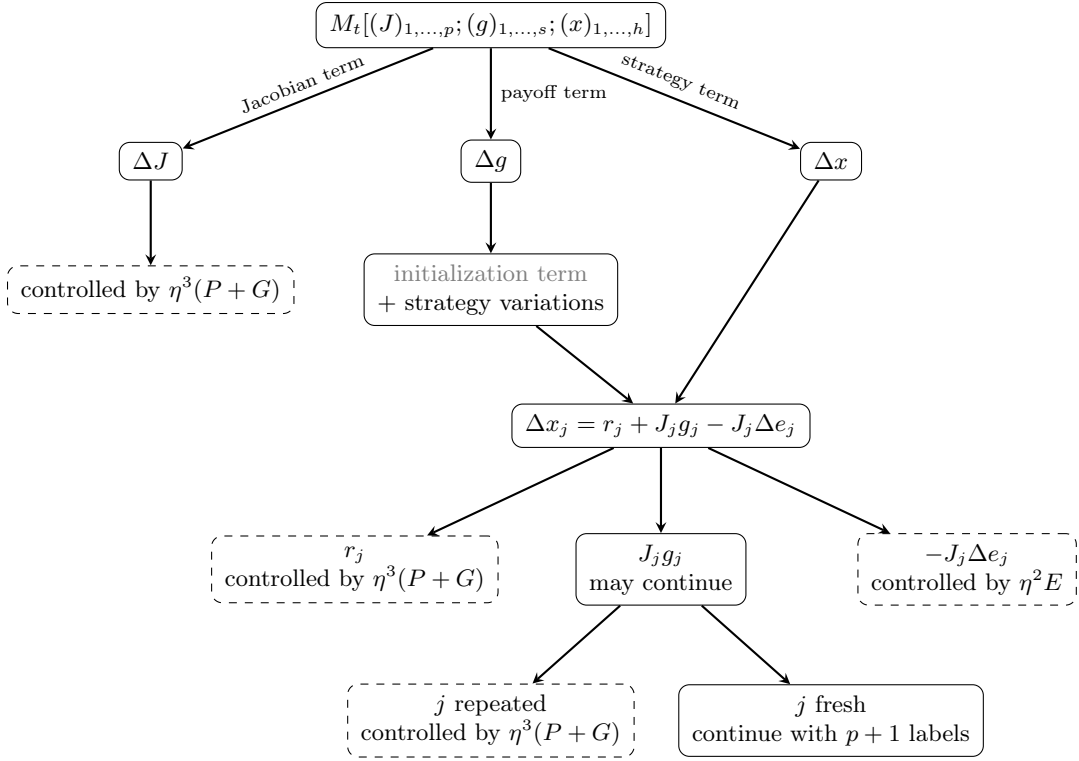
\begin{figure}[t]
\centering
\begin{tikzpicture}[
  >=stealth,
  every node/.style={align=center,font=\small},
  term/.style={
    draw,
    rounded corners,
    inner xsep=5pt,
    inner ysep=4pt
  },
  stop/.style={
    draw,
    dashed,
    rounded corners,
    inner xsep=5pt,
    inner ysep=4pt
  },
  flow/.style={->,thick}
]
\node[term] (root) at (0,0)
{$M_t[(\jac)_{1,\ldots,p};(\cpay)_{1,\ldots,s};(\strat)_{1,\ldots,h}]$};

\node[term] (dJ) at (-4.5,-1.8)
{$\diff\jac$};
\node[term] (dg) at (0,-1.8)
{$\diff\cpay$};
\node[term] (dx) at (4.5,-1.8)
{$\diff\strat$};

\draw[flow] (root) -- node[above,sloped]{\scriptsize Jacobian term} (dJ);
\draw[flow] (root) -- node[right]{\scriptsize payoff term} (dg);
\draw[flow] (root) -- node[above,sloped]{\scriptsize strategy term} (dx);

\node[stop] (Jstop) at (-4.5,-3.5)
{controlled by $\learn^3(\PV+\NV)$};
\draw[flow] (dJ) -- (Jstop);

\node[term] (gdec) at (0,-3.5)
{{\color{gray}initialization term}\\
$+$ strategy variations};
\draw[flow] (dg) -- (gdec);

\node[term] (xdec) at (2.25,-5.3)
{$\diff\strat_j
 =\remd_j+\jac_j\cpay_j-\jac_j\diff\perr_j$};
\draw[flow] (gdec) -- (xdec);
\draw[flow] (dx) -- (xdec);

\node[stop] (rstop) at (-1.8,-7.2)
{$\remd_j$\\
controlled by $\learn^3(\PV+\NV)$};
\node[term] (jg) at (2.25,-7.2)
{$\jac_j\cpay_j$\\
may continue};
\node[stop] (estop) at (6.3,-7.2)
{$-\jac_j\diff\perr_j$\\
controlled by $\learn^2\EE$};

\draw[flow] (xdec) -- (rstop);
\draw[flow] (xdec) -- (jg);
\draw[flow] (xdec) -- (estop);

\node[stop] (rep) at (0,-9.2)
{$j$ repeated\\
controlled by $\learn^3(\PV+\NV)$};
\node[term] (fresh) at (4.5,-9.2)
{$j$ fresh\\
continue with $p+1$ labels};

\draw[flow] (jg) -- (rep);
\draw[flow] (jg) -- (fresh);
\end{tikzpicture}
\caption{One-step expansion of a continuing term.}
\label{fig:diagram}
\end{figure}

We sketch the two facts in \Cref{lem:step-overview} that drive the argument. By \eqref{eq:Krec-overview}, moving from order $p$ to order $p+1$ amounts to applying one ordinary difference. If this difference falls on a Jacobian, \eqref{eq:Jenergy-overview} controls the resulting term; if it falls on a payoff, \Cref{lem:paydiff} produces strategy variations, up to the finite initialization term; if it falls on a strategy factor, such a variation appears directly. In either case, \eqref{eq:basic-expansion-overview} shows that only $\jac_j\cpay_j$ may continue. If $j$ is already one of the existing Jacobian labels, \eqref{eq:repeat-overview} controls the resulting term, so every continuing branch introduces one fresh label. It remains then to count how many such branches are possible. Suppose the term contains $s$ payoff factors, then, by \Cref{lem:paydiff}, each payoff factor can reveal a fixed player label at most twice, while the  conditions in \Cref{lem:step} allow at most $p$ strategy factors carrying any fixed label, so a fixed fresh label can produce at most $2s+p\le3p$ continuing terms, and since $\nPlayers-p$ labels remain fresh, this gives \eqref{eq:cont-overview}. In particular, after all $\nPlayers$ labels have appeared, no branch can continue.

Finally, the branching must be compensated by the small factor gained by each continuation. By \eqref{eq:seq-overview}, moving one level further costs one additional factor $\bigoh(\nPlayers+1)$, while every continuing branch gains a new Jacobian of norm at most $2\learn$ by \Cref{lem:mir1}. Hence the total weighted continuation factor is at most
\begin{equation}\label{eq:contraction-overview}
 3p(\nPlayers-p)\,2\learn(\nPlayers+1)
 \le
 \frac32\nPlayers^2(\nPlayers+1)\learn,
\end{equation}
and this is in particular why the step size must be of order $\nPlayers^{-3}$. Under \eqref{eq:eta}, the continuing contribution contracts geometrically from one level to the next, while \Cref{lem:step-overview} guarantees that it terminates altogether after $\nPlayers$ levels.

Iterating \Cref{lem:step}, together with \Cref{lem:seq,lem:paydiff,lem:mir1}, gives \Cref{prop:err}, i.e., \eqref{eq:E-overview}. The three terms in that bound correspond respectively to the finite initialization cost, the terms stopped using Bregman movement, and the $-\jac_j\diff\perr_j$ terms in \eqref{eq:basic-expansion-overview}.

\subsection{Why the lifted regularizer preserves the expansion bounds}
\label{sec:proof-lifted-geometry}

We now return to \eqref{eq:Jbd-overview}--\eqref{eq:repeat-overview} and explain why they hold for the lifted choice map. Recall from \eqref{eq:lreg} that, for $z=\mass\strat\ne0$,
\begin{equation}\label{eq:lreg-overview}
 \lreg(\mass\strat)
 =
 (\mass+\sqrt{\mass})\hreg(\strat)
 -3\Lreg\sqrt{\mass(1-\mass)},
 \qquad
 \Lreg
 =
 4+\log\nPures+\frac14\log^2\nPures.
\end{equation}
The two parts of \eqref{eq:lreg-overview} are chosen to preserve the local entropic bounds needed by the expansion while controlling the additional mass variable. Indeed, to see what must be preserved, first consider ordinary entropy on the original simplex. In that case the analogues of the strategy movement and Jacobian bounds are
\begin{equation}\label{eq:ordinary-local-overview}
 \norm*{\diff\strat_{i,t}}_1^2
 \lesssim
 \learn(\PV_{i,t}+\NV_{i,t-1}),
 \qquad
 \opnorm{\jac_{i,t}}{\infty}{1}^2
 \lesssim
 \learn^2.
\end{equation}
Thus an existing Jacobian paired with a newly revealed variation of the same player contributes
\begin{equation}
 \learn^2\cdot\learn(\PV_{i,t}+\NV_{i,t-1})
 =
 \learn^3(\PV_{i,t}+\NV_{i,t-1}),
\end{equation}
which is the bound used in \eqref{eq:repeat-overview}. The lifted regularizer is thus chosen so that this cancellation survives when the mass is allowed to vary.

Indeed, consider first a naive lift
\begin{equation}\label{eq:naive-lift-overview}
 \lreg_0(\mass\strat)
 =
 \mass\hreg(\strat)+\varphi(\mass),
\end{equation}
where $\varphi$ depends only on $\mass$. Conditional on the optimizing mass, the factor $\mass$ cancels from the strategy dependent part of the optimization, so its normalized choice map is still $\choice_0(\score)=\softmax(\score)$, and its Jacobian therefore remains of constant order when $\mass$ becomes small. On the other hand, the strategy-dependent part $\mass\hreg(\strat)$ gives only
\begin{equation}\label{eq:naive-movement-overview}
 \norm*{\strat'-\strat}_1^2
 \lesssim
 \frac{1}{\mass}
 \bregofX{\lreg_0}{\mass'\strat'}{\mass\strat}.
\end{equation}
The analogue of the repeated player bound would therefore contain $\learn^3/\mass$ and become uncontrolled as $\mass\to0$.

As for the term $\mass\hreg(\strat)$ in \eqref{eq:lreg-overview}, it retains this Bregman control, where \Cref{lem:coerc} gives
\begin{equation}\label{eq:strategy-bregman-overview}
 \norm*{\strat'-\strat}_1^2
 \lesssim
 \frac{1}{\mass}
 \bregofX{\lreg}{\mass'\strat'}{\mass\strat}
\end{equation}
for nearby scores. The additional term $\sqrt{\mass}\hreg(\strat)$ then contributes by suppling the compensating factor in the Jacobian. Indeed, \Cref{lem:maps} gives the choice map formula already stated in \eqref{eq:chomap},
\begin{equation}\label{eq:choice-map-overview}
 \choice(\score)
 =
 \softmax\!\left(
 \frac{\sqrt{\mass}}{1+\sqrt{\mass}}\score
 \right).
\end{equation}
If the optimizing mass were fixed, \eqref{eq:choice-map-overview} would immediately make the Jacobian smaller by a factor of order $\sqrt{\mass}$. Nonetheless, the optimizing mass itself depends on the score, so we must also control this dependence. This is the role of the single mass-dependent term $-3\Lreg\sqrt{\mass(1-\mass)}$ and of the particular choice of $\Lreg$. Indeed, \Cref{lem:logmom} gives
\begin{equation}\label{eq:logmoment-overview}
 \Var_{\strat}(\log\strat)
 \le
 \frac14\log^2\nPures+\log\nPures+2
 =
 \Lreg-2.
\end{equation}
Together with the mass-dependent term, this controls the interaction between mass and strategy variation in the Hessian, where \Cref{lem:frac} and \eqref{eq:lhess} then show in particular that the lifted regularizer is $1$-strongly convex in $\ell_1$, so that \Cref{lem:maps} yields a unique smooth lifted mirror map and choice map. More importantly for the expansion, the same term also gives the lower bound \eqref{eq:Dlow} for the denominator arising when the optimality condition for the mass is differentiated. Equations \eqref{eq:dmass} and \eqref{eq:logmass} then give
\begin{equation}\label{eq:mass-derivative-overview}
 |D\mass(\score)[q]|
 \lesssim
 \mass^{3/2}\norm q_\infty,
 \qquad
 |D\log\mass(\score)[q]|
 \le
 \frac14\norm q_\infty.
\end{equation}
The first bound ensures that the dependence of the optimizing mass on the score does not destroy the $\sqrt{\mass}$ factor suggested by \eqref{eq:choice-map-overview}, where \Cref{lem:mdiff} gives indeed
\begin{equation}\label{eq:Jmass-overview}
 \opnorm{\Jfun(\score)}{\infty}{1}
 \le
 2\sqrt{\mass}.
\end{equation}
The second bound integrates to the mass comparison \eqref{eq:masscmp}, which allows the masses corresponding to the nearby times in the $\nPlayers+1$-step expansion to be compared. Consequently, along the learning dynamics, \Cref{lem:coerc,lem:mdiff} give, for the nearby times appearing in the expansion,
\begin{equation}\label{eq:mass-cancel-overview}
 \norm*{\diff\strat_{i,t}}_1^2
 \lesssim
 \frac{\learn}{\mass_{i,t-1}}
 \bigl(\PV_{i,t}+\NV_{i,t-1}\bigr),
 \qquad
 \opnorm{\jac_{i,t'}}{\infty}{1}^2
 \lesssim
 \learn^2\mass_{i,t-1}.
\end{equation}
The mass factors therefore cancel and give a bound as in \eqref{eq:ordinary-local-overview}:
\begin{equation}
 \bigl(\learn^2\mass_{i,t-1}\bigr)
 \left(
 \frac{\learn}{\mass_{i,t-1}}
 \right)
 =
 \learn^3,
\end{equation}
which gives the repeated player bound \eqref{eq:repeat-overview} through \Cref{lem:mir1}.

Finally, the lifted Bregman control also controls changes in the choice map. Indeed, \Cref{lem:Jstab} allows us to bound Jacobian variation by lifted Bregman divergence and \Cref{lem:taylor} converts this into a Taylor remainder bound. Applied to the iterates, these bounds give \eqref{eq:local-energy-overview} and \eqref{eq:Jenergy-overview} through \Cref{lem:mir1}. Hence all the local bound assumed in the higher-order expansion are preserved after lifting, which was our goal. Note then that the dependence on the number of actions also enters through the regularizer, where, by \Cref{lem:osc},
\begin{equation}\label{eq:range-overview}
 \rrange
 \le
 \frac{13}{6}\Lreg
 =
 \bigoh(\log^2\nPures).
\end{equation}

\subsection{Closing the loop}

We can now combine the two parts of the proof. Substituting \eqref{eq:PN-overview} into \eqref{eq:E-overview} and using $\sqrt{a+b}\le\sqrt a+\sqrt b$ gives
\begin{align}
 \sqrt{\EE(\horizon)}
 &\le
 5\sqrt{\nPlayers}(\nPlayers+1)
 +
 2000(\nPlayers+1)^5\learn
 \sqrt{\nPlayers\rrange}
 \\
 &\quad+
 \left(
 2000(\nPlayers+1)^5\learn^2+\frac3{10}
 \right)
 \sqrt{\EE(\horizon)}.
 \label{eq:close-overview}
\end{align}
The step-size condition \eqref{eq:eta} makes the coefficient multiplying $\sqrt{\EE(\horizon)}$ on the right-hand side small enough to absorb into the left-hand side. Together with \eqref{eq:range-overview}, the calculation in \Cref{app:proof} gives
\begin{equation}\label{eq:Efinal-overview}
 \EE(\horizon)
 <
 62^2(\nPlayers+1)^6\Lreg,
\end{equation}
uniformly in $\horizon$. Finally, the player-specific version of the lifted \acs{RVU} inequality gives
\begin{equation}
 [\reg_i(\horizon)]_+
 \le
 \frac{\rrange}{\learn}
 +
 \frac{\learn}{2}\EE(\horizon),
\end{equation}
and, using \Cref{lem:osc} and \eqref{eq:Efinal-overview},
\begin{equation}
 [\reg_i(\horizon)]_+
 <
 \frac{3\Lreg}{\learn},
\end{equation}
hence, since $\reg_i(\horizon)\le[\reg_i(\horizon)]_+$, this proves \Cref{thm:main}. In particular, for the largest admissible step size $\learn=1/[60(\nPlayers+1)^3]$, this gives
\begin{equation}
 \sup_{\horizon\ge1}\reg_i(\horizon)
 =
 \bigoh(\nPlayers^3\log^2\nPures).
\end{equation}

\section{Concluding remarks}\label{sec:concl}

We have shown that uncoupled full information learning can achieve regret that is uniformly bounded in time in every finite game. Compared with the previous $\bigoh(\nPlayers\log^2 \nPures\,\log \horizon)$ per-player bound \citep{SPF25a}, our result removes the remaining dependence on the horizon, at the cost of an additional $\nPlayers^2$ factor in the number of players. Natural future research directions are whether this $\nPlayers^2$ price can be reduced, whether the $\log^2 \nPures$ dependence can be improved, and whether bounded regret is possible in broader game classes, and for stronger notions of regret, such as swap regret. Another important question is whether we can obtain bounded regret from simpler dynamics---such as the classical \acs{OFTRL} update, for a suitable choice of regularizer and step size---and whether these dynamics retain good regret bounds in the game's parameters.

\section*{Acknowledgments}
\begingroup
\small
Some design choices in the algorithm and its analysis were assisted by ChatGPT 5.6 Sol.
\endgroup

\appendix \crefalias{section}{appendix}
\crefalias{subsection}{appendix}
\numberwithin{theorem}{section}
\numberwithin{lemma}{section} \numberwithin{corollary}{section} \numberwithin{proposition}{section} \numberwithin{definition}{section} 
\numberwithin{equation}{section}
\renewcommand{\theHtheorem}{\thesection.\arabic{theorem}} \renewcommand{\theHlemma}{\thesection.\arabic{lemma}} \renewcommand{\theHcorollary}{\thesection.\arabic{corollary}} \renewcommand{\theHproposition}{\thesection.\arabic{proposition}} \renewcommand{\theHdefinition}{\thesection.\arabic{definition}}

\section{Lifted regularizer and choice map}\label{app:lift}

This section proves properties of the lifted regularizer  used in \cref{sec:proof}.

\subsection{Hessian bounds}\label{app:geom}

For $\strat\in\istrats$ and values $a_\pure$ defined on the support of $\strat$, write
\begin{equation}
 \Var_{\strat}(a)=\sum_{\pure:\strat_\pure>0}\strat_\pure
 \left(a_\pure-\sum_{\purealt:\strat_\purealt>0}\strat_\purealt a_\purealt\right)^2.
\end{equation}
For probability vectors $p,q\in\R_+^\nPures$, we use the following KL bounds \citep{CT06,SV16}:
\begin{equation}\label{eq:info}
 \dkl(p\|q)\ge\frac12\norm*{p-q}_1^2,
 \qquad
 \dkl(p\|q)\ge\norm*{\sqrt p-\sqrt q}_2^2.
\end{equation}
For $\strat\in\istrats$, Gibbs' inequality gives \citep{CT06}
\begin{equation}\label{eq:hrange}
 -\log \nPures\le \hreg(\strat)\le0,
 \qquad
 \log \nPures+\hreg(\strat)=\dkl(\strat\|\nPures^{-1}\one)\ge0.
\end{equation}
Set $\istrats^\circ=\istrats\cap\R_{++}^\nPures$.
If $z=\mass \strat\in\Kint$, every $\xi\in\R^\nPures$ has the unique decomposition
\begin{equation}\label{eq:decomp}
 \xi=b\strat+\mass w,
 \qquad
 b=\braket*{\one}{\xi},
 \qquad
 \braket*{\one}{w}=0.
\end{equation}
\begin{lemma}[Logarithmic moment]\label{lem:logmom}
For every $\strat\in\istrats$,
\begin{equation}\label{eq:logvar}
 \Var_{\strat}(\log \strat)
 \le \frac14\log^2\nPures+\log \nPures+2=\Lreg-2.
\end{equation}
\end{lemma}

\begin{proof}
We first establish the auxiliary moment bounds, valid for every $p\in\R_+^\nPures$ with $\sum_\pure p_\pure\le1$:
\begin{align}
 \sum_{\pure:p_\pure>0}p_\pure|\log p_\pure|^2
 &\le4\Lreg,                                                     \label{eq:log2}\\
 \sum_{\pure:p_\pure>0}p_\pure|\log p_\pure|^4
 &\le26\Lreg^2.                                                  \label{eq:log4}
\end{align}
For $t\ge0$,
\begin{equation}\label{eq:logtail}
 \sum_{\pure:-\log p_\pure\ge t}p_\pure\le\min\{1,de^{-t}\},
\end{equation}
where the sum is over $\pure$ with $p_\pure>0$.
Hence, for every integer $q\ge1$, using $a_+^q=q\int_0^\infty u^{q-1}\mathbf{1}\{a\ge u\}\,du$ and interchanging the finite sum with the integral gives
\begin{align}
 \sum_{\pure:p_\pure>0}p_\pure(-\log p_\pure-\log \nPures)_+^q
 &=q\int_0^\infty u^{q-1}
   \sum_{\pure:-\log p_\pure\ge \log \nPures+u}p_\pure\,du\\
 &\le q\int_0^\infty u^{q-1}e^{-u}\,du=q!.
\end{align}
Minkowski's inequality therefore yields
\begin{equation}\label{eq:logmom}
 \left(\sum_{\pure:p_\pure>0}p_\pure|\log p_\pure|^q\right)^{1/q}
 \le \log \nPures+(q!)^{1/q}.
\end{equation}
Since $2\sqrt \Lreg\ge \log \nPures+2$, \eqref{eq:logmom} with $q=2$ gives \eqref{eq:log2}. For $q=4$, use $24^{1/4}<5/2$ and $\sqrt \Lreg\ge2$:
\begin{equation}
 \log \nPures+24^{1/4}<\log \nPures+\frac52
 \le2\sqrt \Lreg+\frac12
 \le\frac94\sqrt \Lreg.
\end{equation}
Hence
\begin{equation}
 \sum_\pure p_\pure|\log p_\pure|^4
 \le\left(\frac94\right)^4\Lreg^2<26\Lreg^2,
\end{equation}
which is \eqref{eq:log4}. Now let $p=\strat\in\istrats$.
Since variance is bounded by the second moment around any deterministic center,
\begin{equation}
 \Var_{\strat}(\log \strat)
 \le\sum_{\pure:\strat_\pure>0}\strat_\pure
 \left(-\log \strat_\pure-\frac12\log \nPures\right)^2.
\end{equation}
If $-\log \strat_\pure\le\log \nPures$, the squared term is at most $(\log \nPures)^2/4$; otherwise,
\begin{equation}
 \left(-\log \strat_\pure-\frac12\log \nPures\right)^2
 =\frac{(\log \nPures)^2}{4}
 +\int_{\log \nPures}^{-\log \strat_\pure}2\left(t-\frac12\log \nPures\right)\,dt.
\end{equation}
Using \eqref{eq:logtail} and writing the square as an integral,
\begin{align}
 \sum_{\pure:\strat_\pure>0}\strat_\pure
 \left(-\log \strat_\pure-\frac12\log \nPures\right)^2
 &\le\frac{(\log \nPures)^2}{4}
   +\int_{\log \nPures}^{\infty}2\left(t-\frac12\log \nPures\right)de^{-t}\,dt\\
 &=\frac14\log^2\nPures+\log \nPures+2=\Lreg-2.
\end{align}
This proves \eqref{eq:logvar}.
\end{proof}

For nonzero $z=\mass \strat\in\Kset$, define
\begin{equation}
 \areg(z)=\sqrt\mass\,\hreg(\strat)-3\Lreg\sqrt{\mass(1-\mass)},
 \qquad \areg(0)=0.
\end{equation}

\begin{lemma}[Hessian of the square-root terms]\label{lem:frac}
Let $z=\mass \strat\in\Kint$ and let $\xi=b\strat+\mass w$ satisfy \eqref{eq:decomp}. Then
\begin{equation}\label{eq:fconv}
 D^2\areg(z)[\xi,\xi]\ge\frac{\Lreg b^2}{2\mass^{3/2}}\ge0.
\end{equation}
Moreover,
\begin{equation}\label{eq:fstrong}
 D^2\areg(z)[\xi,\xi]
 \ge
 \frac12\sqrt\mass\sum_\pure\frac{w_\pure^2}{\strat_\pure}
 +3b^2.
\end{equation}
\end{lemma}

\begin{proof}
For $z(q)=z+q\xi=\mass(q)\strat(q)$,
\begin{equation}
 \mass(q)=\mass+qb,
 \qquad
 \strat(q)=\strat+\frac{q\mass}{\mass+qb}w,
\end{equation}
so $\strat'(0)=w$ and $\strat''(0)=-2bw/\mass$. Since $\braket*{\one}{w}=0$,
\begin{equation}
 D\hreg(\strat)[w]=\braket*{w}{\log \strat},
 \qquad
 D^2\hreg(\strat)[w,w]=\sum_\pure\frac{w_\pure^2}{\strat_\pure}.
\end{equation}
Differentiating $\areg(z(q))$ twice at $q=0$ gives
\begin{equation}\label{eq:fhess}
 D^2\areg(z)[\xi,\xi]
 =\mass^{-3/2}\left[
 \mass^2\sum_\pure\frac{w_\pure^2}{\strat_\pure}
 -\mass b\braket*{w}{\log \strat}
 +\frac{b^2}{4}\left(
 \frac{3\Lreg}{(1-\mass)^{3/2}}-\hreg(\strat)
 \right)\right].
\end{equation}
Moreover,
\begin{equation}
 |\braket*{w}{\log \strat}|
 =|\braket*{w}{\log \strat-\hreg(\strat)\one}|
 \le\sqrt{\Var_{\strat}(\log \strat)}
 \left(\sum_\pure\frac{w_\pure^2}{\strat_\pure}\right)^{1/2}.
\end{equation}
Set $S=\sum_\pure w_\pure^2/\strat_\pure$ and $V=\Var_{\strat}(\log\strat)$. Then
\begin{equation}
 \mass^2S-\mass b\braket*{w}{\log\strat}
 \ge-\frac{Vb^2}{4}.
\end{equation}
Using \eqref{eq:hrange}, \eqref{eq:logvar}, and $(1-\mass)^{-3/2}\ge1$, the bracket in \eqref{eq:fhess} is therefore at least
\begin{equation}
 \frac{b^2}{4}(3\Lreg-V-\hreg(\strat))
 \ge\frac{2\Lreg+2}{4}b^2
 \ge\frac{\Lreg}2b^2,
\end{equation}
which proves \eqref{eq:fconv}. For the stronger bound, use instead
\begin{equation}
 \mass^2S-\mass b\braket*{w}{\log\strat}
 \ge\frac12\mass^2S-\frac12Vb^2.
\end{equation}
Writing $s=\sqrt\mass$, \eqref{eq:fhess}, \eqref{eq:logvar}, and $\Lreg\ge4$ give
\begin{align}
 D^2\areg(z)[\xi,\xi]
 &\ge
 \frac12\sqrt\mass S
 +\mass^{-3/2}\frac{b^2}{4}
 \left(\frac{3\Lreg}{(1-\mass)^{3/2}}-2V-\hreg(\strat)\right)\\
 &\ge
 \frac12\sqrt\mass S
 +s^{-3}b^2\left(\frac{3}{(1-s^2)^{3/2}}-1\right).
\end{align}
The function $3(1-s^2)^{-3/2}-1-3s^3$ is increasing on $[0,1)$, since its derivative is $9s[(1-s^2)^{-5/2}-s]\ge0$, and its value at $s=0$ is $2$.
Thus the last coefficient is at least $3$, proving \eqref{eq:fstrong}.
\end{proof}

Since $\areg$ is continuous on $\Kset$ and has positive semidefinite Hessian on $\Kint$ by \eqref{eq:fconv}, it is convex on $\Kset$. For $z=\mass \strat\in\Kint$, write
\begin{equation}\label{eq:lregdec}
 \lreg(z)=\mass \hreg(\strat)+\areg(z).
\end{equation}
If $\xi=b\strat+\mass w$ is as in \eqref{eq:decomp}, then
\begin{equation}
 D^2[\mass \hreg(\strat)][\xi,\xi]
 =\mass\sum_\pure\frac{w_\pure^2}{\strat_\pure}.
\end{equation}
Together with \eqref{eq:fstrong},
\begin{equation}\label{eq:lhess}
 D^2\lreg(z)[\xi,\xi]
 \ge
 \left(\mass+\frac12\sqrt\mass\right)
 \sum_\pure\frac{w_\pure^2}{\strat_\pure}
 +3b^2
 \ge\norm*{\xi}_1^2.
\end{equation}
Indeed, Cauchy--Schwarz gives $\norm*{w}_1^2\le\sum_\pure w_\pure^2/\strat_\pure$. If $a=|b|$ and $c=\mass\norm*{w}_1$, then $\mass\le1$ implies that the middle expression in \eqref{eq:lhess} is at least $3a^2+\frac32c^2\ge(a+c)^2$, while $\norm*{\xi}_1\le a+c$. Integrating \eqref{eq:lhess} along line segments in $\Kint$ and approximating boundary points shows that $\lreg$ is $1$-strongly convex with respect to the $\ell_1$ norm on $\Kset$. In particular, $D^2\lreg(z)$ is positive definite on $\Kint$.

\subsection{Lifted mirror map and choice map}\label{app:choice}

\begin{lemma}[Lifted mirror and choice map]\label{lem:maps}
For every $\score\in\R^\nPures$, $\lmir(\score)$ is unique and belongs to $\Kint$.
If $\lmir(\score)=\mass \strat$, then
\begin{equation}\label{eq:softmax}
 \strat=\choice(\score)=\softmax\!\left(\frac{\sqrt\mass}{1+\sqrt\mass}\score\right).
\end{equation}
The maps $\lmir$ and $\choice$ are $C^\infty$.
\end{lemma}

\begin{proof}
Continuity follows from \eqref{eq:lreg}, \eqref{eq:hrange}, and $r\log r\to0$ as $r\downarrow0$. By \eqref{eq:lhess}, $\lreg$ is strictly convex on $\Kint$. Approximating any two points of $\Kset$ by points of $\Kint$ and passing to the limit shows that $\lreg$ is convex on $\Kset$. Now, fix $\score$. The objective in \eqref{eq:lmir} is continuous on the compact set $\Kset$, hence attains its maximum. For fixed $0<\mass\le1$, its $\strat$-dependent part is
\begin{equation}
 \mass\braket*{\score}{\strat}-(\mass+\sqrt\mass)\hreg(\strat).
\end{equation}
The unique maximizer is \eqref{eq:softmax}. The maximized value at mass $\mass\in(0,1)$ is
\begin{equation}
 V_\score(\mass)
 =(\mass+\sqrt\mass)
 \log\!\left(\sum_{\pure=1}^\nPures
 \exp\!\left(\frac{\sqrt\mass}{1+\sqrt\mass}\score_\pure\right)\right)
 +3\Lreg\sqrt{\mass(1-\mass)}.
\end{equation}
This extends continuously to $[0,1]$ with $V_\score(0)=0$, and
\begin{equation}
 \lim_{\mass\downarrow0}\frac{V_\score(\mass)}{\sqrt\mass}
 =\log \nPures+3\Lreg>0.
\end{equation}
Thus $\mass=0$ is not maximizing. For $0<\mass<1$, the maximizing $\strat=\strat(\mass)$ is unique and smooth by \eqref{eq:softmax}. Differentiating the unmaximized objective at $(\mass,\strat(\mass))$, the derivative through $\strat(\mass)$ vanishes: its $\strat$-gradient is a scalar multiple of $\one$ at the maximizer, while $\braket*{\one}{\strat'(\mass)}=0$. Hence
\begin{equation}\label{eq:prof1}
 V_\score'(\mass)
 =\braket*{\score}{\strat}
 -\left(1+\frac{1}{2\sqrt\mass}\right)\hreg(\strat)
 +\frac{3\Lreg}2\frac{1-2\mass}{\sqrt{\mass(1-\mass)}}.
\end{equation}
The last term tends to $-\infty$ as $\mass\uparrow1$, while all other terms remain bounded. Hence $V_\score'<0$ on some interval $(1-\eps,1)$; integrating $V_\score'$ on this interval shows $V_\score(1)<V_\score(1-\eps)$, so $\mass=1$ is not maximizing. Every maximizer therefore lies in $\Kint$, where strict concavity of $z\mapsto\braket*{\score}{z}-\lreg(z)$ gives uniqueness. Finally, at the maximizer,
\begin{equation}\label{eq:foc}
 \nabla\lreg(\lmir(\score))=\score,
\end{equation}
and the Hessian of $\lreg$ is positive definite on $\Kint$ by \eqref{eq:lhess}, so the implicit function theorem, applied to \eqref{eq:foc}, gives $\lmir\in C^\infty(\R^\nPures)$, and normalization gives $\choice\in C^\infty(\R^\nPures)$.
\end{proof}

\subsection{Local stability of the choice map}\label{app:mirror}

For $\score\in\R^\nPures$, write
\begin{equation}
 \lmir(\score)=\mass \strat,
 \qquad
 \mass=\mass(\score)=\braket*{\one}{\lmir(\score)},
 \qquad
 \strat=\choice(\score).
\end{equation}
For $\strat\in\istrats^\circ$, set
\begin{equation}
 \mu_{\strat}=\strat\odot(\log \strat-\hreg(\strat)\one),
 \qquad
 \Gmat_{\strat}=\diag(\strat)-\strat\strat^\top.
\end{equation}
Then
\begin{equation}\label{eq:cov}
 \Gmat_{\strat}v=\strat\odot(v-\braket*{\strat}{v}\one),
 \qquad
 \Gmat_{\strat}\log \strat=\mu_{\strat},
 \qquad
 \norm*{\mu_{\strat}}_1\le\sqrt{\Var_{\strat}(\log \strat)}\le\sqrt{\Lreg-2}.
\end{equation}
For $\mass\in(0,1)$ and $\strat\in\istrats^\circ$, define
\begin{equation}
 \Dfun(\mass,\strat)
 =\frac{3\Lreg}{4(1-\mass)^{3/2}}
 -\frac14\left[
 \hreg(\strat)+\frac{\Var_{\strat}(\log \strat)}{1+\sqrt\mass}\right].
\end{equation}
By \eqref{eq:hrange} and \eqref{eq:logvar},
\begin{equation}\label{eq:Dlow}
 \Dfun(\mass,\strat)
 \ge \frac{\Lreg}{2(1-\mass)^{3/2}}+\frac12
 \ge\frac{\Lreg+1}{2}.
\end{equation}
Indeed, $\Var_{\strat}(\log\strat)\le\Lreg-2$ and $(1-\mass)^{-3/2}\ge1$.
For $\mass\in(0,1)$ and $\strat\in\istrats^\circ$, set
\begin{equation}\label{eq:nvec}
 \nu_{\mass,\strat}=\strat+\frac{\mu_{\strat}}{2(1+\sqrt\mass)}.
\end{equation}

\begin{lemma}[Differential of the choice map]\label{lem:mdiff}
Let $\score\in\R^\nPures$ and write $\lmir(\score)=\mass \strat$. Then, for every $q\in\R^\nPures$,
\begin{align}
 \norm*{D\choice(\score)[q]}_1
 &\le\norm q_\infty,                                      \label{eq:Qunif}\\
 \norm*{D\choice(\score)[q]}_1
 &\le2\sqrt\mass\norm q_\infty.                     \label{eq:Qmass}
\end{align}
\end{lemma}

\begin{proof}
The first-order condition \eqref{eq:prof1} reads
\begin{equation}
 \braket*{\score}{\strat}
 =\left(1+\frac{1}{2\sqrt\mass}\right)\hreg(\strat)
 -\frac{3\Lreg}2\frac{1-2\mass}{\sqrt{\mass(1-\mass)}}.
\end{equation}
From \eqref{eq:softmax},
\begin{equation}
 \partial_\mass \strat
 =\frac{1}{2\sqrt\mass(1+\sqrt\mass)^2}\Gmat_{\strat}\score
 =\frac{\mu_{\strat}}{2\mass(1+\sqrt\mass)},
\end{equation}
where the second identity uses $\Gmat_{\strat}\score=(1+\sqrt\mass)\mu_{\strat}/\sqrt\mass$. Hence
\begin{equation}
 \partial_\mass \hreg(\strat)=\frac{\Var_{\strat}(\log \strat)}{2\mass(1+\sqrt\mass)},
 \qquad
 \partial_\mass\braket*{\score}{\strat}=\frac{\Var_{\strat}(\log \strat)}{2\mass^{3/2}}.
\end{equation}
Differentiating \eqref{eq:prof1} and collecting terms gives
\begin{align}
 V_\score''(\mass)
 &=\mass^{-3/2}\left[
 \frac{\Var_{\strat}(\log \strat)}{4(1+\sqrt\mass)}
 +\frac{\hreg(\strat)}4
 -\frac{3\Lreg}{4(1-\mass)^{3/2}}\right]\\
 &=-\mass^{-3/2}\Dfun(\mass,\strat).                    \label{eq:pcurv}
\end{align}
At fixed $\mass$, \eqref{eq:softmax} gives $D_{\score}\strat[q]=\sqrt\mass \Gmat_{\strat}q/(1+\sqrt\mass)$. Since $\Gmat_{\strat}\score=(1+\sqrt\mass)\mu_{\strat}/\sqrt\mass$ and $\braket*{\one}{D_{\score}\strat[q]}=0$,
\begin{align}
 D_{\score}V_{\score}'(\mass)[q]
 &=\braket*{\strat}{q}+\braket*{\score}{D_{\score}\strat[q]}
 -\left(1+\frac1{2\sqrt\mass}\right)\braket*{\log \strat}{D_{\score}\strat[q]}\\
 &=\braket*{\strat}{q}+\frac{1}{2(1+\sqrt\mass)}\braket*{\mu_{\strat}}{q}
 =\braket*{\nu_{\mass,\strat}}{q}.
\end{align}
Differentiating $V_\score'(\mass(\score))=0$ and using \eqref{eq:pcurv} gives
\begin{equation}\label{eq:dmass}
 D\mass(\score)[q]
 =\frac{\mass^{3/2}}{\Dfun(\mass,\strat)}
   \braket*{\nu_{\mass,\strat}}{q}.
\end{equation}
Finally, differentiating \eqref{eq:softmax} in $\score$ and $\mass$ gives
\begin{equation}
 D\choice(\score)[q]
 =\frac{\sqrt\mass}{1+\sqrt\mass}\Gmat_{\strat}q
 +\frac{D\mass(\score)[q]}{2\mass(1+\sqrt\mass)}\mu_{\strat}.
\end{equation}
Substituting \eqref{eq:dmass} yields
\begin{equation}\label{eq:Jfac}
 \Jfun(\score)=\frac{\sqrt\mass}{1+\sqrt\mass}
 \left[\Gmat_{\strat}+\frac{1}{2\Dfun(\mass,\strat)}
 \mu_{\strat}\nu_{\mass,\strat}^\top\right].
\end{equation}
Assume $\norm q_\infty\le1$. By \eqref{eq:cov},
\begin{equation}
 \norm*{\Gmat_{\strat}q}_1\le\sqrt{\Var_{\strat}(q)}\le1,
 \qquad
 \norm*{\mu_{\strat}}_1\le\sqrt \Lreg,
 \qquad
 \norm*{\nu_{\mass,\strat}}_1\le1+\frac{\sqrt \Lreg}{2}\le\sqrt \Lreg.
\end{equation}
Together with \eqref{eq:Dlow}, this gives
\begin{equation}
 \frac{\norm*{\mu_{\strat}}_1\norm*{\nu_{\mass,\strat}}_1}
      {2\Dfun(\mass,\strat)}
 \le\frac{\Lreg}{\Lreg+1}<1.
\end{equation}
Equation \eqref{eq:Jfac} now gives
\begin{equation}
 \norm*{D\choice(\score)[q]}_1
 \le 2\frac{\sqrt\mass}{1+\sqrt\mass}\norm q_\infty.
\end{equation}
Since $\sqrt\mass/(1+\sqrt\mass)\le1/2$ and $\sqrt\mass/(1+\sqrt\mass)\le\sqrt\mass$, \eqref{eq:Qunif} and \eqref{eq:Qmass} follow. Next, \eqref{eq:Dlow} and $\norm*{\nu_{\mass,\strat}}_1\le\sqrt\Lreg$ give
\begin{equation}
 |D\mass(\score)[q]|
 \le\frac{2\sqrt\Lreg}{\Lreg+1}\mass^{3/2}\norm q_\infty
 \le\mass^{3/2}\norm q_\infty.
\end{equation}
For the logarithmic derivative, put $s=\sqrt\mass$. Since $\sqrt{\Lreg-2}\le\sqrt\Lreg-1/\sqrt\Lreg$ and, for $\sqrt\Lreg\ge2$, $3\sqrt\Lreg/8+1/(2\sqrt\Lreg)\ge1$,
\begin{equation}
 \norm*{\nu_{\mass,\strat}}_1
 \le1+\frac12\sqrt{\Lreg-2}
 \le\frac78\sqrt\Lreg.
\end{equation}
Moreover, \eqref{eq:hrange} and \eqref{eq:logvar} give
\begin{equation}
 4\Dfun(\mass,\strat)
 \ge
 \Lreg\left[
 \frac{3}{(1-s^2)^{3/2}}-\frac{1}{1+s}
 \right].
\end{equation}
For $0\le s<1$,
\begin{equation}
 \frac{3}{(1-s^2)^{3/2}}-\frac{1}{1+s}\ge7s.
\end{equation}
Indeed, $(1-s^2)^{-1/2}\ge1+s^2/2$, so it is enough that $2-6s+\frac32s^2+7s^3\ge0$. This is clear for $s\le1/6$, while for $s\ge1/6$ the left-hand side equals
\begin{equation}
 (s-\tfrac12)^2(7s+\tfrac{17}{2})+\frac{3s}{4}-\frac18\ge0.
\end{equation}
Using \eqref{eq:dmass} and $\sqrt\Lreg\ge2$ therefore gives
\begin{equation}\label{eq:logmass}
 |D\log\mass(\score)[q]|
 \le\frac14\norm q_\infty.
\end{equation}
\end{proof}

For $\score,\score'\in\R^\nPures$, with $\lmir(\score)=\mass \strat$ and $\lmir(\score')=\mass'\strat'$, integration of \eqref{eq:logmass} gives
\begin{equation}\label{eq:masscmp}
 \exp\!\left(-\frac14\norm*{\score'-\score}_\infty\right)
 \le\frac{\mass'}{\mass}
 \le
 \exp\!\left(\frac14\norm*{\score'-\score}_\infty\right).
\end{equation}
By \eqref{eq:lregdec},
\begin{equation}\label{eq:bregdec}
 \bregofX{\lreg}{\mass'\strat'}{\mass \strat}
 =\mass'\dkl(\strat'\|\strat)
 +\bregofX{\areg}{\mass'\strat'}{\mass \strat}.
\end{equation}

\begin{lemma}[Local Bregman coercivity]\label{lem:coerc}
Let $\score,\score'\in\R^\nPures$. Suppose $\lmir(\score)=\mass \strat$, $\lmir(\score')=\mass'\strat'$, and $\norm*{\score'-\score}_\infty\le M\le1$. Then
\begin{equation}\label{eq:coerc}
 \bregofX{\lreg}{\lmir(\score')}{\lmir(\score)}
 \ge e^{-M/4}\left[
 \mass\dkl(\strat'\|\strat)
 +\frac{(\sqrt{\mass'}-\sqrt\mass)^2}{\sqrt\mass}\right],
\end{equation}
and
\begin{equation}\label{eq:l1}
 \norm*{\strat'-\strat}_1^2
 \le\frac{2e^{M/4}}{\mass}\bregofX{\lreg}{\lmir(\score')}{\lmir(\score)}.
\end{equation}
\end{lemma}

\begin{proof}
By \eqref{eq:masscmp},
\begin{equation}\label{eq:massrat}
 e^{-M/4}\mass\le\mass'\le e^{M/4}\mass.
\end{equation}
Hence
\begin{equation}\label{eq:stratkl}
 \mass'\dkl(\strat'\|\strat)\ge e^{-M/4}\mass\dkl(\strat'\|\strat).
\end{equation}
Let $z(\param)=(1-\param)\lmir(\score)+\param \lmir(\score')$ and $\mass(\param)=\braket*{\one}{z(\param)}$. By \eqref{eq:massrat}, $\mass(\param)\le e^{M/4}\mass$. Therefore \eqref{eq:fconv} gives
\begin{align}
 \bregofX{\areg}{\lmir(\score')}{\lmir(\score)}
 &=\int_0^1(1-\param)D^2\areg(z(\param))[\lmir(\score')-\lmir(\score),\lmir(\score')-\lmir(\score)]\,\dd\param\\
 &\ge\frac \Lreg4e^{-3M/8}
 \frac{(\mass'-\mass)^2}{\mass^{3/2}}\\
 &\ge e^{-M/4}
 \frac{(\sqrt{\mass'}-\sqrt\mass)^2}{\sqrt\mass}.
\end{align}
For the last inequality, \eqref{eq:massrat} gives $\sqrt{\mass'/\mass}\ge e^{-M/8}$, while $\Lreg\ge4$ and $e^{-M/8}\ge1-M/8\ge7/8$ imply
\begin{equation}
 \frac \Lreg4e^{-M/8}(1+e^{-M/8})^2
 \ge\frac78\left(\frac{15}{8}\right)^2>1.
\end{equation}
Now \eqref{eq:bregdec} and \eqref{eq:stratkl} prove \eqref{eq:coerc}. Finally, \eqref{eq:info} and \eqref{eq:stratkl} give
\begin{equation}
 \bregofX{\lreg}{\lmir(\score')}{\lmir(\score)}
 \ge\frac12e^{-M/4}\mass\norm*{\strat'-\strat}_1^2,
\end{equation}
which is \eqref{eq:l1}.
\end{proof}

The next lemma proves the continuity bounds used in the Jacobian calculation.

\begin{lemma}[Square-root stability]\label{lem:sqrt}
Let $\strat,\strat'\in\istrats^\circ$ and $\mass,\mass'\in(0,1)$. Set $\eps=\norm*{\sqrt{\strat'}-\sqrt \strat}_2$. Then
\begin{align}
 |\hreg(\strat')-\hreg(\strat)|&\le5\sqrt \Lreg\,\eps,                          \label{eq:sqrth}\\
 \norm*{\mu_{\strat'}-\mu_{\strat}}_1&\le14\sqrt \Lreg\,\eps,                   \label{eq:sqrtu}\\
 |\Var_{\strat'}(\log \strat')-\Var_{\strat}(\log \strat)|&\le31\Lreg\,\eps,                                  \label{eq:sqrtV}\\
 \opnorm{\Gmat_{\strat'}-\Gmat_{\strat}}{\infty}{1}&\le3\eps.                    \label{eq:sqrtC}
\end{align}
Moreover,
\begin{align}
 \norm*{\nu_{\mass,\strat}}_1
 &\le\sqrt \Lreg,                                                        \label{eq:vbd}\\
 \norm*{\nu_{\mass',\strat'}-\nu_{\mass,\strat}}_1
 &\le8\sqrt \Lreg\,\eps
 +\frac12\sqrt \Lreg\,|\sqrt{\mass'}-\sqrt\mass|.                   \label{eq:vdiff}
\end{align}
\end{lemma}

\begin{proof}
For $0\le\param\le1$, put
\begin{equation}
 a(\param)=(1-\param)\sqrt \strat+\param\sqrt{\strat'},
 \qquad p_\pure(\param)=a_\pure(\param)^2.
\end{equation}
Since $\norm*{a(\param)}_2\le(1-\param)\norm*{\sqrt \strat}_2 +\param\norm*{\sqrt{\strat'}}_2=1$, we have $\sum_\pure p_\pure(\param)\le1$. For $f_1(s)=s^2\log s^2$ and $f_2(s)=s^2(\log s^2)^2$, with the values and derivatives at $s=0$ defined by continuity,
\begin{equation}
 |f_1'(s)|^2=4s^2(1+\log s^2)^2,
 \qquad
 |f_2'(s)|^2=4s^2(\log s^2)^2(2+\log s^2)^2.
\end{equation}
Since $p_\pure(\param)\le1$, $\log p_\pure(\param)\le0$. Thus \eqref{eq:log2}--\eqref{eq:log4} give
\begin{align}
 \sum_\pure|f_1'(a_\pure(\param))|^2
 &\le4+4\sum_\pure p_\pure(\param)|\log p_\pure(\param)|^2
 \le16\Lreg+4,\\
 \sum_\pure|f_2'(a_\pure(\param))|^2
 &\le4\sum_\pure p_\pure(\param)|\log p_\pure(\param)|^4
 +16\sum_\pure p_\pure(\param)|\log p_\pure(\param)|^2\\
 &\le104\Lreg^2+64\Lreg.
\end{align}
Since $\Lreg\ge4$, the right-hand sides are smaller than $25\Lreg$ and $121\Lreg^2$, respectively. Integration and Cauchy--Schwarz therefore give
\begin{align}
 \sum_\pure|\strat_\pure'\log \strat_\pure'-\strat_\pure\log \strat_\pure|
 &\le5\sqrt \Lreg\,\eps,                              \label{eq:logcont}\\
 \sum_\pure|\strat_\pure'(\log \strat_\pure')^2-\strat_\pure(\log \strat_\pure)^2|
 &\le11\Lreg\,\eps.
\end{align}
Also,
\begin{equation}\label{eq:sqrtx}
 \norm*{\strat'-\strat}_1
 \le\norm*{\sqrt{\strat'}-\sqrt \strat}_2\norm*{\sqrt{\strat'}+\sqrt \strat}_2
 \le2\eps.
\end{equation}
Together with \eqref{eq:logcont}, this proves \eqref{eq:sqrth}. Next, using $\mu_{\strat}=\strat\odot\log \strat-\hreg(\strat)\strat$, \eqref{eq:sqrtx}, and $|\hreg(\strat)|\le\log \nPures\le2\sqrt \Lreg$,
\begin{align}
 \norm*{\mu_{\strat'}-\mu_{\strat}}_1
 &\le 5\sqrt \Lreg\,\eps
      +|\hreg(\strat')-\hreg(\strat)|+|\hreg(\strat)|\norm*{\strat'-\strat}_1\\
 &\le14\sqrt \Lreg\,\eps.
\end{align}
This proves \eqref{eq:sqrtu}. Similarly,
\begin{align}
 |\Var_{\strat'}(\log \strat')-\Var_{\strat}(\log \strat)|
 &\le11\Lreg\,\eps
 +|\hreg(\strat')-\hreg(\strat)|\bigl(|\hreg(\strat')|+|\hreg(\strat)|\bigr)\\
 &\le31\Lreg\,\eps,
\end{align}
which proves \eqref{eq:sqrtV}. For \eqref{eq:sqrtC}, note that
\begin{equation}
 \Gmat_{\strat}
 =\diag(\sqrt \strat)
 \bigl(I-\sqrt \strat(\sqrt \strat)^\top\bigr)
 \diag(\sqrt \strat).
\end{equation}
Hence
\begin{align}
 \Gmat_{\strat'}-\Gmat_{\strat}
 &=\diag(\sqrt{\strat'}-\sqrt \strat)
   \bigl(I-\sqrt{\strat'}(\sqrt{\strat'})^\top\bigr)
   \diag(\sqrt{\strat'})\\
 &\quad+\diag(\sqrt \strat)
   \bigl(\sqrt \strat(\sqrt \strat)^\top-\sqrt{\strat'}(\sqrt{\strat'})^\top\bigr)
   \diag(\sqrt{\strat'})\\
 &\quad+\diag(\sqrt \strat)
   \bigl(I-\sqrt \strat(\sqrt \strat)^\top\bigr)
   \diag(\sqrt{\strat'}-\sqrt \strat).
\end{align}
For vectors $c,d$ and a matrix $A$,
\begin{equation}
 \opnorm{\diag(c)A\diag(d)}{\infty}{1}
 \le\norm c_2\opnorm{A}{2}{2}\norm d_2.
\end{equation}
Since $\norm*{\sqrt \strat}_2=\norm*{\sqrt{\strat'}}_2=1$, the first and third terms are at most $\eps$. Moreover,
\begin{align}
 \opnorm{\sqrt \strat(\sqrt \strat)^\top-\sqrt{\strat'}(\sqrt{\strat'})^\top}{2}{2}^2
 &=1-\braket*{\sqrt \strat}{\sqrt{\strat'}}^2 \\
 &\le 2\bigl(1-\braket*{\sqrt \strat}{\sqrt{\strat'}}\bigr)\\
 &=\norm*{\sqrt{\strat'}-\sqrt \strat}_2^2 =\eps^2.
\end{align}
Thus the middle term is also at most $\eps$, proving \eqref{eq:sqrtC}. Next, by \eqref{eq:cov} and $\Lreg\ge4$,
\begin{equation}
 \norm*{\nu_{\mass,\strat}}_1
 \le1+\frac12\norm*{\mu_{\strat}}_1
 \le1+\frac12\sqrt{\Lreg-2}
 \le\sqrt \Lreg,
\end{equation}
which proves \eqref{eq:vbd}. Finally, subtracting the two expressions in \eqref{eq:nvec} and using \eqref{eq:sqrtx}, \eqref{eq:sqrtu}, and $|(1+\sqrt{\mass'})^{-1}-(1+\sqrt\mass)^{-1}|\le|\sqrt{\mass'}-\sqrt\mass|$ gives
\begin{align}
 \norm*{\nu_{\mass',\strat'}-\nu_{\mass,\strat}}_1
 &\le2\eps+7\sqrt \Lreg\,\eps
 +\frac12\sqrt \Lreg\,|\sqrt{\mass'}-\sqrt\mass|\\
 &\le8\sqrt \Lreg\,\eps
 +\frac12\sqrt \Lreg\,|\sqrt{\mass'}-\sqrt\mass|,
\end{align}
where $\sqrt \Lreg\ge2$. This proves \eqref{eq:vdiff}.
\end{proof}

\begin{lemma}[Jacobian stability]\label{lem:Jstab}
Let $\score,\score'\in\R^\nPures$. If $\norm*{\score'-\score}_\infty\le M\le1$, then
\begin{equation}\label{eq:Jstab}
 \opnorm{\Jfun(\score')-\Jfun(\score)}{\infty}{1}^2
 \le2000e^{M/4}\bregofX{\lreg}{\lmir(\score')}{\lmir(\score)}.
\end{equation}
\end{lemma}

\begin{proof}
Write $\lmir(\score)=\mass \strat$, $\lmir(\score')=\mass'\strat'$, and $\eps=\norm*{\sqrt{\strat'}-\sqrt \strat}_2$. Holding $\mass$ fixed, \eqref{eq:sqrth} and \eqref{eq:sqrtV} give
\begin{equation}
 |\Dfun(\mass,\strat')-\Dfun(\mass,\strat)|
 \le\frac14(5\sqrt \Lreg+31\Lreg)\eps.
\end{equation}
Using \eqref{eq:Dlow} at both $\strat$ and $\strat'$ therefore gives
\begin{equation}\label{eq:Dix}
 |\Dfun(\mass,\strat')^{-1}-\Dfun(\mass,\strat)^{-1}|
 \le\frac{35}{\Lreg}\eps.
\end{equation}
Indeed, $\sqrt\Lreg\le\Lreg/2$ and \eqref{eq:Dlow} give
\begin{equation}
 \frac{5\sqrt\Lreg+31\Lreg}{(\Lreg+1)^2}
 <\frac{35}{\Lreg}.
\end{equation}
For fixed $\strat$, the function $s\mapsto\Dfun(s^2,\strat)$ satisfies
\begin{equation}
 \frac{\dd}{\dd s}\Dfun(s^2,\strat)
 =\frac{9\Lreg s}{4(1-s^2)^{5/2}}
 +\frac{\Var_{\strat}(\log \strat)}{4(1+s)^2}.
\end{equation}
Together with \eqref{eq:Dlow} and $\Var_{\strat}(\log \strat)\le\Lreg$, for $0<s<1$,
\begin{align}
 \left|\frac{\dd}{\dd s}\Dfun(s^2,\strat)^{-1}\right|
 &\le\frac9{\Lreg}s\sqrt{1-s^2}
 +\frac1{\Lreg}\frac{(1-s^2)^3}{(1+s)^2}\\
 &\le\frac{11}{2\Lreg}<\frac6\Lreg.
\end{align}
Here $s\sqrt{1-s^2}\le1/2$. The mean value theorem therefore gives
\begin{equation}\label{eq:Dim}
 |\Dfun(\mass',\strat)^{-1}-\Dfun(\mass,\strat)^{-1}|
 \le\frac6\Lreg|\sqrt{\mass'}-\sqrt\mass|.
\end{equation}
Set now
\begin{equation}
 B(\mass,\strat)=\Gmat_{\strat}+\frac{\mu_{\strat}\nu_{\mass,\strat}^\top}{2\Dfun(\mass,\strat)}.
\end{equation}
By \eqref{eq:Dix}--\eqref{eq:Dim},
\begin{equation}
 |\Dfun(\mass',\strat')^{-1}-\Dfun(\mass,\strat)^{-1}|
 \le\frac{35}{\Lreg}\eps+\frac6\Lreg|\sqrt{\mass'}-\sqrt\mass|.
\end{equation}
Using this bound together with \eqref{eq:sqrtC}, \eqref{eq:sqrtu}, \eqref{eq:vbd}, \eqref{eq:vdiff}, and $2\Dfun\ge \Lreg+1$, we obtain
\begin{align}
 \opnorm{B(\mass',\strat')-B(\mass,\strat)}{\infty}{1}
 &\le3\eps+14\eps
 +(8\eps+\tfrac12|\sqrt{\mass'}-\sqrt\mass|)
 +(\tfrac{35}{2}\eps+3|\sqrt{\mass'}-\sqrt\mass|)\\
 &\le43\eps+4|\sqrt{\mass'}-\sqrt\mass|.
\end{align}
Moreover, by \eqref{eq:cov}, \eqref{eq:vbd}, and \eqref{eq:Dlow},
\begin{equation}
 \opnorm{\Gmat_{\strat}}{\infty}{1}\le1,
 \qquad
 \frac{\norm*{\mu_{\strat}}_1\norm*{\nu_{\mass,\strat}}_1}
      {2\Dfun(\mass,\strat)}
 \le\frac{\Lreg}{\Lreg+1}<1.
\end{equation}
Thus $\opnorm{B(\mass,\strat)}{\infty}{1}\le2$. Since
\begin{equation}
 \left|\frac{\sqrt{\mass'}}{1+\sqrt{\mass'}}-\frac{\sqrt\mass}{1+\sqrt\mass}\right|
 \le|\sqrt{\mass'}-\sqrt\mass|,
\end{equation}
\eqref{eq:Jfac} yields
\begin{equation}
 \opnorm{\Jfun(\score')-\Jfun(\score)}{\infty}{1}
 \le43\sqrt\mass\eps+6|\sqrt{\mass'}-\sqrt\mass|.
\end{equation}
Hence, by Cauchy--Schwarz, \eqref{eq:info}, and $0<\sqrt\mass\le1$,
\begin{align}
 \opnorm{\Jfun(\score')-\Jfun(\score)}{\infty}{1}^2
 &\le(43^2+6^2)(\mass\eps^2+|\sqrt{\mass'}-\sqrt\mass|^2)\\
 &<1900\left[
 \mass\dkl(\strat'\|\strat)+\frac{|\sqrt{\mass'}-\sqrt\mass|^2}{\sqrt\mass}\right].
\end{align}
Now apply \eqref{eq:coerc}; enlarging $1900$ to $2000$ proves \eqref{eq:Jstab}.
\end{proof}

\begin{lemma}[Taylor remainder bound]\label{lem:taylor}
Let $\score,v\in\R^\nPures$. If $\norm v_\infty\le M\le1$, then
\begin{equation}\label{eq:taylor}
 \norm*{\choice(\score+v)-\choice(\score)-\Jfun(\score)v}_1^2
 \le2000e^{M/4}\norm v_\infty^2\bregofX{\lreg}{\lmir(\score+v)}{\lmir(\score)}.
\end{equation}
\end{lemma}

\begin{proof}
By integration,
\begin{equation}
 \choice(\score+v)-\choice(\score)-\Jfun(\score)v
 =\int_0^1[\Jfun(\score+\param v)-\Jfun(\score)]v\,\dd\param.
\end{equation}
Also
\begin{equation}
 \frac{\dd}{\dd\param}\bregofX{\lreg}{\lmir(\score+\param v)}{\lmir(\score)}
 =\braket*{\param v}{D\lmir(\score+\param v)[v]}\ge0,
\end{equation}
because $DZ=(D^2\lreg)^{-1}$ is positive definite by \eqref{eq:lhess}. Hence the Bregman divergence along the segment is at most its value at $\param=1$. Apply \eqref{eq:Jstab} inside the integral and then Cauchy--Schwarz to obtain \eqref{eq:taylor}.
\end{proof}

\section{Higher-order predictor and one-step expansions}\label{app:pred}

This section proves the predictor bounds used in \cref{app:energy}: first the discounted-difference bounds, then one payoff difference and one mirror-map difference.

\subsection{Discounted finite differences}\label{app:diff}

For the analysis, extend $\cpay_{i,t}=\pred_{i,t}=\perr_{i,t}=0$ to every $t\le0$ and recall $\disc=\nPlayers/(\nPlayers+1)$.  For a sequence $w=(w_t)_{t\in\Z}$, define
\begin{equation}
 (\diff w)_t=w_t-w_{t-1},
 \qquad
 (\dsum w)_t=\sum_{k=0}^{\infty}\disc^k w_{t-k}.
\end{equation}
The sum defining $\dsum w$ converges absolutely whenever $w$ is bounded. For the zero-prehistory sequences (that is, sequences whose nonpositive index terms vanish) used below, it is actually finite at every fixed time, so it is well defined without any boundedness assumption. Then, directly from the geometric sum,
\begin{equation}\label{eq:ginv}
 (\dsum w)_t-\disc(\dsum w)_{t-1}=w_t.
\end{equation}
Conversely, if $w_t=0$ for $t\le0$, the sum below is finite at every fixed $t$ and telescopes:
\begin{equation}
 \sum_{k=0}^{\infty}\disc^k
 (w_{t-k}-\disc w_{t-k-1})=w_t.
\end{equation}
Thus, on the space of sequences with zero prehistory, $\dsum$ and the map $w\mapsto(w_t-\disc w_{t-1})_t$ are inverse. The same conclusion holds for bounded sequences by absolute convergence. Moreover,
\begin{equation}
 \dsum\diff=\diff\dsum,
\end{equation}
because both sides at time $t$ are equal to $\sum_{k\ge0}\disc^k(w_{t-k}-w_{t-k-1})$. We now begin with the central identity that justifies the predictor's definition.

\begin{lemma}[Predictor identity]
For every player $i$,
\begin{equation}\label{eq:derr}
 \perr_i=\dsum^{\nPlayers+1}\diff^{\nPlayers+1}\cpay_i.
\end{equation}
\end{lemma}

\begin{proof}
For $t\ge1$, \cref{alg:hood} and \eqref{eq:perr} give
\begin{align}
 \perr_{i,t}
 &=\sum_{k=0}^{\nPlayers+1}(-1)^k\binom{\nPlayers+1}{k}\cpay_{i,t-k}
 -\sum_{k=1}^{\nPlayers+1}(-1)^k\binom{\nPlayers+1}{k}\disc^k \perr_{i,t-k}.
\end{align}
For $t\le0$ both sides are zero.  Equivalently,
\begin{equation}\label{eq:prec}
 \sum_{k=0}^{\nPlayers+1}(-1)^k\binom{\nPlayers+1}{k}\disc^k \perr_{i,t-k}
 =\diff^{\nPlayers+1}\cpay_{i,t},
 \qquad \forall t\in\Z.
\end{equation}
The left-hand side of \eqref{eq:prec} is obtained by applying $w\mapsto(w_t-\disc w_{t-1})_t$ successively $\nPlayers+1$ times to $\perr_i$.
Applying its inverse $\dsum$ successively $\nPlayers+1$ times to \eqref{eq:prec} gives \eqref{eq:derr}.
\end{proof}

For $0\le p\le \nPlayers+1$, define
\begin{equation}
 \dop_p
 =\dsum^{\nPlayers+1}\diff^{\nPlayers+1-p}
 =\dsum^p(\dsum\diff)^{\nPlayers+1-p},
\end{equation}
then, for $0\le p\le \nPlayers$
\begin{equation}\label{eq:Krec}
 \dop_p=\dop_{p+1}\diff.
\end{equation}
The next lemma collects important sequence bounds used in what will follow. Fix $r\in\{1,\infty\}$ and, for a sequence $w_t\in\R^\nPures$, write
\begin{equation}
 \norm*{w}_{\ell_2}
 =\left(\sum_{t\in\Z}\norm*{w_t}_r^2\right)^{1/2}.
\end{equation}
In particular, for scalar sequences, this is the usual $\ell_2$ norm.

\begin{lemma}[Higher-order differences sequence bounds]\label{lem:seq}
For either $r=1$ or $r=\infty$, every $1\le p\le \nPlayers+1$, and every sequence $w_t\in\R^\nPures$ with $\norm w_{\ell_2}<\infty$,
\begin{equation}\label{eq:Kl2}
 \norm*{\dop_pw}_{\ell_2}
 \le 4(\nPlayers+1)^{p+1}\norm w_{\ell_2}.
\end{equation}
\end{lemma}

\begin{proof}
We first establish the required bounds for scalar sequences. For every finitely supported scalar sequence $a$,
\begin{align}
 &4\sum_t(a_t-\disc a_{t-1})^2
 -(1+\disc)^2\norm*{\diff a}_{\ell_2}^2\\
 &\qquad
 =2(1-\disc)^2\sum_t(a_t^2+a_ta_{t-1})
 \ge0,                                                     \label{eq:dquad}
\end{align}
where the last inequality follows by summing $2a_ta_{t-1}\ge-a_t^2-a_{t-1}^2$ over $t$. Moreover, Minkowski's inequality gives
\begin{equation}
 \norm*{\dsum w}_{\ell_2}
 \le\sum_{k\ge0}\disc^k\norm*{(w_{t-k})_t}_{\ell_2}
 =(\nPlayers+1)\norm w_{\ell_2},
\end{equation}
so \eqref{eq:dquad} extends by density to every scalar $\ell_2$ sequence.
Since
\begin{equation}
 (\dsum w)_t-\disc(\dsum w)_{t-1}=w_t,
 \qquad
 \diff\dsum w=\dsum\diff w,
\end{equation}
applying \eqref{eq:dquad} to $\dsum w$ yields
\begin{equation}\label{eq:Gdl2}
 \norm*{(\dsum\diff)w}_{\ell_2}
 \le\frac{2}{1+\disc}\norm w_{\ell_2}.
\end{equation}
Set now
\begin{equation}
 \beta=\frac{2}{1+\disc}
 =\left(1-\frac{1}{2(\nPlayers+1)}\right)^{-1}.
\end{equation}
then, iterating \eqref{eq:Gdl2}, for $0\le q\le\nPlayers+1$ we obtain
\begin{equation}
 \norm*{(\dsum\diff)^qw}_{\ell_2}
 \le\beta^q\norm w_{\ell_2}.
\end{equation}
Furthermore,
\begin{equation}
 \beta^q
 \le
 \left(1-\frac{1}{2(\nPlayers+1)}\right)^{-(\nPlayers+1)}
 \le
 \exp\left(\frac{\nPlayers+1}{2\nPlayers+1}\right)
 \le e^{2/3}<2,
\end{equation}
where we used $-\log(1-u)\le u/(1-u)$. Consequently, for $ 0\le q\le\nPlayers+1$,
\begin{equation}\label{eq:Dpow}
 \norm*{(\dsum\diff)^qw}_{\ell_2}
 \le2\norm w_{\ell_2}.
\end{equation}
We next bound the scalar coefficients of $(\dsum\diff)^q$. Directly from the definition of $\dsum$,
\begin{equation}
 ((\dsum\diff)v)_t
 =\sum_{k\ge0}b_kv_{t-k},
\end{equation}
where
\begin{equation}
 b_0=1,
 \qquad
 b_k=-(1-\disc)\disc^{k-1},
 \quad k\ge1.
\end{equation}
Extend $b_k$ by $0$ for $k<0$. Define $a^{(0)}_0=1$ and $a^{(0)}_k=0$ for $k\ne0$, and inductively set
\begin{equation}\label{eq:kernel-recursion}
 a_k^{(q+1)}
 =\sum_{\ell=0}^k b_\ell a_{k-\ell}^{(q)},
 \qquad k\ge0,
\end{equation}
where we also extend $a_k^{(q)}$ by $0$ for $k<0$. Since $\sum_{k\ge0}|b_k|=2$, induction shows that $a^{(q)}$ is absolutely summable and that
\begin{equation}
 ((\dsum\diff)^qv)_t
 =\sum_{k\ge0}a_k^{(q)}v_{t-k}.
\end{equation}
Applying the preceding scalar $\ell_2$ estimate to a sequence supported at time $0$ gives
\begin{equation}\label{eq:kernel-l2}
 \left(\sum_{k\ge0}|a_k^{(q)}|^2\right)^{1/2}
 \le\beta^q.
\end{equation}
Set now
\begin{equation}
 s_q
 =\left(\sum_{k\ge0}k^2|a_k^{(q)}|^2\right)^{1/2}.
\end{equation}
We shall bound $s_q$ by induction, where the same argument also proves that it is finite. Indeed, multiplying \eqref{eq:kernel-recursion} by $k$ and writing $k=(k-\ell)+\ell$ gives
\begin{equation}
 k a_k^{(q+1)}
 =
 \sum_{\ell=0}^k b_\ell (k-\ell)a_{k-\ell}^{(q)}
 +
 \sum_{\ell=0}^k \ell b_\ell a_{k-\ell}^{(q)},
\end{equation}
and, by \eqref{eq:Gdl2} for the first sum, and by Minkowski's inequality for the second,
\begin{align}
 s_{q+1}
 &\le
 \beta s_q
 +\left(\sum_{\ell\ge0}\ell|b_\ell|\right)
  \left(\sum_{k\ge0}|a_k^{(q)}|^2\right)^{1/2}\\
 &\le
 \beta s_q+(\nPlayers+1)\beta^q,
\end{align}
because
\begin{equation}
 \sum_{\ell\ge0}\ell|b_\ell|
 =(1-\disc)\sum_{\ell\ge1}\ell\disc^{\ell-1}
 =\frac{1}{1-\disc}
 =\nPlayers+1,
\end{equation}
hence, since $s_0=0$, induction yields, for $1\le q\le\nPlayers+1$,
\begin{equation}\label{eq:ql2}
 s_q
 \le q(\nPlayers+1)\beta^{q-1}
 <2q(\nPlayers+1).
\end{equation}
Combining \eqref{eq:kernel-l2}, \eqref{eq:ql2}, and $\beta^q<2$, we obtain
\begin{equation}
 \sum_{k\ge0}|a_k^{(q)}|^2<4,
 \qquad
 \sum_{k\ge0}k^2|a_k^{(q)}|^2
 <4q^2(\nPlayers+1)^2.
\end{equation}
For $1\le q\le\nPlayers+1$, Cauchy--Schwarz gives
\begin{align}
 \left(\sum_{k\ge0}|a_k^{(q)}|\right)^2
 &\le
 \left(
 \sum_{k\ge0}
 \frac{1}{1+\left(k/[q(\nPlayers+1)]\right)^2}
 \right)\\
 &\qquad\times
 \left(
 \sum_{k\ge0}|a_k^{(q)}|^2
 \left(1+\frac{k^2}{q^2(\nPlayers+1)^2}\right)
 \right).
\end{align}
The first factor satisfies
\begin{align}
 \sum_{k\ge0}
 \frac{1}{1+\left(k/[q(\nPlayers+1)]\right)^2}
 &\le
 1+\int_0^\infty
 \frac{dx}{1+\left(x/[q(\nPlayers+1)]\right)^2}\\
 &=1+\frac{\pi q(\nPlayers+1)}2
 \le2(\nPlayers+1)^2,
\end{align}
where we used $\pi<7/2$, $q\le\nPlayers+1$, and $\nPlayers+1\ge2$, and the second factor is smaller than $8$, hence for $1\le q\le\nPlayers+1$,
\begin{equation}\label{eq:Dql1}
 \sum_{k\ge0}|a_k^{(q)}|
 <4(\nPlayers+1).
\end{equation}
For $q=0$, the same bound follows immediately from the definition of $a^{(0)}$. Let now $w_t\in\R^\nPures$. Minkowski's inequality and \eqref{eq:Dql1} give
\begin{align}
 \norm*{(\dsum\diff)^qw}_{\ell_2}
 &\le
 \sum_{k\ge0}|a_k^{(q)}|
 \norm*{(w_{t-k})_t}_{\ell_2}\\
 &\le
 4(\nPlayers+1)\norm w_{\ell_2}.
\end{align}
Repeated application of the corresponding bound for $\dsum$ also gives
\begin{equation}
 \norm*{\dsum^pw}_{\ell_2}
 \le(\nPlayers+1)^p\norm w_{\ell_2}.
\end{equation}
Since
\begin{equation}
 \dop_p
 =\dsum^p(\dsum\diff)^{\nPlayers+1-p},
\end{equation}
we conclude that
\begin{equation}
 \norm*{\dop_pw}_{\ell_2}
 \le4(\nPlayers+1)^{p+1}\norm w_{\ell_2},
\end{equation}
which proves \eqref{eq:Kl2}. If $w$ is supported at time $t_0$ with value $v$, then
\begin{equation}
 ((\dsum\diff)^{\nPlayers+1-p}w)_t
 =a_{t-t_0}^{(\nPlayers+1-p)}v.
\end{equation}
Therefore, \eqref{eq:kernel-l2} and $\beta^{\nPlayers+1-p}<2$ imply
\begin{equation}
 \norm*{(\dsum\diff)^{\nPlayers+1-p}w}_{\ell_2}
 \le2\norm v_r.
\end{equation}
Combining this with the bound for $\dsum^p$ gives
\begin{equation}\label{eq:Kpt}
 \norm*{\dop_pw}_{\ell_2}
 \le2(\nPlayers+1)^p\norm v_r.
\end{equation}
Finally, \eqref{eq:Dql1} with $q=\nPlayers+1$ gives
\begin{equation}\label{eq:dinf}
 \sup_t\norm*{((\dsum\diff)^{\nPlayers+1}w)_t}_r
 \le4(\nPlayers+1)\sup_t\norm*{w_t}_r.
\end{equation}
\end{proof}

Applying \eqref{eq:derr} and \eqref{eq:dinf} to $\cpay_i$, and using \eqref{eq:gprop}, gives
\begin{equation}\label{eq:eunif}
 \norm*{\perr_{i,t}}_\infty\le8(\nPlayers+1).
\end{equation}
Since $\pred_{i,t}=\cpay_{i,t}-\perr_{i,t}$ and $\nPlayers+1\ge2$,
\begin{equation}\label{eq:punif}
 \norm*{\pred_{i,t}}_\infty\le9(\nPlayers+1).
\end{equation}

\subsection{Payoff differences}\label{app:paydiff}

We call any strategy vector $\strat_{k,t}$ or $\strat_{k,t-1}$ appearing as an argument of a multilinear payoff map a \emph{strategy term}. Now, for $t\le0$, set $\strat_{i,t}=\choice(0)$ and $\strat_t=(\strat_{1,t},\ldots,\strat_{\nPlayers,t})$ and for every $t\in\Z$, define
\begin{equation}\label{eq:gext}
 \gext_{i,t}=\pay_i(\,\cdot\,,\strat_{-i,t})-\pay_i(\strat_t)\one.
\end{equation}
Since payoffs lie in $[-1,1]$,
\begin{equation}\label{eq:gextbd}
 \norm*{\gext_{i,t}}_\infty\le2,
\end{equation}
and, for $t\ge1$, $\gext_{i,t}=\cpay_{i,t}$, whereas $\cpay_{i,t}=0$ for $t\le0$.
Consequently, for $t \in \Z$,
\begin{equation}\label{eq:gbdry}
 \diff \cpay_{i,t}=\mathbf 1_{\{t=1\}}\gext_{i,0}+\diff\gext_{i,t}.
\end{equation}
We shall also repeatedly use the following elementary product rule.  If $F$ is $q$-linear, then
\begin{align}
 &F(z_{1,t},\ldots,z_{q,t})-F(z_{1,t-1},\ldots,z_{q,t-1})\\
 &\quad=\sum_{k=1}^q
 F(z_{1,t-1},\ldots,z_{k-1,t-1},
   \diff z_{k,t},z_{k+1,t},\ldots,z_{q,t}).
 \label{eq:multidiff}
\end{align}

\begin{lemma}[One payoff difference]\label{lem:paydiff}
For every player $i$ and every $t\in\Z$, there are linear maps $C_{ij,t}^v$ ($j\ne i$) and $C_{ij,t}^u$ ($j\in\players$) such that
\begin{equation}\label{eq:paydiff}
 \diff\gext_{i,t}
 =\sum_{j\ne i}C_{ij,t}^{\payv}\diff \strat_{j,t}
 +\sum_{j=1}^{\nPlayers}C_{ij,t}^{\pay}\diff \strat_{j,t},
\end{equation}
with
\begin{equation}\label{eq:Ccontr}
 \norm*{C_{ij,t}^{\payv}w}_\infty\le\norm w_1,
 \qquad
 \norm*{C_{ij,t}^{\pay}w}_\infty\le\norm w_1.
\end{equation}
Each coefficient contains at most $\nPlayers-1$ strategy terms, with each player label appearing at most once; consequently, each fixed label occurs at most twice in \eqref{eq:paydiff}.
\end{lemma}

\begin{proof}
Apply \eqref{eq:multidiff} to $\strat_{-i}\mapsto \pay_i(\,\cdot\,,\strat_{-i})$ and $\strat\mapsto \pay_i(\strat)$, and absorb the minus sign from \eqref{eq:gext} into $C_{ij,t}^u$. More explicitly, $C_{ij,t}^v$ is obtained from $\strat_{-i}\mapsto \pay_i(\,\cdot\,,\strat_{-i})$ by leaving player $j$'s argument free and fixing every other argument at a strategy term; $C_{ij,t}^u$ is obtained analogously from $\strat\mapsto-\pay_i(\strat)\one$, so that $C_{ij,t}^v w$ and $C_{ij,t}^u w$ are obtained by inserting $w$ in that free argument. This gives \eqref{eq:paydiff} and shows that each coefficient contains at most $\nPlayers-1$ strategy terms, with each player label appearing at most once. A fixed $j$ can occur once in each of the two sums in \eqref{eq:paydiff}, hence at most twice overall. Finally, if $|a_\purealt|\le1$, then
\begin{equation}
 \left|\sum_\purealt a_\purealt\prod_{k=1}^q z_{k,\purealt_k}\right|
 \le\prod_{k=1}^q\norm*{z_k}_1.
\end{equation}
Taking one $z_k=w$ and the others in $\istrats$ proves \eqref{eq:Ccontr}.
\end{proof}

\subsection{Differences control}\label{app:mdiff}

By \eqref{eq:paydiff}, it remains to control $\diff \strat_{i,t}$. For $t\ge1$, retain the notation $\score_{i,t-1/2},\zlift_{i,t-1/2},\zlift_{i,t},\PV_{i,t},\NV_{i,t}$ introduced in \cref{sec:proof}.
For $t\le0$, set
\begin{equation}\label{eq:pre}
 \score_{i,t}=\pred_{i,t}=\perr_{i,t}=\score_{i,t-1/2}=0,
 \quad
 \zlift_{i,t}=\zlift_{i,t-1/2}=\lmir(0),
 \quad
 \strat_{i,t}=\choice(0),
 \quad
 \PV_{i,t}=\NV_{i,t}=0.
\end{equation}
Set
\begin{equation}\label{eq:yid}
 \score_{i,t-1/2}=\score_{i,t-1}+\pred_{i,t}=\score_{i,t}-\perr_{i,t}.
\end{equation}
Then
\begin{equation}\label{eq:ydiff}
 \diff \score_{i,t-1/2}=\cpay_{i,t}-\diff \perr_{i,t}.
\end{equation}
For $t\ge1$, set
\begin{equation}
 \mass_{i,t-1}=\braket*{\one}{\zlift_{i,t-1}},
 \qquad
 \jac_{i,t}=\learn\Jfun(\learn \score_{i,t-1}),
\end{equation}
and set $\jac_{i,t}=\learn\Jfun(0)$ for $t\le0$. The required bounds on $\diff \strat_{i,t}$, $\remd_{i,t}$, and $\diff \jac_{i,t}$ are collected next.

\begin{lemma}[Expansion terms bounds]\label{lem:mir1}
Assume \eqref{eq:eta}. Then, for every player $i$ and every $t\ge1$,
\begin{align}
 \opnorm{\jac_{i,t}}{\infty}{1}
 &\le\learn,                                           \label{eq:Jbd}\\
 \opnorm{\jac_{i,t}}{\infty}{1}
 &\le2\learn\sqrt{\mass_{i,t-1}}.                    \label{eq:Jmass}
\end{align}
Moreover, if
\begin{equation}\label{eq:mir1}
 \diff \strat_{i,t}=\jac_{i,t}\diff \score_{i,t-1/2}+\remd_{i,t},
\end{equation}
then, for every $\horizon\ge1$,
\begin{align}
 \sum_{t=1}^\horizon\norm*{\remd_{i,t}}_1^2
 &\le300000(\nPlayers+1)^2\learn^3
 \sum_{t=1}^\horizon(\PV_{i,t}+\NV_{i,t}),                            \label{eq:reng}\\
 \sum_{t=1}^\horizon\opnorm{\jac_{i,t}-\jac_{i,t-1}}{\infty}{1}^2
 &\le4100\learn^3
 \sum_{t=1}^\horizon(\PV_{i,t}+\NV_{i,t}),                            \label{eq:Jeng}\\
 \norm*{\diff \strat_{i,t}}_1^2
 &\le\frac{5\learn}{\mass_{i,t-1}}
 (\PV_{i,t}+\NV_{i,t-1}).                                     \label{eq:strat1}
\end{align}
Finally, if $t'\in\Z$ and $|t-t'|\le \nPlayers+1$, then
\begin{equation}\label{eq:mir2}
 \opnorm{\jac_{i,t'}}{\infty}{1}^2
 \norm*{\jac_{i,t}\diff \score_{i,t-1/2}}_1^2
 \le90\learn^3(\PV_{i,t}+\NV_{i,t-1}).
\end{equation}
\end{lemma}

\begin{proof}
The bounds \eqref{eq:Jbd}--\eqref{eq:Jmass} follow from \eqref{eq:Qunif}--\eqref{eq:Qmass}. Fix $t\ge1$. By \eqref{eq:yid},
\begin{equation}
 \strat_{i,t}=\choice\!\left(\learn(\score_{i,t-1}+\pred_{i,t})\right),
 \qquad
 \strat_{i,t-1}=\choice\!\left(\learn(\score_{i,t-1}-\perr_{i,t-1})\right),
\end{equation}
and $\diff \score_{i,t-1/2}=\pred_{i,t}+\perr_{i,t-1}$. Hence
\begin{equation}
 \remd_{i,t}
 =\bigl(\strat_{i,t}-\choice(\learn \score_{i,t-1})-\jac_{i,t}\pred_{i,t}\bigr)
 -\bigl(\strat_{i,t-1}-\choice(\learn \score_{i,t-1})+\jac_{i,t}\perr_{i,t-1}\bigr).
\end{equation}
By \eqref{eq:eunif}--\eqref{eq:punif} and \eqref{eq:eta},
\begin{equation}
 \learn\norm*{\pred_{i,t}}_\infty<\frac1{16},
 \qquad
 \learn\norm*{\perr_{i,t-1}}_\infty<\frac1{16}.
\end{equation}
Apply \eqref{eq:taylor} to the two brackets. Since $e^a\le(1-a)^{-1}$ for $0\le a<1$, $2000e^{1/64}<2050$. Using
\begin{equation}
 \bregofX{\lreg}{\zlift_{i,t-1/2}}{\zlift_{i,t-1}}=\learn \PV_{i,t},
 \qquad
 \bregofX{\lreg}{\zlift_{i,t-3/2}}{\zlift_{i,t-1}}=\learn \NV_{i,t-1},
\end{equation}
we obtain
\begin{equation}
 \norm*{\remd_{i,t}}_1^2
 \le2050(\nPlayers+1)^2\learn^3
 \left(9\sqrt{\PV_{i,t}}+8\sqrt{\NV_{i,t-1}}\right)^2.
\end{equation}
Since $(9a+8b)^2\le145(a^2+b^2)$,
\begin{equation}\label{eq:rpt}
 \norm*{\remd_{i,t}}_1^2
 <300000(\nPlayers+1)^2\learn^3(\PV_{i,t}+\NV_{i,t-1}).
\end{equation}
Summing in $t$ proves \eqref{eq:reng}. For \eqref{eq:Jeng}, insert $\learn \score_{i,t-3/2}$ between $\learn \score_{i,t-1}$ and $\learn \score_{i,t-2}$. The two score differences have $\ell_\infty$ norm smaller than $1/16$. Thus
\begin{align}
 \opnorm{\jac_{i,t}-\jac_{i,t-1}}{\infty}{1}^2
 &\le2\learn^2
 \opnorm{\Jfun(\learn \score_{i,t-1})-
             \Jfun(\learn \score_{i,t-3/2})}{\infty}{1}^2\\
 &\quad+2\learn^2
 \opnorm{\Jfun(\learn \score_{i,t-3/2})-
             \Jfun(\learn \score_{i,t-2})}{\infty}{1}^2\\
 &\le4000e^{1/64}\learn^3(\PV_{i,t-1}+\NV_{i,t-1})\\
 &<4100\learn^3(\PV_{i,t-1}+\NV_{i,t-1}),
\end{align}
where \eqref{eq:Jstab} was applied with the orientation matching the definitions of $\PV_{i,t-1}$ and $\NV_{i,t-1}$. Summing in $t$ and shifting the indices gives \eqref{eq:Jeng}. For \eqref{eq:strat1}, insert $\choice(\learn \score_{i,t-1})$ between $\strat_{i,t}$ and $\strat_{i,t-1}$. By \eqref{eq:l1}, with $\learn \score_{i,t-1}$ as the base score in both applications,
\begin{equation}
 \norm*{\strat_{i,t}-\choice(\learn \score_{i,t-1})}_1^2
 \le\frac{2e^{1/64}\learn}{\mass_{i,t-1}}\PV_{i,t},
 \qquad
 \norm*{\strat_{i,t-1}-\choice(\learn \score_{i,t-1})}_1^2
 \le\frac{2e^{1/64}\learn}{\mass_{i,t-1}}\NV_{i,t-1}.
\end{equation}
Therefore
\begin{equation}
 \norm*{\diff \strat_{i,t}}_1^2
 \le\frac{4e^{1/64}\learn}{\mass_{i,t-1}}
 (\PV_{i,t}+\NV_{i,t-1})
 <\frac{5\learn}{\mass_{i,t-1}}(\PV_{i,t}+\NV_{i,t-1}).
\end{equation}
It remains to prove \eqref{eq:mir2}. If $|t-t'|\le \nPlayers+1$, then
\begin{equation}
 \learn\norm*{\score_{i,t'-1}-\score_{i,t-1}}_\infty
 \le2\learn(\nPlayers+1)\le\frac1{120}.
\end{equation}
Since $\jac_{i,t'}=\learn\Jfun(\learn \score_{i,t'-1})$ also for $t'\le0$, \eqref{eq:masscmp} and \eqref{eq:Qmass} give
\begin{equation}
 \opnorm{\jac_{i,t'}}{\infty}{1}^2
 \le4e^{1/480}\learn^2\mass_{i,t-1}.
\end{equation}
Also, by \eqref{eq:mir1},
\begin{equation}
 \norm*{\jac_{i,t}\diff \score_{i,t-1/2}}_1^2
 \le2\norm*{\diff \strat_{i,t}}_1^2+2\norm*{\remd_{i,t}}_1^2.
\end{equation}
Using \eqref{eq:strat1} and \eqref{eq:rpt},
\begin{align}
 &\opnorm{\jac_{i,t'}}{\infty}{1}^2
 \norm*{\jac_{i,t}\diff \score_{i,t-1/2}}_1^2\\
 &\qquad\le4e^{1/480}\learn^3
 \left(10+600000(\nPlayers+1)^2\learn^2\right)
 (\PV_{i,t}+\NV_{i,t-1})\\
 &\qquad<90\learn^3(\PV_{i,t}+\NV_{i,t-1}),
\end{align}
where $600000(\nPlayers+1)^2\learn^2<11$ and $e^{1/480}<480/479$. This proves \eqref{eq:mir2}.
\end{proof}

For what follows, set $\remd_{i,t}=0$ for $t\le0$.  By \eqref{eq:pre}, for $t \le 0$,
\begin{equation}
 \diff \strat_{i,t}=\diff \score_{i,t-1/2}=0,
\end{equation}
so \eqref{eq:mir1} holds for every $t\in\Z$.

\section{Proof of the main theorem}\label{app:energy}

We combine \cref{app:lift,app:pred}: first the RVU inequality, then the prediction error bound, and finally \cref{thm:main}.

\subsection{The RVU inequality}\label{app:oftrl}

Recall from \cref{sec:proof} the optimistic score and lifted iterates
\begin{equation}
 \score_{i,t-1/2}=\score_{i,t-1}+\pred_{i,t},\qquad
 \zlift_{i,t-1/2}=\lmir(\learn \score_{i,t-1/2}),\qquad
 \zlift_{i,t}=\lmir(\learn \score_{i,t}),
\end{equation}
the Bregman variation terms
\begin{equation}
 \PV_{i,t}=\bregofX{\lreg/\learn}{\zlift_{i,t-1/2}}{\zlift_{i,t-1}},\qquad
 \NV_{i,t}=\bregofX{\lreg/\learn}{\zlift_{i,t-1/2}}{\zlift_{i,t}},
\end{equation}
and the cumulative quantities
\begin{equation}
 \PV(\horizon)=\sum_{i=1}^\nPlayers\sum_{t=1}^\horizon \PV_{i,t},\qquad
 \NV(\horizon)=\sum_{i=1}^\nPlayers\sum_{t=1}^\horizon \NV_{i,t},\qquad
 \EE(\horizon)=\sum_{i=1}^\nPlayers\sum_{t=1}^\horizon\norm*{\perr_{i,t}}_\infty^2.
\end{equation}
Finally, recall the range of the regularizer,
\begin{equation}
 \rrange=\max_{z\in\Kset}\lreg(z)-\min_{z\in\Kset}\lreg(z).
\end{equation}
We now prove the \acs{RVU} inequality.

\begin{proposition}[\acs{RVU} inequality]
For every $\horizon$,
\begin{equation}\label{eq:social}
 \sum_{i=1}^\nPlayers[\reg_i(\horizon)]_+
 \le \frac{\nPlayers\rrange}{\learn}+\frac\learn2 \EE(\horizon)-\PV(\horizon),
\end{equation}
and
\begin{equation}\label{eq:Nerr}
 \NV(\horizon)\le\frac\learn2 \EE(\horizon).
\end{equation}
\end{proposition}

\begin{proof}
Fix $i$. Since
\begin{equation}
 \strat_{i,t}=\frac{\zlift_{i,t-1/2}}{\braket*{\one}{\zlift_{i,t-1/2}}},
\end{equation}
\eqref{eq:gprop} gives
\begin{equation}
 \braket*{\zlift_{i,t-1/2}}{\cpay_{i,t}}=0.
\end{equation}
Therefore, because $\Kset=\conv(\{0,e_1,\ldots,e_\nPures\})$,
\begin{align}
 \max_{u\in\Kset}\sum_{t=1}^\horizon\braket*{u-\zlift_{i,t-1/2}}{\cpay_{i,t}}
 &=\max_{u\in\Kset}\braket*{u}{\score_{i,\horizon}}\\
 &=\max\{0,\max_\pure \score_{i\pure,\horizon}\}
 = [\reg_i(\horizon)]_+.                                       \label{eq:auxreg}
\end{align}
Let
\begin{equation}
 \Fpot=\left(\frac{\lreg}{\learn}+\iota_{\Kset}\right)^*,
\end{equation}
where $\iota_{\Kset}$ is $0$ on $\Kset$ and $+\infty$ outside $\Kset$, and $^*$ denotes convex conjugation. By \cref{lem:maps}, the maximizer is unique and smooth, so, for a generic score $\score$,
\begin{equation}
 \Fpot(\score)=\braket*{\score}{\lmir(\learn \score)}-\frac1\learn\lreg(\lmir(\learn \score)).
\end{equation}
Differentiating in a direction $q$ and using $\nabla\lreg(\lmir(\learn \score))=\learn \score$ from \eqref{eq:foc} gives
\begin{equation}
 D\Fpot(\score)[q]
 =\braket*{q}{\lmir(\learn \score)}+\learn\braket*{\score}{D\lmir(\learn \score)[q]}
  -\braket*{\nabla\lreg(\lmir(\learn \score))}{D\lmir(\learn \score)[q]}
 =\braket*{q}{\lmir(\learn \score)}.
\end{equation}
Hence
\begin{equation}\label{eq:phigrad}
 \nabla\Fpot(\score)=\lmir(\learn \score).
\end{equation}
If $p,q\in\R^\nPures$, then $\learn p=\nabla\lreg(\lmir(\learn p))$ and $\learn q=\nabla\lreg(\lmir(\learn q))$ by \eqref{eq:foc}. Fenchel equality therefore gives
\begin{equation}
 \bregofX{\Fpot}{p}{q}=\bregofX{\lreg/\learn}{\lmir(\learn q)}{\lmir(\learn p)}.
\end{equation}
Indeed, substituting $\Fpot(r)=\braket*{r}{\lmir(\learn r)}-\lreg(\lmir(\learn r))/\learn$ in the left side yields the right side.
Consequently,
\begin{equation}
 \PV_{i,t}=\bregofX{\Fpot}{\score_{i,t-1}}{\score_{i,t-1/2}},
 \qquad
 \NV_{i,t}=\bregofX{\Fpot}{\score_{i,t}}{\score_{i,t-1/2}}.
\end{equation}
Subtracting these two identities and using $\score_{i,t}-\score_{i,t-1}=\cpay_{i,t}$ gives
\begin{equation}
 \NV_{i,t}-\PV_{i,t}
 =\Fpot(\score_{i,t})-\Fpot(\score_{i,t-1})-\braket*{\zlift_{i,t-1/2}}{\cpay_{i,t}}.
\end{equation}
Hence, for every $u\in\Kset$,
\begin{align}
 \sum_{t=1}^\horizon\braket*{\cpay_{i,t}}{u-\zlift_{i,t-1/2}}
 &=\braket*{\score_{i,\horizon}}{u}-\Fpot(\score_{i,\horizon})+\Fpot(0)
   +\sum_{t=1}^\horizon(\NV_{i,t}-\PV_{i,t})\\
 &\le \frac{\lreg(u)-\min_{z\in\Kset}\lreg(z)}{\learn}
   +\sum_{t=1}^\horizon\NV_{i,t}-\sum_{t=1}^\horizon\PV_{i,t},
\end{align}
where we used the Fenchel--Young inequality, that is, $\braket*{\score_{i,\horizon}}{u}-\Fpot(\score_{i,\horizon})\le\lreg(u)/\learn$ and $\Fpot(0)=-\min_{\Kset}\lreg/\learn$. Maximizing in $u$ and using \eqref{eq:auxreg} proves
\begin{equation}\label{eq:indiv}
 [\reg_i(\horizon)]_+
 \le \frac{\rrange}{\learn}
 +\sum_{t=1}^\horizon\NV_{i,t}-\sum_{t=1}^\horizon\PV_{i,t}.
\end{equation}
It remains to bound $\NV_{i,t}$. Let $z',z\in\Kint$. By the $1$-strong convexity in \eqref{eq:lhess},
\begin{equation}
 \bregofX{\lreg/\learn}{z'}{z}
 \ge\frac1{2\learn}\norm*{z'-z}_1^2.
\end{equation}
For arbitrary $p,q\in\R^\nPures$, the first-order conditions imply
\begin{align}
 \braket*{p-q}{\lmir(\learn p)-\lmir(\learn q)}
 &=
 \bregofX{\lreg/\learn}{\lmir(\learn p)}{\lmir(\learn q)}
 +\bregofX{\lreg/\learn}{\lmir(\learn q)}{\lmir(\learn p)}
 \\
 &\ge
 \frac1\learn\norm*{\lmir(\learn p)-\lmir(\learn q)}_1^2.
\end{align}
By H\"older's inequality,
\begin{equation}
 \frac1\learn\norm*{\lmir(\learn p)-\lmir(\learn q)}_1^2
 \le\norm*{p-q}_\infty\norm*{\lmir(\learn p)-\lmir(\learn q)}_1.
\end{equation}
If $\lmir(\learn p)=\lmir(\learn q)$, the desired bound is immediate; otherwise, canceling one factor of $\norm*{\lmir(\learn p)-\lmir(\learn q)}_1$ proves
\begin{equation}\label{eq:Zlip}
 \norm*{\lmir(\learn p)-\lmir(\learn q)}_1
 \le\learn\norm*{p-q}_\infty.
\end{equation}
By \eqref{eq:phigrad} and integration,
\begin{align}
 \bregofX{\Fpot}{p}{q}
 &=\int_0^1
 \braket*{\lmir(\learn(q+\param(p-q)))-\lmir(\learn q)}{p-q}\,\dd\param\\
 &\le\int_0^1\learn\param\norm*{p-q}_\infty^2\,\dd\param
 =\frac\learn2\norm*{p-q}_\infty^2,
\end{align}
where the inequality uses \eqref{eq:Zlip}. Thus
\begin{equation}
 \bregofX{\Fpot}{p}{q}
 \le\frac\learn2\norm*{p-q}_\infty^2.
\end{equation}
Since $\NV_{i,t}=\bregofX{\Fpot}{\score_{i,t}}{\score_{i,t-1/2}}$ and $\score_{i,t}-\score_{i,t-1/2}=\perr_{i,t}$,
\begin{equation}\label{eq:Npt}
 \NV_{i,t}\le\frac\learn2\norm*{\perr_{i,t}}_\infty^2.
\end{equation}
Summing \eqref{eq:indiv} and \eqref{eq:Npt} over players and rounds gives \eqref{eq:social} and \eqref{eq:Nerr}.
\end{proof}

By \eqref{eq:social}--\eqref{eq:Nerr}, it remains to control $\EE(\horizon)$ in terms of $\PV(\horizon)+\NV(\horizon)$.  We combine the bounds of \cref{app:pred} below.

\subsection{Prediction error control}\label{app:elim}

Fix $\horizon$, and let all $\ell_2$ norms below be taken over $t=1,\ldots,\horizon$, with the underlying vector norm clear from context.
By the zero prehistory conventions above, every sequence to which we apply $\dop_p$ vanishes for $t\le0$. Moreover, $(\dop_pw)_t$ depends only on the values $w_s$ with $s\le t$, so values of the sequence after time $\horizon$ do not affect $(\dop_pw)_t$ for $1\le t\le \horizon$. Hence the bounds of \cref{lem:seq} apply directly on the finite horizon $1,\ldots,\horizon$, and we obtain
\begin{align}
 \norm*{\dop_pw}_{\ell_2}
 &\le4(\nPlayers+1)^{p+1}\norm w_{\ell_2},                  \label{eq:Kl2T}\\
 \norm*{\dop_pw}_{\ell_2}
 &\le2(\nPlayers+1)^p\norm v
 \quad\text{if }w_t=\mathbf1_{\{t=t_0\}}v.          \label{eq:KptT}
\end{align}
Summing \cref{lem:mir1} over players yields
\begin{align}
 \sum_{j=1}^\nPlayers\norm*{\remd_j}_{\ell_2}^2
 &\le300000(\nPlayers+1)^2\learn^3(\PV(\horizon)+\NV(\horizon)),                 \label{eq:rglob}\\
 \sum_{j=1}^\nPlayers\norm*{\diff \jac_j}_{\ell_2}^2
 &\le4100\learn^3(\PV(\horizon)+\NV(\horizon)),                          \label{eq:Jglob}\\
 \sum_{j=1}^\nPlayers\norm*{\diff \perr_j}_{\ell_2}^2
 &\le4\EE(\horizon).                                           \label{eq:deglob}
\end{align}
The third bound uses $\norm*{\diff \perr_{j,t}}_\infty^2\le 2\norm*{\perr_{j,t}}_\infty^2+2\norm*{\perr_{j,t-1}}_\infty^2$ and $\perr_{j,0}=0$.
Also, for $0\le\ell,u\le \nPlayers+1$, \eqref{eq:mir2} gives
\begin{equation}\label{eq:repglob}
 \sum_{t=1}^\horizon
 \opnorm{\jac_{j,t-\ell}}{\infty}{1}^2
 \norm*{\jac_{j,t-u}\diff \score_{j,t-u-\frac12}}_1^2
 \le90\learn^3\sum_{k=1}^\horizon(\PV_{j,k}+\NV_{j,k}).
\end{equation}
For $t-u>0$, this is \eqref{eq:mir2} at time $t-u$, since $|u-\ell|\le \nPlayers+1$; for $t-u\le0$ the summand is zero. We will also use the following Cauchy--Schwarz consequence.  If $\norm*{w_\nu}_{\ell_2}\le c\,u_{j_\nu}$, $|I|\le n_0$, and each label occurs at most $n_1$ times, then
\begin{equation}\label{eq:mult}
 \norm*{\sum_{\nu\in I} w_\nu}_{\ell_2}
 \le c\sqrt{n_0n_1}\left(\sum_{j=1}^\nPlayers u_j^2\right)^{1/2}.
\end{equation} Next, in the expansion below, we call $\jac_{j,t-q}$, $\cpay_{k,t-q}$, and
$\strat_{\ell,t-q}$ a \emph{Jacobian term}, \emph{payoff term}, and \emph{strategy term}, respectively, with $q$ denoting the corresponding \emph{time offset}, and each term in the expansion is multilinear in these factors. Accordingly, if its Jacobian, payoff, and strategy terms are replaced by arbitrary maps $B_a:\R^\nPures\to\R^\nPures$ and vectors $v_a,z_a$, we write the resulting expression as $M_t[(B_a);(v_a);(z_a)]$. The terms generated by the expansion satisfy
\begin{equation}\label{eq:pcontr}
 \norm*{M_t[(B_a);(v_a);(z_a)]}_\infty
 \le
 \prod_a\opnorm{B_a}{\infty}{1}
 \prod_a\norm*{v_a}_\infty
 \prod_a\norm*{z_a}_1,
\end{equation}
and the proof below shows that this bound is preserved whenever the expansion continues. The way the expansion continues becomes then simple. Indeed, whenever a strategy or payoff difference, after being expanded, reveals $\diff\strat_j$, \eqref{eq:mir1} and \eqref{eq:ydiff} give, suppressing time offsets,
\begin{equation}\label{eq:cont0}
 \diff\strat_j
 =
 \remd_j+\jac_j\cpay_j-\jac_j\diff\perr_j.
\end{equation}
The first and last terms are controlled directly, while $\jac_j\cpay_j$ can continue only if the player label $j$ is new; and if $j$ has already appeared among the Jacobian terms, the resulting repeated player term is controlled by \eqref{eq:repglob}, so every continuing term must introduce a fresh player label. The next lemma makes this recursive structure precise.

\begin{lemma}[One-step expansion]\label{lem:step}
Fix $1\le p\le \nPlayers$, and let $M$ be a term containing $p$ Jacobian terms with distinct player labels $j_1,\ldots,j_p$, together with $s$ payoff terms and $h$ strategy terms, where $1\le s\le p$ and $0\le h\le \nPlayers p$. Assume that all time offsets lie in $\{0,\ldots,p\}$, that each player occurs at most $p$ times among the strategy terms, and that $M$ satisfies \eqref{eq:pcontr}. Then
\begin{equation}\label{eq:step}
 \diff M=R_M+\sum_{M'\in\cont(M)}M',
\end{equation}
where $R_M$ collects the controlled terms and satisfies
\begin{align}
 \norm*{\dop_{p+1}R_M}_{\ell_2}
 &\le
 2p(\nPlayers+1)^{p+1}(2\learn)^p
 \\
 &\quad+
 815p\sqrt{\nPlayers}\,(\nPlayers+1)^{p+2}
 (2\learn)^{p-1}\learn^{3/2}
 \sqrt{\PV(\horizon)+\NV(\horizon)}
 \\
 &\quad+
 12p\sqrt{\nPlayers}\,(\nPlayers+1)^{p+2}
 (2\learn)^{p+1}\sqrt{\EE(\horizon)}.
 \label{eq:term}
\end{align}
Moreover, every $M'\in\cont(M)$ satisfies the same conditions as $M$ with $p$ replaced by $p+1$, and
\begin{equation}\label{eq:cont}
 |\cont(M)|\le3p(\nPlayers-p).
\end{equation}
In particular, $\cont(M)=\varnothing$ when $p=\nPlayers$.
\end{lemma}

\begin{proof}
Apply the discrete product rule \eqref{eq:multidiff} to $M$, so that exactly one Jacobian, payoff, or strategy term is differenced, while all the other terms remain present, possibly with their time offsets increased by one. We now need to determine which terms may continue. First, a differenced Jacobian is controlled directly; next if a payoff term is differenced, \eqref{eq:gbdry} separates an initialization term and \cref{lem:paydiff} expresses the remaining part as a sum of strategy variations $\diff\strat_j$; and differencing a strategy term produces such a variation directly.
Since each of the $s$ payoff terms produces at most $2\nPlayers-1$ strategy variations, the total number of revealed variations is at most
\begin{equation}
 (2\nPlayers-1)s+h\le3\nPlayers p.
\end{equation}
More importantly, for any fixed player $j$, \cref{lem:paydiff} can produce the label $j$ at most twice from each payoff term, while by assumption $j$ occurs at most $p$ times among the strategy terms.
Hence each label is revealed at most
\begin{equation}\label{eq:labelmult}
 2s+p\le3p
\end{equation}
times.
For every revealed $\diff\strat_j$, use \eqref{eq:cont0}:
\begin{equation}
 \diff\strat_j
 =
 \remd_j+\jac_j\diff\score_{j,\cdot-1/2}
 =
 \remd_j+\jac_j\cpay_j-\jac_j\diff\perr_j.
\end{equation}
We know the remainder $\remd_j$ is already controlled. If $j\in\{j_1,\ldots,j_p\}$, an existing Jacobian with the same player label is already present, so the term $\jac_j\diff\score_{j,\cdot-1/2}$ is controlled by \eqref{eq:repglob}. If instead $j\notin\{j_1,\ldots,j_p\}$, then $-\jac_j\diff\perr_j$ is controlled and $\jac_j\cpay_j$ is the only term that may continue.
Thus every continuing term introduces one fresh Jacobian label.
There are $\nPlayers-p$ possible fresh labels, and by \eqref{eq:labelmult} each can arise at most $3p$ times. Therefore
\begin{equation}
 |\cont(M)|\le3p(\nPlayers-p),
\end{equation}
which proves \eqref{eq:cont}. In particular, no continuation is possible when $p=\nPlayers$.
We next check that a continuing term satisfies the same structural conditions at the next order. If $\diff\strat_j$ came from a strategy term, the continuing replacement is simply $\jac_j\cpay_j$, and
\begin{equation}
 \norm*{\jac_j\cpay_j}_1
 \le
 \opnorm{\jac_j}{\infty}{1}\norm*{\cpay_j}_\infty.
\end{equation}
If it came from a payoff term, \cref{lem:paydiff} produces an expression of the form $C_{kj}\jac_j\cpay_j$, and \eqref{eq:Ccontr} gives
\begin{equation}
 \norm*{C_{kj}\jac_j\cpay_j}_\infty
 \le
 \opnorm{\jac_j}{\infty}{1}\norm*{\cpay_j}_\infty
 \prod_a\norm*{z_a}_1.
\end{equation}
Hence \eqref{eq:pcontr} is preserved in either case. Differencing a payoff term inserts at most $\nPlayers-1$ strategy terms, with each player appearing at most once, whereas differencing a strategy term inserts none. Consequently, a continuing term has at most $\nPlayers(p+1)$ strategy terms, each player occurs at most $p+1$ times among them, it has $p+1$ distinct Jacobian labels and between $1$ and $p+1$ payoff terms, and every time offset lies in $\{0,\ldots,p+1\}$.
It remains then to bound the controlled part. Write
\begin{equation}
 R_M
 =
 R_M^{\jac}+R_M^0+R_M^{\remd}+R_M^{\perr}+R_M^\score,
\end{equation}
where the five pieces contain, respectively, a Jacobian variation, an initialization term, a remainder $\remd_j$, a term $-\jac_j\diff\perr_j$, and a repeated player term $\jac_j\diff\score_{j,\cdot-1/2}$.
For brevity, set
\begin{equation}
 P=\PV(\horizon),\qquad
 G=\NV(\horizon),\qquad
 E=\EE(\horizon).
\end{equation}
Since $\norm*{\strat_{i,t}}_1=1$, $\opnorm{\jac_{i,t}}{\infty}{1}\le\learn$ by \eqref{eq:Jbd}, and $\norm*{\cpay_{i,t}}_\infty\le2$ by \eqref{eq:gprop}, the product of the bounds for the $p$ Jacobian terms, the $s\le p$ payoff terms, and all strategy terms is at most
\begin{equation}
 \learn^p2^s\le(2\learn)^p.
\end{equation}
We use this bound after removing whichever factor is being controlled.
For $R_M^{\jac}$, one of the $p$ Jacobian factors is replaced by $\diff\jac_j$. The remaining factors contribute at most $2(2\learn)^{p-1}$. Hence \eqref{eq:Kl2T}, Cauchy--Schwarz over the $p$ distinct Jacobian labels, and \eqref{eq:Jglob} give
\begin{align}
 \norm*{\dop_{p+1}R_M^{\jac}}_{\ell_2}
 &\le
 4(\nPlayers+1)^{p+2}
 \,2(2\learn)^{p-1}\sqrt p
 \left(
   \sum_{j=1}^{\nPlayers}
   \norm*{\diff\jac_j}_{\ell_2}^2
 \right)^{1/2}
 \\
 &\le
 8\sqrt{4100}\sqrt p\,
 (\nPlayers+1)^{p+2}(2\learn)^{p-1}
 \learn^{3/2}\sqrt{P+G}
 \\
 &\le
 520\sqrt p\,
 (\nPlayers+1)^{p+2}(2\learn)^{p-1}
 \learn^{3/2}\sqrt{P+G},
\end{align}
where $8\sqrt{4100}<520$.
For $R_M^0$, each initialization term is supported at one time and has norm at most $2$ by \eqref{eq:gextbd}. There are at most $s\le p$ such terms, so \eqref{eq:KptT} gives
\begin{equation}
 \norm*{\dop_{p+1}R_M^0}_{\ell_2}
 \le
 2p(\nPlayers+1)^{p+1}(2\learn)^p.
\end{equation}
For $R_M^{\remd}$, there are at most $3\nPlayers p$ revealed variations, while each label occurs at most $3p$ times by \eqref{eq:labelmult}. After singling out $\remd_j$, the remaining factors contribute at most $(2\learn)^p$. Applying \eqref{eq:mult} with
\begin{equation}
 n_0=3\nPlayers p,\qquad n_1=3p,
\end{equation}
and then \eqref{eq:rglob} and \eqref{eq:Kl2T}, we obtain
\begin{align}
 \norm*{\dop_{p+1}R_M^{\remd}}_{\ell_2}
 &\le
 4(\nPlayers+1)^{p+2}(2\learn)^p
 \sqrt{(3\nPlayers p)(3p)}
 \left(
   \sum_{j=1}^{\nPlayers}
   \norm*{\remd_j}_{\ell_2}^2
 \right)^{1/2}
 \\
 &\le
 12\sqrt{300000}\,
 p\sqrt{\nPlayers}\,
 (\nPlayers+1)^{p+3}(2\learn)^p
 \learn^{3/2}\sqrt{P+G}
 \\
 &\le
 6600p\sqrt{\nPlayers}\,
 (\nPlayers+1)^{p+3}(2\learn)^p
 \learn^{3/2}\sqrt{P+G},
\end{align}
where $12\sqrt{300000}<6600$. For $R_M^{\perr}$, the revealed term is $-\jac_j\diff\perr_j$. After singling out $\diff\perr_j$, there are $p+1$ Jacobian factors and at most $p$ payoff factors, so the remaining product is at most
\begin{equation}
 \learn^{p+1}2^p
 =
 \frac12(2\learn)^{p+1}.
\end{equation}
Using again $n_0=3\nPlayers p$, $n_1=3p$, together with \eqref{eq:mult}, \eqref{eq:deglob}, and \eqref{eq:Kl2T}, gives
\begin{align}
 \norm*{\dop_{p+1}R_M^{\perr}}_{\ell_2}
 &\le
 4(\nPlayers+1)^{p+2}
 \frac12(2\learn)^{p+1}
 \sqrt{(3\nPlayers p)(3p)}
 \left(
   \sum_{j=1}^{\nPlayers}
   \norm*{\diff\perr_j}_{\ell_2}^2
 \right)^{1/2}
 \\
 &\le
 12p\sqrt{\nPlayers}\,
 (\nPlayers+1)^{p+2}(2\learn)^{p+1}\sqrt E.
\end{align}
Finally, $R_M^\score$ contains the repeated player terms $\jac_j\diff\score_{j,\cdot-1/2}$. We pair this factor with the existing Jacobian carrying the same player label and apply \eqref{eq:repglob}, and the remaining factors contribute at most $2(2\learn)^{p-1}$. Thus \eqref{eq:mult}, \eqref{eq:repglob}, and \eqref{eq:Kl2T} give
\begin{align}
 \norm*{\dop_{p+1}R_M^\score}_{\ell_2}
 &\le
 4(\nPlayers+1)^{p+2}
 \,2(2\learn)^{p-1}
 \sqrt{(3\nPlayers p)(3p)}
 \sqrt{90}\,\learn^{3/2}\sqrt{P+G}
 \\
 &\le
 240p\sqrt{\nPlayers}\,
 (\nPlayers+1)^{p+2}(2\learn)^{p-1}
 \learn^{3/2}\sqrt{P+G},
\end{align}
where $24\sqrt{90}<240$.
We now put the three $P+G$ terms on the same scale. By \eqref{eq:eta},
\begin{equation}
 2(\nPlayers+1)\learn\le\frac1{120},
\end{equation}
so the $R_M^{\remd}$ bound becomes
\begin{align}
 &6600p\sqrt{\nPlayers}\,
 (\nPlayers+1)^{p+3}(2\learn)^p
 \learn^{3/2}\sqrt{P+G}
 \\
 &\qquad\le
 55p\sqrt{\nPlayers}\,
 (\nPlayers+1)^{p+2}(2\learn)^{p-1}
 \learn^{3/2}\sqrt{P+G}.
\end{align}
Also, since $1\le p\le\nPlayers$, $\sqrt p\le p\sqrt{\nPlayers}$. Therefore the three coefficients of the common $P+G$ factor are bounded by
\begin{equation}
 520+55+240=815.
\end{equation}
Adding the five pieces proves \eqref{eq:term}.
\end{proof}

\begin{proposition}[Prediction error bound]\label{prop:err}
If \eqref{eq:eta} holds, then, for every $\horizon\ge1$,
\begin{equation}\label{eq:esqrt}
 \sqrt{\EE(\horizon)}
 \le
 5\sqrt{\nPlayers}\,(\nPlayers+1)
 +2000(\nPlayers+1)^5\learn^{3/2}
 \sqrt{\PV(\horizon)+\NV(\horizon)}
 +\frac3{10}\sqrt{\EE(\horizon)}.
\end{equation}
\end{proposition}

\begin{proof}
We first iterate \cref{lem:step}. Every continuing term contains at least one payoff term and therefore vanishes for $t\le0$. Let $M$ satisfy the hypotheses of \cref{lem:step} with $p=1$, set $\Fset_1=\{M\}$, and recursively define
\begin{equation}
 \Fset_{p+1}
 =
 \biguplus_{M'\in\Fset_p}\cont(M'),
\end{equation}
so that $\Fset_p$ is the multiset of continuing terms containing $p$ distinct Jacobian labels. Since $\dop_p=\dop_{p+1}\diff$ by \eqref{eq:Krec}, iterating \eqref{eq:step} and using $\cont(M')=\varnothing$ at $p=\nPlayers$ gives
\begin{equation}
 \dop_1M
 =
 \sum_{p=1}^{\nPlayers}
 \sum_{M'\in\Fset_p}
 \dop_{p+1}R_{M'}.
\end{equation}
We next bound the size of the continuing families. By \eqref{eq:cont},
\begin{equation}
 |\Fset_{p+1}|
 \le
 3p(\nPlayers-p)|\Fset_p|.
\end{equation}
At the same time, passing from order $p$ to $p+1$ multiplies each of the weights appearing in \eqref{eq:term} by $2(\nPlayers+1)\learn$. Hence
\begin{align}
 6p(\nPlayers-p)(\nPlayers+1)\learn
 &\le
 \frac32\nPlayers^2(\nPlayers+1)\learn
 \\
 &\le
 \frac1{40},
\end{align}
where the last inequality follows from \eqref{eq:eta}. Starting from $|\Fset_1|=1$, induction therefore gives
\begin{align}
 |\Fset_p|(\nPlayers+1)^{p+2}(2\learn)^{p-1}
 &\le
 (\nPlayers+1)^3
 \left(\frac1{40}\right)^{p-1},
 \\
 |\Fset_p|(\nPlayers+1)^{p+1}(2\learn)^p
 &\le
 (\nPlayers+1)^2(2\learn)
 \left(\frac1{40}\right)^{p-1}.
\end{align}
Since
\begin{equation}
 \sum_{p\ge1}
 p\left(\frac1{40}\right)^{p-1}
 =
 \left(\frac{40}{39}\right)^2
 <
 \frac{11}{10},
\end{equation}
summing \eqref{eq:term} over all $M'\in\Fset_p$ and all $p$ yields
\begin{align}
 \norm*{\dop_1M}_{\ell_2}
 &\le
 5(\nPlayers+1)^2\learn
 \\
 &\quad+
 900\sqrt{\nPlayers}\,(\nPlayers+1)^3
 \learn^{3/2}\sqrt{\PV(\horizon)+\NV(\horizon)}
 \\
 &\quad+
 53\sqrt{\nPlayers}\,(\nPlayers+1)^3
 \learn^2\sqrt{\EE(\horizon)}.
 \label{eq:init}
\end{align}
Indeed, the three numerical constants follow from
\begin{equation}
 2\cdot2\cdot\frac{11}{10}<5,\qquad
 815\cdot\frac{11}{10}<900,\qquad
 12\cdot4\cdot\frac{11}{10}<53.
\end{equation}
We now apply this bound to the prediction error. Define
\begin{equation}
 b_{i,t}
 =
 \mathbf1_{\{t=1\}}\gext_{i,0},
 \qquad
 A_i^{\remd}
 =
 \sum_{j\ne i}C_{ij}^{\payv}\remd_j
 +
 \sum_{j=1}^{\nPlayers}C_{ij}^{\pay}\remd_j,
\end{equation}
and
\begin{equation}
 A_i^{\perr}
 =
 \sum_{j\ne i}C_{ij}^{\payv}\jac_j\diff\perr_j
 +
 \sum_{j=1}^{\nPlayers}C_{ij}^{\pay}\jac_j\diff\perr_j,
 \qquad
 A_i^{\cpay}
 =
 \sum_{j\ne i}C_{ij}^{\payv}\jac_j\cpay_j
 +
 \sum_{j=1}^{\nPlayers}C_{ij}^{\pay}\jac_j\cpay_j.
\end{equation}
Then \eqref{eq:Krec}, \eqref{eq:gbdry}, \eqref{eq:paydiff}, \eqref{eq:mir1}, and \eqref{eq:ydiff} give
\begin{equation}\label{eq:errdec}
 \perr_i
 =
 \dop_1b_i
 +
 \dop_1(A_i^{\remd}-A_i^{\perr}+A_i^{\cpay}).
\end{equation}
We bound the four contributions in \eqref{eq:errdec}. First, \eqref{eq:KptT} and \eqref{eq:gextbd} give
\begin{equation}
 \left(
   \sum_i\norm*{\dop_1b_i}_{\ell_2}^2
 \right)^{1/2}
 \le
 4\sqrt{\nPlayers}\,(\nPlayers+1).
\end{equation}
Next, each of $A_i^{\remd}$ and $A_i^{\perr}$ contains $2\nPlayers-1$ terms, and by \cref{lem:paydiff} each fixed player label occurs at most twice. Thus \eqref{eq:mult} may be applied with
\begin{equation}
 n_0=2\nPlayers-1,\qquad n_1=2.
\end{equation}
Using \eqref{eq:Kl2T} and \eqref{eq:rglob}, and then taking the $\ell_2$ norm over $i$, gives
\begin{align}
 \left(
   \sum_i\norm*{\dop_1A_i^{\remd}}_{\ell_2}^2
 \right)^{1/2}
 &\le
 4\sqrt{300000}\,
 \sqrt{2\nPlayers(2\nPlayers-1)}
 (\nPlayers+1)^3
 \learn^{3/2}
 \sqrt{\PV(\horizon)+\NV(\horizon)}
 \\
 &<
 2200\sqrt{2\nPlayers(2\nPlayers-1)}
 (\nPlayers+1)^3
 \learn^{3/2}
 \sqrt{\PV(\horizon)+\NV(\horizon)},
\end{align}
where $4\sqrt{300000}<2200$.
Similarly, using \eqref{eq:Jbd} and \eqref{eq:deglob},
\begin{align}
 \left(
   \sum_i\norm*{\dop_1A_i^{\perr}}_{\ell_2}^2
 \right)^{1/2}
 &\le
 4\cdot2\,
 \sqrt{2\nPlayers(2\nPlayers-1)}
 (\nPlayers+1)^2
 \learn\sqrt{\EE(\horizon)}
 \\
 &=
 8\sqrt{2\nPlayers(2\nPlayers-1)}
 (\nPlayers+1)^2
 \learn\sqrt{\EE(\horizon)}.
\end{align}
Finally, $A_i^{\cpay}$ is a sum of $2\nPlayers-1$ terms satisfying the hypotheses of \cref{lem:step} with $p=1$. Applying \eqref{eq:init} to each such term, summing them, and then taking the $\ell_2$ norm over $i$, gives the three contributions
\begin{align}
 &5\sqrt{\nPlayers}(2\nPlayers-1)
   (\nPlayers+1)^2\learn,
 \\
 &900\nPlayers(2\nPlayers-1)
   (\nPlayers+1)^3\learn^{3/2}
   \sqrt{\PV(\horizon)+\NV(\horizon)},
 \\
 &53\nPlayers(2\nPlayers-1)
   (\nPlayers+1)^3\learn^2
   \sqrt{\EE(\horizon)}.
\end{align}
It remains only to combine these bounds. For the initialization terms, \eqref{eq:eta} and $2\nPlayers-1\le(\nPlayers+1)^2$ imply
\begin{align}
 &4\sqrt{\nPlayers}\,(\nPlayers+1)
 +5\sqrt{\nPlayers}(2\nPlayers-1)
   (\nPlayers+1)^2\learn
 \\
 &\qquad\le
 \left(4+\frac1{12}\right)
 \sqrt{\nPlayers}\,(\nPlayers+1)
 <
 5\sqrt{\nPlayers}\,(\nPlayers+1).
\end{align}
For the coefficient of $\sqrt{\PV(\horizon)+\NV(\horizon)}$, we use
\begin{equation}
 \sqrt{2\nPlayers(2\nPlayers-1)}
 \le2\nPlayers,
 \qquad
 2\nPlayers-1\le2\nPlayers.
\end{equation}
Hence
\begin{align}
 &2200\sqrt{2\nPlayers(2\nPlayers-1)}
 +900\nPlayers(2\nPlayers-1)
 \\
 &\qquad\le
 4400\nPlayers+1800\nPlayers^2
 \\
 &\qquad\le
 2000(\nPlayers+1)^2,
\end{align}
because
\begin{equation}
 2000(\nPlayers+1)^2
 -(1800\nPlayers^2+4400\nPlayers)
 =
 200(\nPlayers^2-2\nPlayers+10)>0.
\end{equation}
After restoring the common factor $(\nPlayers+1)^3\learn^{3/2}$, this gives
\begin{equation}
 2000(\nPlayers+1)^5\learn^{3/2}
 \sqrt{\PV(\horizon)+\NV(\horizon)}.
\end{equation}
Finally, for the coefficient of $\sqrt{\EE(\horizon)}$, \eqref{eq:eta} gives
\begin{align}
 8\sqrt{2\nPlayers(2\nPlayers-1)}
 (\nPlayers+1)^2\learn
 &\le
 \frac{4\nPlayers}{15(\nPlayers+1)}
 <
 \frac4{15},
 \\
 53\nPlayers(2\nPlayers-1)
 (\nPlayers+1)^3\learn^2
 &\le
 \frac{106\nPlayers^2}
 {60^2(\nPlayers+1)^3}
 \le
 \frac{53}{3600}
 <
 \frac1{60},
\end{align}
where we used
\begin{equation}
 \frac{\nPlayers^2}{(\nPlayers+1)^3}
 \le
 \frac1{\nPlayers+1}
 \le
 \frac12.
\end{equation}
Since
\begin{equation}
 \frac4{15}+\frac1{60}
 =
 \frac{17}{60}
 <
 \frac3{10},
\end{equation}
the total $\sqrt{\EE(\horizon)}$ contribution is at most $\frac3{10}\sqrt{\EE(\horizon)}$. Combining the three bounds proves \eqref{eq:esqrt}.
\end{proof}

\subsection{Proof of the main theorem}\label{app:proof}

We use the following bound on $\rrange$.

\begin{lemma}[Regularizer range]\label{lem:osc}
\begin{equation}\label{eq:osc}
 \rrange=\max_{\Kset}\lreg-\min_{\Kset}\lreg
 \le\frac{13}{6}\Lreg.
\end{equation}
\end{lemma}

\begin{proof}
Both terms in \eqref{eq:lreg} are nonpositive, so $\max_{\Kset}\lreg=0$. Since
\begin{equation}
 \Lreg-3\log \nPures=\frac14(\log \nPures-4)^2\ge0,
\end{equation}
we have $\log \nPures/\Lreg\le1/3$. Hence
\begin{equation}
 \frac{-\lreg(\mass \strat)}\Lreg
 \le
 \frac{\log \nPures}{\Lreg}(\mass+\sqrt\mass)
 +3\sqrt{\mass(1-\mass)}
 \le\frac23+\frac32=\frac{13}{6}.
\end{equation}
This proves \eqref{eq:osc}.
\end{proof}

\begin{proof}[Proof of \cref{thm:main}]
From \eqref{eq:social}, \eqref{eq:Nerr}, and $\sum_{i=1}^\nPlayers[\reg_i(\horizon)]_+\ge0$,
\begin{equation}
 \PV(\horizon)+\NV(\horizon)\le\frac{\nPlayers\rrange}{\learn}+\learn \EE(\horizon).
\end{equation}
Substituting this into \eqref{eq:esqrt} and using $\sqrt{a+b}\le\sqrt a+\sqrt b$ gives
\begin{equation}
 \sqrt{\EE(\horizon)}
 \le5\sqrt \nPlayers\,(\nPlayers+1)
 +2000(\nPlayers+1)^5\learn\sqrt{\nPlayers\rrange}
 +2000(\nPlayers+1)^5\learn^2\sqrt{\EE(\horizon)}
 +\frac3{10}\sqrt{\EE(\horizon)}.
\end{equation}
By \eqref{eq:osc}, $\Lreg\ge4$, and $(\nPlayers+1)^3\learn\le1/60$,
\begin{align}
 5\sqrt \nPlayers\,(\nPlayers+1)
 &\le\frac58(\nPlayers+1)^3\sqrt \Lreg,\\
 2000(\nPlayers+1)^5\learn\sqrt{\nPlayers\rrange}
 &\le\frac{100}{3}(\nPlayers+1)^2\sqrt{\frac{13\nPlayers\Lreg}6}
 <25(\nPlayers+1)^3\sqrt \Lreg,\\
 2000(\nPlayers+1)^5\learn^2
 &\le\frac5{18}.
\end{align}
For the second inequality, $\nPlayers+1\ge2\sqrt \nPlayers$, so $\sqrt \nPlayers/(\nPlayers+1)\le1/2$, and we used $\sqrt{13/6}<3/2$. Therefore
\begin{equation}
 \sqrt{\EE(\horizon)}
 <26(\nPlayers+1)^3\sqrt \Lreg+\frac{26}{45}\sqrt{\EE(\horizon)}.
\end{equation}
Since $26\cdot45/19<62$,
\begin{equation}\label{eq:E}
 \EE(\horizon)<62^2(\nPlayers+1)^6\Lreg.
\end{equation}
Finally, \eqref{eq:indiv}, \eqref{eq:Npt}, and $\PV_{i,t}\ge0$ give
\begin{equation}
 [\reg_i(\horizon)]_+\le\frac{\rrange}{\learn}+\frac\learn2\EE(\horizon).
\end{equation}
Using \eqref{eq:osc}, \eqref{eq:E}, and $(\nPlayers+1)^3\learn\le1/60$,
\begin{equation}
 [\reg_i(\horizon)]_+
 <\left(\frac{13}{6}+\frac59\right)\frac{\Lreg}{\learn}
 =\frac{49}{18}\frac{\Lreg}{\learn}
 <\frac{3\Lreg}{\learn}.
\end{equation}
This proves \eqref{eq:bd}. Taking $\learn=1/(60(\nPlayers+1)^3)$ gives \eqref{eq:bdmax}, and \eqref{eq:rate} follows.
\end{proof}

\section{Adversarial robustness}\label{app:switch}

For the switching rule stated after \cref{thm:main}, set
\begin{equation}
 B=\frac{3}{\learn}\left(4+\log \nPures+\frac14\log^2\nPures\right).
\end{equation}
Fix any full information adversarial rule whose $k$-th strategy after a fresh start is $q_k$ and whose remaining regret satisfies, for every payoff sequence $w_1,\ldots,w_s\in[-1,1]^\nPures$,
\begin{equation}\label{eq:sufreg}
 \max_{\pure\in\pures}
 \sum_{k=1}^s\braket*{\pure-q_k}{w_k}
 \le \adv(s),
 \qquad
 \adv(s)=o(s).
\end{equation}
Standard examples of such rules include Hedge / exponential weights with well-chosen step sizes.
Put $\reg_i(0)=0$.  Define
\begin{equation}
 \tau_i=\inf\{t\ge1:\reg_i(t)>B\},
 \qquad \inf\varnothing=\infty.
\end{equation}
Run \cref{alg:hood} through round $\tau_i$ and, if $\tau_i<\infty$, restart the adversarial rule from round $\tau_i+1$. The switching rule follows the bounded regret dynamics until the first bound violation and then makes a single permanent switch to the adversarially robust rule.

\begin{proposition}[Switching guarantee]
If every player uses the above rule, then $\tau_i=\infty$ for every player.
Against arbitrary deviations of the other players,
\begin{equation}\label{eq:switch}
 \limsup_{\horizon\to\infty}\frac{\reg_i(\horizon)}\horizon\le0.
\end{equation}
\end{proposition}

\begin{proof}
By definition of $\tau_i$, for $0\le \horizon<\tau_i$, $\reg_i(\horizon)\le B$.
Moreover, one round can increase external regret by at most $2$, because $|\cpay_{i\pure,t}|\le2$ for every action $\pure$.  Hence, if $\tau_i<\infty$,
\begin{equation}
 \reg_i(\tau_i)\le B+2.
\end{equation}
After round $\tau_i$, the restarted adversarial rule is run on the payoff vectors $\payv_{i,\tau_i+1},\ldots,\payv_{i,\horizon}$.  Since
\begin{equation}
 \braket*{\pure-\strat_{i,t}}{\payv_{i,t}}
 =\pay_i(\pure,\strat_{-i,t})-\pay_i(\strat_t),
\end{equation}
its regret guarantee \eqref{eq:sufreg} gives, for $\horizon>\tau_i$,
\begin{equation}
 \reg_i(\horizon)
 \le \reg_i(\tau_i)+\adv(\horizon-\tau_i)
 \le B+2+\adv(\horizon-\tau_i),
\end{equation}
so the switching rule satisfies the three bounds above for every realized play of the other players. Suppose now that every player uses the switching rule and that a switch occurs at some point for some player.  Let $\tau=\min_j\tau_j$, then up to and including the round $\tau$, no player has switched and so the realized play coincides with a trajectory on which all players use \cref{alg:hood}.  This means that the bound \eqref{eq:bd} implies $\reg_j(\tau)\le B$ for every $j$, contradicting the fact that $\reg_j(\tau)>B$ for any player with $\tau_j=\tau$, hence $\tau_i=\infty$ for every player. Finally, against arbitrary opponents, if $\tau_i=\infty$ then $\reg_i(\horizon)\le B$ for all $\horizon$, and if $\tau_i<\infty$, then for $\horizon>\tau_i$,
\begin{equation}
 \frac{\reg_i(\horizon)}\horizon
 \le\frac{B+2}{\horizon}
 +\frac{\horizon-\tau_i}{\horizon}\frac{\adv(\horizon-\tau_i)}{\horizon-\tau_i}.
\end{equation}
Both terms tend to zero because $\adv(s)=o(s)$, proving \eqref{eq:switch}.
\end{proof}

\begingroup \emergencystretch=1.5em \bibliographystyle{abbrvnat} \bibliography{hood-references} \endgroup

\end{document}